\documentclass[10pt]{article}
\usepackage[preprint]{tmlr}
\usepackage[utf8]{inputenc}
\usepackage{textcomp}
\usepackage{amsthm,amsmath,amsfonts,amssymb}
\usepackage{thmtools}
\usepackage{graphicx}
\usepackage{xcolor}
\usepackage{tikz}
\usetikzlibrary{bayesnet}
\usepackage{multicol}
\usepackage{placeins}
\usepackage{booktabs,array}
\usepackage{etoolbox}
\usepackage{microtype}
\usepackage[normalem]{ulem}
\usepackage{subcaption}
\usepackage{url}
\usepackage[colorlinks,citecolor=blue,urlcolor=blue,linkcolor=blue]{hyperref}

\definecolor{revcolor}{RGB}{124,45,160}

\colorlet{revtwo}{revcolor}

\definecolor{fixcolor}{RGB}{198,74,0}

\newif\ifanon
\anonfalse   %% <-- set to \anonfalse for the camera-ready version

\theoremstyle{plain}
\newtheorem{theorem}{Theorem}[section]
\newtheorem{proposition}[theorem]{Proposition}
\DeclareMathOperator*{\argmax}{arg\,max}
\DeclareMathOperator*{\argmin}{arg\,min}
\title{A unifying Bayesian framework for adversarial robustness}

\author{\name Pablo G. Arce \email pablo.garcia@icmat.es \\
      \addr Universidad Aut\'onoma de Madrid, Escuela de Doctorado, Madrid, Spain \\
      Institute of Mathematical Sciences, Spanish National Research Council, Madrid, Spain
      \AND
      \name Roi Naveiro \email roi.naveiro@cunef.edu \\
      \addr CUNEF Universidad, Madrid, Spain
      \AND
      \name David R\'ios Insua \email david.rios@icmat.es \\
      \addr Institute of Mathematical Sciences, Spanish National Research Council, Madrid, Spain}

\def\month{11}
\def\year{2026}

\begin{document}

\maketitle

\begin{abstract}
The vulnerability of machine learning models to adversarial attacks remains a critical societal security challenge. Yet the defense of Bayesian predictive models has received little attention, even though their attack surface is wider: an adversary may corrupt not only their predictions but also the uncertainty they quantify. Moreover, traditional defenses, such as adversarial training, typically robustify models by minimizing a worst-case loss. These deterministic approaches do not account for uncertainty in the adversary's attack. While stochastic defenses placing a probability distribution on the adversary exist, they often lack statistical rigor and fail to make explicit their underlying assumptions. Our aim is thus twofold: to protect Bayesian predictive models, and to do it in a Bayesian way, explicitly modeling uncertainty about the attacker. To this end, we introduce a formal  Bayesian framework for adversarial robustness that captures adversarial uncertainty through a stochastic channel and makes every probabilistic assumption explicit. This yields two robustification strategies that protect the entire posterior predictive distribution: a proactive defense enacted during training, aligned with adversarial training, and a reactive defense enacted during operations, aligned with adversarial purification. Several state-of-the-art defenses can be recovered as limiting cases of the proposed framework. Consistent with this wider attack surface, our experiments assess both accuracy and the quality of predictive uncertainty under attack, and showcase the benefits of explicitly modeling adversarial uncertainty.
\end{abstract}

%\rev{\noindent\textbf{Keywords:} Machine learning; Security; Bayesian analysis; Defenses; Attacks}

\section{Introduction}

The increasing importance of machine learning, amplified by large language models, underscores the transformative potential of artificial intelligence \citep{zhao2023survey}. However, this progress is shadowed by security issues, particularly due to adversarial attacks, which have given rise to the relatively recent field of {\em adversarial machine learning} (AML) \citep{adversarialClassification2004,joseph}. By maliciously manipulating test inputs to alter model predictions and corrupt downstream decisions, adversaries break the core i.i.d. assumption underlying standard machine learning (ML) predictive models \citep{murphy2023probabilistic}. AML thus pursues two complementary goals: understanding how adversarial manipulations compromise ML models, and designing adversarially robust algorithms that remain reliable under them.
% To Do: improve connection here.
Both goals have been extensively studied for classical ML models, but the adversarial robustness of Bayesian predictive ML models remains a critical and underexplored frontier, still under foundational debate \citep{feng2024attacking}. This is a significant gap, as Bayesian methods are essential in high-stakes domains where principled uncertainty quantification is paramount. Their attack surface is also wider: an adversary may corrupt their uncertainty quantification as well as their predictions, so defenses must protect, and be evaluated on, the full predictive distribution. Existing work has mainly focused on demonstrating vulnerabilities of Bayesian models to adversarial data manipulations \citep{arce2025evasion}, but a principled foundation for designing defenses is still largely absent.

This paper bridges this gap with a twofold aim: to protect Bayesian predictive/ML models, and to do so in a Bayesian way, explicitly modeling uncertainty about the attacker. To this end, we reframe core approaches to adversarial defense. Rather than adopting the standard AML paradigm, which essentially formulates defense as a deterministic minimax game against a worst-case adversary, we cast it as a hierarchical modeling problem. From a statistical perspective, the presence of an adversary manipulating data is essentially a severe form of model misspecification. While robust Bayesian methods handle misspecification agnostically, tempering or coarsening the likelihood whatever the source of the departure \citep{grunwald2017inconsistency, miller2019robust}, we argue that adversarial attacks possess specific structural properties, stemming from the adversary's goals and capabilities, that should be explicitly modeled. We therefore introduce a Bayesian framework that captures the attacker's behavior via a stochastic \emph{adversarial channel}: a probabilistic model that incorporates our prior beliefs about the threat process, describing how latent, clean inputs are corrupted into the observations the ML model receives. 
Our contributions include:
\begin{itemize}
    \item A statistically grounded Bayesian framework for adversarial defense that makes all probabilistic assumptions transparent and protects Bayesian predictive models.
    \item The derivation of two complementary strategies, a reactive defense that infers latent clean inputs from corrupted data during model deployment, and a proactive defense that builds inherent robustness %directly into the model parameters 
during training, along with tractable inference and prediction schemes that enable learning against complex mixtures of adversaries.
    \item A demonstration that our framework generalizes prior art by recovering as limiting cases prominent AML defenses, such as adversarial training (AT) and randomized smoothing (RS).
    \item An empirical validation that assesses both accuracy and the quality of predictive uncertainty under attack, and shows that our approach outperforms AT, including AT against a multi-step adversary, in robustness and calibration. On CIFAR-10, without extra training data, a Bayesian wide residual network trained with our approach matches the $L_2$ AutoAttack accuracy of deterministic AT, which, to the best of our knowledge, no previously reported Bayesian defense achieves.
\end{itemize}

\noindent  Code to reproduce the experiments as well as full hyperparameter specifications can be found at \ifanon \url{https://anonymous.4open.science/r/advDef}\else\url{https://github.com/pablogarciarce/advDef}\fi. Proofs of all propositions are given in Appendix~\ref{app:proofs}.

% This paper establishes such foundation. % for Bayesian predictive models.  We propose a fully Bayesian framework that models adversarial actions through a stochastic channel, allowing us to formally incorporate uncertainty about the adversary. Our contributions include:

% \begin{itemize}
%     \item A statistically grounded Bayesian framework for adversarial defenses that makes all assumptions transparent.
% 
%     \item The derivation of two strategies: a {\em reactive defense} for deployment and a {\em proactive defense} for training, along with tractable inference schemes.
% 
%     \item A demonstration that our framework generalizes prior art, recovering prominent defenses like {\em adversarial training} (AT)  or {\em randomized smoothing} (RS)  as limiting cases.
% 
% \end{itemize}

%%%%%%%%%%%%e.g
\section{Related Work}

AML has gained significant attention as adversaries can manipulate data inputs to change model inferences or predictions, ultimately altering critical downstream decisions \citep{joseph,vorobeichikantar,riosInsua2023}. While early AML work focused on classification \citep{adversarialClassification2004, goodfellow2014explaining}, the impact of these vulnerabilities is now recognized across diverse learning tasks, including regression \citep{arce2025evasion} and reinforcement learning \citep{gallego2026opponent}. In general, AML defenses fall into two categories: {\em proactive}, enacted during training by modifying the learning process to anticipate future adversaries, and {\em reactive}, deployed during testing or operations by altering how predictions are made to account for potential input corruption. However, most existing strategies suffer from two fundamental limitations: 1) they are essentially deterministic, failing to quantify uncertainty about the adversary's actions, and 2) they mainly address classical, point-estimate predictive models, not Bayesian ones.

The most prominent \emph{proactive} AML \emph{defense} is AT \citep{madry2018towards}, which frames the problem through minimax optimization: its inner loop finds a worst-case attack, typically through projected gradient descent (PGD), whereas its outer loop trains the model to minimize its loss on these attacks. While providing strong empirical protection, AT illustrates the limitations highlighted: by optimizing against a single, worst-case perturbation, its formulation assumes a deterministic threat model that ignores the uncertainty in the adversary's strategy. Furthermore, as AT is inherently designed for point-estimate models, it lacks a native mechanism to protect a full predictive distribution. Popular variants like TRADES \citep{zhang2019theoretically} or adversarial logit pairing (ALP) \citep{kannan2018adversarial} decompose the loss into classification and regularization terms, and Ensemble AT \citep{tramer2017ensemble} diversifies attack generation; yet, the core paradigm remains unchanged. Even heuristic attempts to accommodate uncertainty about adversaries, such as curriculum \citep{cai2018curriculum} or adaptive \citep{balaji2019instance} training, fail to provide a rigorous, natively probabilistic defense.

Complementary to these are \emph{reactive defenses}. The subfield of \emph{adversarial purification} (AP) aims to remove perturbations from corrupted input before classification, focusing on restoring a single ``clean'' input for a non-Bayesian model, rather than propagating the uncertainty about the original input through a posterior. In \textit{model-agnostic} AP strategies, a generative model trained on clean data is used to purify inputs by projecting them back to a learned data manifold, independent of the downstream predictive model. A prime example \citep{nie2022diffusion} employs diffusion models to iteratively denoise the input to find a likely clean predecessor. In contrast, \textit{model-guided} strategies leverage downstream predictive model parameters to actively shape purification. Instead of just seeking a plausible clean instance, they seek one that the specific predictive model is likely to deal with correctly. For instance, {\em AToP} \citep{lin2024adversarial} uses the classifier gradients to guide a generative process, steering it towards high-confidence regions. Another key strategy is RS \citep{cohen2019certified}. Instead of deterministically purifying an input, it constructs a new certifiably robust classifier by predicting the majority vote of a base classifier over several noisy versions of the input. Although probabilistic, RS uses noise to smooth the classifier rather than to model the attack, and involves no posterior inference.
%Although not Bayesian, RS is inherently probabilistic, as it smooths the decision boundary by convolving it with a noise distribution, bridging towards uncertainty-aware defenses, though in its standard form it does not involve posterior inference over model parameters.

These classical paradigms are well-established, yet their limitations are particularly acute when facing \emph{Bayesian models}, as their attack surface is wider: adversaries can target not only point predictions but also entire posterior predictive distributions (PPD) \citep{arce2025evasion}. This makes the adversarial robustness of Bayesian models a critical and developing frontier. Despite early hopes that Bayesian methods might be inherently adversarially robust \citep{de2021adversarial}, recent findings show that this is not guaranteed, through novel attacks like PGD$^+$ \citep{feng2024attacking} or those in \cite{arce2025evasion}, \cite{carreau2025poisoning}, \cite{caballero2025posterior} and \cite{naveiro2026steering}.

Prior attempts to defend Bayesian models, or to defend them in a Bayesian way, have faced their own challenges. Some are heuristic, such as considering distributions of attacks to inform a distributional AT \citep{dong2024enhancing}. There are more formal ones like Bayesian Adversarial Learning  \citep{ye2018bayesian}, which introduces a framework based on Gibbs sampling to approximate a robust posterior. However, as Appendix~\ref{sec:bal} discusses, the conditional distributions defining their sampler cannot be derived from a single, valid joint posterior distribution. Adv-BNN \citep{liu2019advbnn} combines variational BNNs with AT, fitting the variational posterior on PGD examples, and \cite{wicker2021bayesian} train BNNs on worst-case perturbations bounded through interval bound propagation to obtain certifiable robustness.
Both keep the deterministic worst-case adversary of AT inside a Bayesian training objective, rather than explicitly modeling uncertainty about the attacker.
Principled Bayesian frameworks, like that in \cite{gallego2024protecting}, have been limited in scope to classification problems.

Beyond these specific adversarial defenses, our work also relates to the broader theme of model robustness in Bayesian statistics \citep{berger2000bayesian}, including tempered posteriors \citep{grunwald2017inconsistency}, coarsened posteriors \citep{miller2019robust} and generalized variational inference (VI) \citep{knoblauch2022optimization}. These treat data corruption (and by extension, adversarial attacks) as a form of model misspecification and provide rigorous tools for handling general misspecifications. However, these tools are agnostic to the source of the misspecification, whereas the mechanisms of adversarial attacks are well characterized and can be modeled explicitly.
 
In summary, existing defenses either target non-Bayesian models, treat the attack deterministically or ignore its structure. Our framework addresses these gaps. Through a stochastic adversarial channel, it explicitly models the attack and our uncertainty about it, protects the full PPD of Bayesian predictive models, and recovers AT and RS as limiting cases. 

% The absence of a framework that is both probabilistic in nature and natively designed for Bayesian models motivates this work. We introduce the first, to our knowledge, statistically rigorous and fully Bayesian framework that addresses these gaps. It not only models adversarial uncertainty via a stochastic channel but is specifically architected to robustify Bayesian predictive distributions. Furthermore, it recovers prominent defenses like AT and RS as limiting cases.

% Others, like Bayesian Adversarial Learning (BAL) \citep{ye2018bayesian}, are based on a mathematically inconsistent probabilistic model. More principled frameworks, like that of \cite{gallego2024protecting}, have been limited in scope to classification problems only.
% {\color{red} Beyond these specific adversarial defenses, our work also relates to the broader field of robust Bayesian statistics, including coarsened posteriors \citep{miller2019robust} and Generalized VI \citep{knoblauch2022optimization}. While these frameworks provide rigorous tools for handling general model misspecification, they do not explicitly model the unique structure of adversarial attacks. By contrast, our formulation captures this specific mechanism via a stochastic channel. This structural approach is crucial: it allows us to incorporate domain knowledge about the adversary and, unlike generic robust methods, explicitly recover modern engineering defenses like AT and Purification as variational limits.}
%%%%%%%%%%%%%%%%%%%%%%%%%%%%%%%%%%%%%%%%%%
\section{Methodology}\label{sec:methodology}
%%%%%
\subsection{Problem Formulation}\label{subsec:problem}

We frame our problem within the setting of Bayesian predictive models, covering both classification and regression. Each instance pairs a covariate vector $\mathbf{x}\in\mathcal{X}$ (for example $\mathcal{X}=\mathbb{R}^d$, or a space of images) with a response $y\in\mathcal{Y}$, where $\mathcal{Y}$ is a finite set of labels in classification and $\mathcal{Y}=\mathbb{R}$ in regression. Instances are drawn from a joint distribution $p(\mathbf{x},y\mid\phi,\theta)=p(\mathbf{x}\mid\phi)\,p(y\mid\mathbf{x},\theta)$, in which $\phi$ and $\theta$ parameterize the covariate and the conditional response distributions, respectively.\footnote{Throughout, $p(\cdot)$ denotes the density of the relevant distribution with respect to a fixed dominating base measure of each variable (Lebesgue measure for continuous variables, counting measure for discrete ones). Thus $p(\mathbf{x}\mid\phi)$ is a density, whereas $p(y\mid\mathbf{x},\theta)$ is a probability mass function when the response is discrete (classification) and a density when it is continuous (regression). Every expectation is an integral against the corresponding base measure, reducing to a sum over discrete coordinates.} Bayesian inference then proceeds in two stages. Given a clean training set $\mathcal{D}=\{(\mathbf{x}_i,y_i)\}_{i=1}^N$, a prior $p(\phi,\theta)$ is updated into the posterior $p(\phi,\theta\mid\mathcal{D})\propto p(\phi,\theta)\prod_{i=1}^N p(\mathbf{x}_i\mid\phi)\,p(y_i\mid\mathbf{x}_i,\theta)$ by conditioning on the data. Prediction for a new input $\mathbf{x}$ then marginalizes the response likelihood over this posterior, yielding the PPD $p(y\mid\mathbf{x},\mathcal{D})=\int p(y\mid\mathbf{x},\theta)\,p(\theta\mid\mathcal{D})\,d\theta$.

In adversarial settings, at deployment, this setup is disrupted. We focus on \emph{evasion} attacks, which corrupt the covariates of test instances at inference time; poisoning of the training data, corruption of the response, and non-i.i.d.\ settings, such as time series or sequential decision problems, are beyond our scope. Under such an attack, we no longer observe the clean test covariate $\mathbf{x}$: an attacker corrupts it into an observation $\mathbf{x}'$, so that the clean input $\mathbf{x}$ becomes a latent variable and only $\mathbf{x}'$ is available for prediction. Reasoning about this latent input, as our reactive defense does, requires the covariate model $p(\mathbf{x}\mid\phi)$, ignored in standard discriminative prediction, as a generative prior. In our framework, the key modeling piece is the \emph{adversarial channel}, a conditional distribution $p(\mathbf{x}'\mid\mathbf{x},\theta)$ describing this corruption.\footnote{We write the channel as a density, but it may be a more general Markov kernel, such as a Dirac measure at the output of a deterministic attack.} Note that it encodes our beliefs about the attack, not necessarily the true corruption mechanism, which is unknown. Letting it depend on $\theta$ reflects a white-box adversary, who knows the model parameters; treating it as a distribution (an adversary's mixed strategy in the game-theoretic sense) captures our uncertainty about the attack. As a special case, it subsumes the deterministic worst-case adversary of standard AT (see Section~\ref{subsec:at}).

In this setting, our goal is to reliably predict the response $y$ given only the corrupted observation $\mathbf{x}'$, using the adversarial channel as our model of the attacker. To achieve this, we introduce two complementary defensive strategies. The first one is \textit{reactive}, designed to protect the model during operations. It uses a standardly trained model but, upon receiving a possibly corrupted input at test time, delivers an adversarially robust inference mechanism to account for the adversarial channel, in the spirit of {\em adversarial purification} methods.
The second one is \textit{proactive}, building adversarial robustness directly within the training phase, in the spirit of AT. %By integrating the adversarial channel into the learning objective, our method yields a novel Bayesian formulation of classic AT paradigms.

Let us now develop both defenses and their learning schemes. Each defense is grounded in its own probabilistic graphical model (PGM): the reactive defense (Section~\ref{operations}) in Figure~\ref{fig:operations} and the proactive defense (Section~\ref{sec:protect-train}) in Figure~\ref{fig:train}. Although both address the same challenge, they are distinct and, in general, lead to different PPDs (see Appendix~\ref{sec:counterexample}). For now, we take the adversarial channel $p(\mathbf{x}'\mid\mathbf{x},\theta)$ as given; concrete instantiations, shared by both defenses, are presented in Section~\ref{subsec:channels}. 
\subsection{Protection During Operations}\label{operations}

A natural approach to defending a model during deployment (or at test time) is to assume that the labeling mechanism is invariant under attack: while an adversary may corrupt an input from $\mathbf{x}$ to $\mathbf{x}'$, the label $y$ remains conditionally dependent only on the parameters $\theta$ and the original, now latent, covariate vector $\mathbf{x}$. Figure~\ref{fig:operations} captures this, depicting a standard training phase with clean data $\mathcal{D}=\{(\mathbf{x}_i,y_i)\}_{i=1}^N$ used to learn the posterior over $\phi$ and $\theta$. At test time, the defense is enacted upon observing a possibly corrupted input $\mathbf{x}'_j$. It reasons backward through the adversarial channel $p(\mathbf{x}'\mid\mathbf{x},\theta)$ to infer the latent clean input $\mathbf{x}_j$ and, \emph{thereby}, predict the label $y_j$.

\begin{figure}[!h]
\centering
\begin{tikzpicture}[>=stealth, node distance = 1.0cm and 1.6cm]

% ---------------- global parameters -----------------------------------
\node[latent]                       (phi)   {$\phi$};
\node[latent, right=2.5cm of phi]   (theta) {$\theta$};

% ---------------- TRAINING plate (left) -------------------------------
\node[obs,    below=1.2cm of phi]   (xtr)  {$\mathbf{x}_i$};
\node[obs,    right=1.8cm of xtr]   (ytr)  {$y_i$};

\edge {phi} {xtr}
\path (xtr)   edge[->, bend left =  0] (ytr);
\path (theta) edge[->, bend right = 0] (ytr);

\plate {train}{(xtr)(ytr)} {$i = 1,\dots,N$}
\node[anchor=south west, inner sep=6pt, font=\scriptsize] at (train.south west) {Train};

% ---------------- TEST plate (right) ----------------------------------
\node[latent, right=3.5cm of xtr]  (xtest)  {$\mathbf{x}_j$};
\node[obs,    below=0.8cm of xtest] (xptest) {$\mathbf{x}'_j$};
\node[latent, right=1.8cm of xtest] (ytest)  {$y_j$};

\path (phi) edge[->, bend left = 0] (xtest);
\edge {xtest} {xptest}
\path (xtest) edge[->, bend left =  0] (ytest);

\path (theta) edge[->, bend right = 20] (xptest);
\path (theta) edge[->, bend right = 0] (ytest);

\plate {test}{(xtest)(xptest)(ytest)} {$j = 1,\dots,M$}
\node[anchor=south west, inner sep=6pt, font=\scriptsize] at (test.south west) {Test};

\end{tikzpicture}
\caption{PGM for reactive defense. Nodes represent random variables. Arrows, conditional dependencies. Shaded, observed variables. Unshaded, latent variables.} \label{fig:operations}
\end{figure}

%Throughout this subsection, $p(\cdot)$ will denote either a density or a probability mass function with respect to the natural base measure of each variable. For instance, when the covariates are continuous and the labels are discrete, $p(\mathbf{x}\mid\phi)$ and $p(\mathbf{x}'\mid\mathbf{x},\theta)$ are densities, whereas $p(y\mid\mathbf{x},\theta)$ is a class probability; integrals are then only over continuous latent quantities such as $\mathbf{x}_j$. If the covariates are discrete, the corresponding integrals are replaced by sums. If the adversarial channel is deterministic or, otherwise, not absolutely continuous, $p(\mathbf{x}'\mid\mathbf{x},\theta)$ should be read as a Markov kernel, for example a Dirac measure at an attack output. 
%
We assume that the channel admits a density with respect to a fixed $\sigma$-finite measure\footnote{The expressions below evaluate this density at the observed $\mathbf{x}'_j$, which a deterministic attack on continuous inputs does not allow. Such channels remain admissible in the proactive defense of Section~\ref{sec:protect-train}, which only integrates against the channel.} and that the normalizing constants below are finite and positive for the observed corrupted input. Under the PGM in Figure~\ref{fig:operations}, the joint density %/mass
of the parameters, the $N$ training data, the latent clean test input $\mathbf{x}_j$ and a test instance $(\mathbf{x}'_j,y_j)$ factorizes as
\begin{equation*}
p(\theta,\phi,\mathcal{D},\mathbf{x}_j,\mathbf{x}'_j,y_j) = 
p(\theta)p(\phi)\left[
\prod_{i=1}^N 
p(\mathbf{x}_i\mid \phi)
p(y_i\mid \mathbf{x}_i,\theta)\right]\left[
\,
p(\mathbf{x}_j\mid \phi)
p(\mathbf{x}'_j\mid \mathbf{x}_j,\theta)
p(y_j\mid \mathbf{x}_j,\theta)\right].
\end{equation*}
This factorization formalizes our invariance assumption: the adversary affects the observed covariates through the channel $p(\mathbf{x}'_j\mid \mathbf{x}_j,\theta)$, but the conditional distribution of the label remains $p(y_j\mid \mathbf{x}_j,\theta)$.

The reactive defense computes the robust PPD $p(y_j\mid \mathbf{x}'_j,\mathcal D)$. A full Bayesian treatment requires marginalizing over all unobserved quantities,
\begin{equation} \label{eq:op_predictive}
p(y_j\mid \mathbf{x}'_j,\mathcal{D})
=
\iiint
p(y_j\mid \mathbf{x}_j,\theta)
p(\mathbf{x}_j,\theta,\phi\mid \mathbf{x}'_j,\mathcal{D})
\,d\mathbf{x}_j\,d\theta\,d\phi .
\end{equation}
This expression depends on the joint posterior $p(\mathbf{x}_j,\theta,\phi\mid \mathbf{x}'_j,\mathcal{D})$, hence coupling the latent clean input $\mathbf{x}_j$ with the global parameters $(\theta,\phi)$ and the full training dataset $\mathcal D$. Proposition \ref{prop:reactive_ppd} rewrites this PPD in two equivalent ways: first as a nested expectation that separates latent-input inference from parameter updating, and then as a ratio of expectations under the clean-data posterior $p(\theta,\phi\mid\mathcal D)$. These enable practical computational approaches.

\begin{restatable}[Robust reactive PPD]{proposition}{reactiveppd}\label{prop:reactive_ppd}
Under the reactive model in Figure~\ref{fig:operations}, the robust PPD \eqref{eq:op_predictive} can be written as
\begin{equation}\label{eq:exact_reactive_nested}
p(y_j\mid \mathbf{x}'_j,\mathcal D)
=
\mathbb{E}_{(\theta,\phi)\mid \mathbf{x}'_j,\mathcal D}
\left[
\mathbb{E}_{\mathbf{x}_j\mid \mathbf{x}'_j,\theta,\phi}
\left(
p(y_j\mid \mathbf{x}_j,\theta)
\right)
\right].
\end{equation}
Let
\begin{equation*}
m_{\theta,\phi}(\mathbf{x}'_j)
=
\int 
p(\mathbf{x}'_j\mid \mathbf{x}_j,\theta)
p(\mathbf{x}_j\mid \phi)
\,d\mathbf{x}_j
=
\mathbb{E}_{\mathbf{x}_j\mid\phi}
\left[p(\mathbf{x}'_j\mid \mathbf{x}_j,\theta)\right]
\end{equation*}
be the marginal channel likelihood of the corrupted input under parameters $(\theta,\phi)$. Then, the PPD \eqref{eq:exact_reactive_nested} can be evaluated based on the clean-data posterior through
\begin{equation}\label{eq:exact_reactive_ratio}
p(y_j\mid \mathbf{x}'_j,\mathcal D)
=
\frac{
\mathbb{E}_{(\theta,\phi)\mid \mathcal D}
\left[
\mathbb{E}_{\mathbf{x}_j\mid \phi}
\left\{
p(y_j\mid \mathbf{x}_j,\theta)
p(\mathbf{x}'_j\mid \mathbf{x}_j,\theta)
\right\}
\right]
}{
\mathbb{E}_{(\theta,\phi)\mid \mathcal D}
\left[
m_{\theta,\phi}(\mathbf{x}'_j)
\right]
}.
\end{equation}
\end{restatable}

Proposition~\ref{prop:reactive_ppd} shows that the reactive defense performs two Bayesian updates at test time. For fixed $(\theta,\phi)$, it infers the latent clean input $\mathbf{x}_j$ from the corrupted observation $\mathbf{x}'_j$ through
$p(\mathbf{x}_j\mid \mathbf{x}'_j,\theta,\phi)
\propto
p(\mathbf{x}'_j\mid \mathbf{x}_j,\theta)p(\mathbf{x}_j\mid \phi)$. Across parameter values, it also updates the posterior over the global parameters through the marginal channel likelihood $m_{\theta,\phi}(\mathbf{x}'_j)$, since $p(\theta,\phi\mid \mathbf{x}'_j,\mathcal D) \propto
m_{\theta,\phi}(\mathbf{x}'_j)p(\theta,\phi\mid \mathcal D)$.

Attempting to approximate \eqref{eq:op_predictive} directly, for instance by sampling from the full joint posterior $p(\mathbf{x}_j,\theta,\phi\mid \mathbf{x}'_j,\mathcal D)$, is often intractable. First, the joint space of covariates $\mathbf{x}_j$ and parameters $(\theta,\phi)$ is typically high-dimensional, posing a challenge for MCMC. Second, if the adversarial channel $p(\mathbf{x}'_j\mid \mathbf{x}_j,\theta)$ is only accessible through simulation, as with many black-box or optimization-based attacks, the problem becomes one of likelihood-free inference. The ratio form \eqref{eq:exact_reactive_ratio} avoids sampling from the joint posterior, as it only involves expectations under the clean-data posterior $p(\theta,\phi\mid\mathcal D)$, which can be approximated before deployment. However, it exposes another bottleneck: evaluating it requires expectations under the covariate model $p(\mathbf{x}_j\mid\phi)$, evaluations of the channel density $p(\mathbf{x}'_j\mid\mathbf{x}_j,\theta)$, and the marginal channel likelihood $m_{\theta,\phi}(\mathbf{x}'_j)$. Learning and integrating over a generative model for $p(\mathbf{x}\mid\phi)$ is itself difficult, especially with high-dimensional covariates.

A simple way to alleviate this bottleneck is to replace the generative model $p(\mathbf{x}\mid \phi)$ by the empirical distribution of the training inputs, a non-parametric approximation that avoids learning an explicit, high-dimensional density. After this replacement, $\phi$ no longer appears in the reactive approximation. If, in addition, the posterior over $\theta$ is represented by samples $\{\theta_s\}_{s=1}^S$ from $p(\theta\mid\mathcal D)$, the exact expression \eqref{eq:exact_reactive_ratio} leads to what we call the \emph{empirical online} reactive defense, as it retains the test-time update of the global parameters,
\begin{equation}\label{onPure}
p_{\mathrm{ON}}(y_j\mid\mathbf{x}'_j,\mathcal D)
\approx
\frac{
\sum_{s=1}^S \sum_{i=1}^N
p(y_j\mid \mathbf{x}_i,\theta_s)
p(\mathbf{x}'_j\mid \mathbf{x}_i,\theta_s)
}{
\sum_{s'=1}^S \sum_{k=1}^N
p(\mathbf{x}'_j\mid \mathbf{x}_k,\theta_{s'})
}.
\end{equation}
To interpret this expression, define, for each posterior sample $\theta_s$, the marginal likelihood it assigns to the corrupted test input and the corresponding purified prediction,
\begin{equation*}
L_s(\mathbf{x}'_j)
=
\sum_{i=1}^N
p(\mathbf{x}'_j\mid \mathbf{x}_i,\theta_s),
\qquad
\widehat p_s(y_j\mid \mathbf{x}'_j)
=
\frac{1}{L_s(\mathbf{x}'_j)}
\sum_{i=1}^N
p(y_j\mid \mathbf{x}_i,\theta_s)
\,p(\mathbf{x}'_j\mid \mathbf{x}_i,\theta_s).
\end{equation*}%
Then, \eqref{onPure} can be rewritten as
\begin{equation*}
p_{\mathrm{ON}}(y_j\mid\mathbf{x}'_j,\mathcal D)
\approx
\sum_{s=1}^S
\frac{
L_s(\mathbf{x}'_j)
}{
\sum_{r=1}^S L_r(\mathbf{x}'_j)
}
\widehat p_s(y_j\mid \mathbf{x}'_j).
\end{equation*}
Thus, the online defense reweights the prediction from each posterior sample by the marginal likelihood that this sample assigns to the observed corrupted input. In this sense, it performs a Bayesian update at test time, giving more influence to parameter values that better explain $\mathbf{x}'_j$ under the assumed adversarial channel.

A simpler approximation emerges by suppressing this test-time update of the global parameters. Specifically, if the marginal likelihood $m_{\theta,\phi}(\mathbf{x}'_j)$ is approximately constant over the region where $p(\theta,\phi\mid\mathcal D)$ places most of its mass, then
\[
p(\theta,\phi\mid \mathbf{x}'_j,\mathcal D)
\propto
m_{\theta,\phi}(\mathbf{x}'_j)
p(\theta,\phi\mid\mathcal D)
\approx
p(\theta,\phi\mid\mathcal D).
\]
This leads to the empirical \emph{offline} reactive defense
\begin{equation*}
p_{\mathrm{OFF}}(y_j\mid \mathbf{x}'_j,\mathcal D)
=
\mathbb{E}_{(\theta,\phi)\mid \mathcal D}
\left[
\mathbb{E}_{\mathbf{x}_j\mid \mathbf{x}'_j,\theta,\phi}
\left\{
p(y_j\mid \mathbf{x}_j,\theta)
\right\}
\right],
\end{equation*}
which can be written as
\begin{equation*}
p_{\mathrm{OFF}}(y_j\mid \mathbf{x}'_j,\mathcal D)
=
\mathbb{E}_{(\theta,\phi)\mid \mathcal D}
\left[
\mathbb{E}_{\mathbf{x}_j\mid \phi}
\left\{
\frac{
p(\mathbf{x}'_j\mid \mathbf{x}_j,\theta)
p(y_j\mid \mathbf{x}_j,\theta)
}{
\mathbb{E}_{\tilde{\mathbf{x}}\mid \phi}
[
p(\mathbf{x}'_j\mid \tilde{\mathbf{x}},\theta)
]
}
\right\}
\right].
\end{equation*}
Applying the same empirical approximation to the clean covariate distribution yields
\begin{equation}\label{offPure}
p_{\mathrm{OFF}}(y_j\mid \mathbf{x}'_j,\mathcal D)
\approx
\frac{1}{S}
\sum_{s=1}^S
\sum_{i=1}^N
w_{si}
p(y_j\mid \mathbf{x}_i,\theta_s),
\end{equation}
where
\begin{equation*}
w_{si}
=
\frac{
p(\mathbf{x}'_j\mid \mathbf{x}_i,\theta_s)
}{
\sum_{k=1}^N
p(\mathbf{x}'_j\mid \mathbf{x}_k,\theta_s)
}.
\end{equation*}
The weights $w_{si}$ are normalized over the training inputs for each fixed posterior sample $\theta_s$. Hence, unlike the online defense, the offline one treats all posterior samples as equally likely after observing $\mathbf{x}'_j$. It only performs purification of the latent input conditional on each parameter sample and does not reweight the parameter samples themselves. This makes the offline defense less faithful to the exact Bayesian model, but useful as a benchmark as well as a diagnostic to separate the effect of input purification from the effect of test-time parameter reweighting.

Both reactive approximations are closely related to \emph{adversarial purification}. In our formulation, purification corresponds to inferring the full posterior distribution $p(\mathbf{x}_j\mid \mathbf{x}'_j,\theta,\phi)$ over the latent clean input. Most existing purification methods can be viewed as non-Bayesian counterparts of this operation: instead of propagating the full posterior uncertainty, they approximate it with a point mass at a single restored estimate $\widehat{\mathbf{x}}_j$. Such methods are typically \emph{model-agnostic} when the purification rule is independent of the downstream predictive model parameters $\theta$, and \emph{model-guided} when $\theta$ is used to shape the restoration process. Beyond adversarial purification, the reactive formulation also encompasses RS \citep{cohen2019certified}, which emerges as a special case under a simple Gaussian channel and a suitable limiting choice of the latent-input prior, as Section \ref{subsec:at} shows.

The empirical defenses avoid learning a generative covariate model, but they still carry test-time costs. Both $p_{\mathrm{ON}}$ and $p_{\mathrm{OFF}}$ require storing the clean training inputs, or an adequate approximation to their distribution, and evaluating the channel density $p(\mathbf{x}'\mid \mathbf{x},\theta)$ at test time. For simple noise channels, this density is available analytically and the weights in \eqref{onPure} and \eqref{offPure} can be computed directly. However, for many realistic attack-based channels, especially those defined through iterative optimization, the channel may be available only as a simulator. In that case, the Bayesian target remains well-defined through the latent-variable model, but the weights in \eqref{onPure}--\eqref{offPure} cannot be evaluated directly. Implementing the reactive defense then requires likelihood-free or simulation-based inference methods \citep{cranmer2020frontier}, ranging from Approximate Bayesian Computation \citep{gallego2024protecting} to more scalable neural density-estimation schemes \citep{papamakarios2019sequential}. Developing such methods for high-dimensional adversarial channels is beyond the scope of this work. This limitation partially motivates the proactive defense introduced next, which avoids test-time density evaluations and only requires sampling from the adversarial~channel~during~training.

%%%%%%%%%%%%%%%%%%%%%%%%%%%%%%%%%%%%%%%%%%%%%%%%%%%%%%%%%%%%%%%
\subsection{Protection During Training}\label{sec:protect-train}

The proactive defense shifts the computational burden from deployment to training. There, it confronts the model with a hypothetical adversary, described by the channel, so that robustness is built into the posterior itself. Thus, the key modeling change is to introduce, for each clean training input $\mathbf{x}_i$, a latent adversarially corrupted version $\mathbf{x}'_i$ generated by the channel. The observed label $y_i$ is then assumed to be generated from this corrupted input, as Figure~\ref{fig:train} portrays. Training therefore integrates over the attacks that could affect each observation, whereas prediction later uses the robust posterior directly.

\begin{figure}[h!]
\centering
\begin{tikzpicture}[>=stealth, node distance = 1.0cm and 1.6cm]

% --------------- global parameters ------------------------------------
\node[latent]                       (phi)   {$\phi$};
\node[latent, right=2.5cm of phi]   (theta) {$\theta$};

% ---------------- TRAINING plate (left) -------------------------------
\node[obs,    below=1.2cm of phi]   (xtr)   {$\mathbf{x}_i$};
\node[latent, below=0.8cm of xtr]   (xptr)  {$\mathbf{x}'_i$};
\node[obs,    right=1.8cm of xptr]  (ytr)   {$y_i$};

\edge {phi}  {xtr}
\edge {xtr}  {xptr}
\path (xptr)  edge[->, bend left =  0] (ytr);
\path (theta) edge[->, bend right = 0] (ytr);
\path (theta) edge[->, bend right = 0] (xptr);

\plate {train}{(xtr)(xptr)(ytr)} {$i = 1,\dots,N$}
\node[anchor=south west, inner sep=6pt, font=\scriptsize] at (train.south west) {Train};

% ---------------- TEST plate (right) ----------------------------------
\node[latent, right=3.5cm of xtr]  (xtest)   {$\mathbf{x}_j$};
\node[obs,    below=0.8cm of xtest] (xptest) {$\mathbf{x}'_j$};
\node[latent, right=1.8cm of xptest] (ytest) {$y_j$};

\path (phi) edge[->, bend left = 0] (xtest);
\edge {xtest} {xptest}
\path (xptest) edge[->, bend left =  0] (ytest);
\path (theta)  edge[->, bend right = 0] (ytest);
\path (theta) edge[->, bend right = 20] (xptest);

\plate {test}{(xtest)(xptest)(ytest)} {$j = 1,\dots,M$}
\node[anchor=south west, inner sep=6pt, font=\scriptsize] at (test.south west) {Test};

\end{tikzpicture}
\caption{PGM \mbox{for proactive defense. Shaded, observed variables. Unshaded, latent variables.}}
\label{fig:train}
\end{figure}

Under Figure~\ref{fig:train}, the training-data joint density factorizes as
\begin{equation}\label{eq:proactive_joint}
\begin{aligned}
&p(\theta,\phi,\mathcal D,\mathbf{x}'_{1:N})  =
p(\theta)p(\phi)
\prod_{i=1}^N
p(\mathbf{x}_i\mid\phi)
p(\mathbf{x}'_i\mid\mathbf{x}_i,\theta)
p(y_i\mid\mathbf{x}'_i,\theta).
\end{aligned}
\end{equation}
Proposition \ref{prop:proactive_factorization} formalizes the resulting robust posterior and the main computational simplification relative to the reactive model, enabling practical approaches.

\begin{restatable}[Robust proactive posterior distribution]{proposition}{proactivefactorization}\label{prop:proactive_factorization}
Under the proactive model in \eqref{eq:proactive_joint}, define the channel-marginalized likelihood contribution
\begin{equation*}
r_\theta(y_i\mid\mathbf{x}_i)
=
\int
p(y_i\mid\mathbf{x}'_i,\theta)
p(\mathbf{x}'_i\mid\mathbf{x}_i,\theta)
\,d\mathbf{x}'_i
=
\mathbb{E}_{\mathbf{x}'_i\mid\mathbf{x}_i,\theta}
\left[
p(y_i\mid\mathbf{x}'_i,\theta)
\right].
\end{equation*}
Then, the posterior factorizes as
\begin{equation}\label{eq:proactive_posterior}
p(\theta,\phi\mid\mathcal D)
=
p_{\mathrm{R}}(\theta\mid\mathcal D)\,
p(\phi\mid\mathbf{x}_{1:N}),
\end{equation}
where
\begin{equation}\label{eq:robust_theta_posterior}
p_{\mathrm{R}}(\theta\mid\mathcal D)
\propto
p(\theta)
\prod_{i=1}^N
r_\theta(y_i\mid\mathbf{x}_i),
\qquad
p(\phi\mid\mathbf{x}_{1:N})
\propto
p(\phi)
\prod_{i=1}^N
p(\mathbf{x}_i\mid\phi).
\end{equation}
Moreover, under a no-test-update approximation
$p(\theta\mid\mathbf{x}'_j,\mathcal D)\approx p_{\mathrm{R}}(\theta\mid\mathcal D)$, the proactive PPD for an observed (potentially) corrupted input is
\begin{equation}\label{eq:proactive_ppd}
p_{\mathrm{R}}(y_j\mid\mathbf{x}'_j,\mathcal D)
\approx
\mathbb{E}_{\theta\mid\mathcal D}
\left[
p(y_j\mid\mathbf{x}'_j,\theta)
\right],
\end{equation}
where the expectation is with respect to $p_{\mathrm{R}}(\theta\mid\mathcal D)$.
\end{restatable}

Proposition~\ref{prop:proactive_factorization} makes this shift of computational cost precise: the test-time expression \eqref{eq:proactive_ppd} has the form of a standard PPD, but with the standard posterior over $\theta$ replaced by the robust proactive posterior $p_{\mathrm{R}}(\theta\mid\mathcal D)$. The robust predictive model only depends on $\theta$ through this posterior; the covariate model $p(\mathbf{x}\mid\phi)$ factors out of the prediction. Consequently, at test time it requires neither purification, nor access to the training data, nor a generative model for clean covariates. This is the practical advantage of the proactive formulation.% over the reactive one.

Note that this formulation generalizes standard AT, recovering it by adopting a deterministic point-mass adversarial channel and a point estimate of $\theta$, as Section~\ref{subsec:at} shows. However, and importantly, the proactive Bayesian model replaces this single worst-case attack with a stochastic channel, thereby encoding uncertainty about the attacker's strategy. As the experiments demonstrate, training against a distribution of attacks can even yield robustness against attack modalities not seen during training, having a regularizing effect.

The remaining difficulty is computational: the robust posterior \eqref{eq:robust_theta_posterior} contains one channel integral per observation which, in general, is not available in closed form. We therefore approximate $p_{\mathrm{R}}(\theta\mid\mathcal D)$ through VI \citep{blei2017variational}, using a tractable family $q_\psi(\theta)$ parameterized by $\psi$. The exact evidence lower bound (ELBO) for the robust posterior is
\begin{equation}\label{ELBONITA}
\mathcal{L}(\psi)
=
\mathbb{E}_{\theta \sim q_\psi}
\left[
\sum_{i=1}^N
\log
\mathbb{E}_{\mathbf{x}'_i\mid\mathbf{x}_i,\theta}
\left\{
p(y_i\mid\mathbf{x}'_i,\theta)
\right\}
\right]
-
\mathrm{KL}\{q_\psi(\theta)\|p(\theta)\}.
\end{equation}
One can estimate the gradient of \eqref{ELBONITA} directly through a ratio-of-means Monte Carlo estimator (see Appendix~\ref{app:unbiased_grad}). In practice, however, we optimize a simpler lower bound, referred to as the surrogate ELBO (SELBO)
\begin{equation}\label{eq:surrogate_elbo}
\tilde{\mathcal{L}}(\psi)
=
\mathbb{E}_{\theta \sim q_\psi}
\left[
\sum_{i=1}^N
\mathbb{E}_{\mathbf{x}'_i\mid\mathbf{x}_i,\theta}
\left\{
\log p(y_i\mid\mathbf{x}'_i,\theta)
\right\}
\right]
-
\mathrm{KL}\{q_\psi(\theta)\|p(\theta)\}.
\end{equation}
Note that $\tilde{\mathcal{L}}(\psi)\le \mathcal{L}(\psi)$ follows by applying Jensen's inequality to each
$\log \mathbb{E}_{\mathbf{x}'_i\mid\mathbf{x}_i,\theta}\{p(y_i\mid\mathbf{x}'_i,\theta)\}$ term.
This bound introduces a gap that reflects the variability of the likelihood over the adversarial channel. Tighter objectives, including importance-weighted bounds that draw several channel samples per observation \citep{burda2015importance}, are possible. Yet, the gap need not be damaging in practice. In our MNIST experiments (Appendix~\ref{app:tightness}), the single-sample SELBO attained robustness comparable to these tighter bounds at substantially lower cost. Moreover, increasing the number of channel samples, which reduces gradient variance without changing the gap, further improved robustness. This suggests that gradient variance, rather than the Jensen gap, is the dominant computational concern, and is the reason why we use the SELBO in our experiments.
Interestingly, this objective \eqref{eq:surrogate_elbo} is the variational form of a generalized Bayes posterior \citep{knoblauch2022optimization}: each observation contributes a channel-averaged loss, and the unconstrained optimizer is the corresponding Gibbs posterior.
Treating the SELBO as a functional $\tilde{\mathcal L}(q)$ over arbitrary densities $q$ (with the parameterized objective as its restriction $\tilde{\mathcal L}(\psi)=\tilde{\mathcal L}(q_\psi)$), we formalize this below.

\begin{restatable}[Unconstrained solution of SELBO]{proposition}{genpos}\label{prop:generalized_posterior}
Let $\Theta$ be the parameter space, and $\mathcal{Q}$ be the set of probability densities on $\Theta$ that are absolutely continuous with respect to the prior $p(\theta)$ and have finite KL divergence from it. For $i=1,\ldots,N$, define $\bar\ell_i(\theta)
=
-
\mathbb{E}_{\mathbf{x}'_i\mid\mathbf{x}_i,\theta}
\left\{
\log p(y_i\mid\mathbf{x}'_i,\theta)
\right\}$ and let $\bar L_{\mathcal D}(\theta)=\sum_{i=1}^N\bar\ell_i(\theta)$. Assume that
$ Z(\mathcal D) = \int_\Theta p(\theta) \exp\{-\bar L_{\mathcal D}(\theta)\} \,d\theta$
satisfies $0<Z(\mathcal D)<\infty$, and that the density $q^*$ defined in \eqref{eq:generalized_posterior} below satisfies $\mathrm{KL}\{q^*(\theta\mid\mathcal D)\|p(\theta)\}<\infty$, so that $q^*\in\mathcal Q$. Then, for any $q\in\mathcal Q$,
\begin{equation}\label{eq:surrogate_kl_identity}
\tilde{\mathcal L}(q)
=
\log Z(\mathcal D)
-
\mathrm{KL}\{q(\theta)\|q^*(\theta\mid\mathcal D)\},
\end{equation}
where
\begin{equation}\label{eq:generalized_posterior}
q^*(\theta\mid\mathcal D)
=
\frac{
p(\theta)\exp\left\{-\sum_{i=1}^N\bar\ell_i(\theta)\right\}
}{
Z(\mathcal D)
}.
\end{equation}
Hence $\tilde{\mathcal L}(q)\le \log Z(\mathcal D)$, with equality if and only if $q=q^*$ almost everywhere. Consequently, the unconstrained maximizer of the SELBO over $\mathcal Q$ is the Gibbs posterior $q^*(\theta\mid\mathcal D)$.
\end{restatable}

Note that the finiteness conditions hold automatically in classification, where $\bar\ell_i\ge 0$, implying $Z(\mathcal D)\le 1$ and $\mathrm{KL}\{q^*(\theta\mid\mathcal D)\|p(\theta)\}\le -\log Z(\mathcal D)$.

The SELBO \eqref{eq:surrogate_elbo} is a double expectation and is, therefore, amenable to stochastic optimization. Operationally, this may be done through a black-box stochastic VI procedure \citep{ranganath2014black,blei2017variational}. If the variational posterior is reparameterizable, $\theta=f_\psi(\boldsymbol{\epsilon})$ with $\boldsymbol{\epsilon}\sim p(\boldsymbol{\epsilon})$, and the adversarial channel is also reparameterizable,\footnote{This condition is met by the attack-based channels we consider, since they add reparameterizable noise to a deterministic attack output.} $\mathbf{x}'=h(\mathbf{x},\theta,\boldsymbol{\eta})$ with $\boldsymbol{\eta}\sim p(\boldsymbol{\eta})$, then the pathwise gradient of the variational objective is
\begin{equation*}
\nabla_\psi \tilde{\mathcal{L}}(\psi)
=
\mathbb{E}_{\boldsymbol{\epsilon}}
\left[
\sum_{i=1}^N
\mathbb{E}_{\boldsymbol{\eta}}
\left\{
\nabla_\psi
\log p\left(
y_i
\mid
h(\mathbf{x}_i,f_\psi(\boldsymbol{\epsilon}),\boldsymbol{\eta}),
f_\psi(\boldsymbol{\epsilon})
\right)
\right\}
\right]
-
\nabla_\psi
\mathrm{KL}\{q_\psi(\theta)\|p(\theta)\}.
\end{equation*}
The first term can be estimated unbiasedly by Monte Carlo (under the regularity conditions discussed at the end of Section~\ref{subsec:channels}), while the KL gradient term is available in closed form for common variational families, yielding an efficient procedure for learning the variational posterior used in place of $p_{\mathrm{R}}(\theta\mid\mathcal D)$ in \eqref{eq:proactive_ppd}.

\subsection{Adversarial Channels}\label{subsec:channels}

%As anticipated, t
% {\color{revtwo}Both defenses developed above operate through an adversarial channel; we now make this object concrete.} The specification of the adversarial channel $p(\mathbf{x}'\mid\mathbf{x},\theta)$ 
Both defenses developed above take as given the adversarial channel $p(\mathbf{x}'\mid\mathbf{x},\theta)$ introduced in Section~\ref{subsec:problem}; we now make it concrete. It is important to distinguish these channels from the separate attacks used to evaluate robustness: the former encode the corruptions we anticipate and defend against, whereas the latter probe the trained model.
To build intuition, we present a few representative channels, ordered by increasing sophistication of the assumed adversary. These are illustrative choices rather than an exhaustive catalogue: they showcase the flexibility of the framework and are the channels we later instantiate in our experiments, although other constructions are possible. 
 Additional implementation details are collected in Appendix~\ref{app:channel}.

\paragraph*{Baseline noise channel.}
The simplest, \emph{attack-agnostic} channel introduces no adversarial objective, returning the clean input up to an isotropic Gaussian augmentation,
\begin{equation*}
p_{\mathrm{id},\sigma}(\mathbf{x}'\mid\mathbf{x},\theta)
=
\mathcal N(\mathbf{x}';\mathbf{x},\sigma^2 I),
\end{equation*}
which depends on neither $y$ nor $\theta$, recovering the identity mapping as $\sigma\to 0$. Beyond serving as a baseline, this channel is the non-adversarial ingredient of the mixture channels below and, as established by the RS limit in Section~\ref{subsec:at}, it is precisely the channel under which the reactive defense reduces to RS.

\paragraph*{Attack-based channels.}

A richer family of channels is obtained by rendering deterministic classical AML attacks probabilistic. Typically, these attacks operate by maximizing an adversarial score that depends on the true response $y$. 

Up to this point, our formulation of the adversarial channel $p(\mathbf{x}'\mid\mathbf{x},\theta)$ has assumed that the adversary's actions do not depend on the label. We argue that this independence is a much more realistic threat model at deployment. Moreover, it guarantees that the proactive training integral $\int p(y\mid\mathbf{x}',\theta)\,p(\mathbf{x}'\mid\mathbf{x},\theta)\,d\mathbf{x}'$ normalizes correctly over $y$, preserving the strict Bayesian generative model in Section~\ref{sec:protect-train}. However, to adapt our framework to the vast majority of classical attacks in the literature, which presume an adversary who knows the label, we can accommodate label-dependent channels $p(\mathbf{x}'\mid\mathbf{x},y,\theta)$ during supervised training. Figure~\ref{fig:train}, which depicts the label-independent case covered in Proposition~\ref{prop:proactive_factorization}, then ceases to apply: it generates $y_i$ from $\mathbf{x}'_i$, so a channel that reads $y_i$ closes a cycle between the two, and the proactive model admits no directed factorization. Accordingly, the channel integral need not normalize over the response, so \eqref{eq:robust_theta_posterior} no longer has, in general, the ordinary-posterior interpretation supplied by the proactive model. What our variational scheme fits is in any case not \eqref{eq:robust_theta_posterior} itself, but the \emph{generalized posterior} \citep{bissiri2016general} built from the channel-averaged loss
$
\bar\ell_i(\theta)
=
\mathbb{E}_{\mathbf{x}'_i\mid\mathbf{x}_i,y_i,\theta}
\left[-\log p(y_i\mid\mathbf{x}'_i,\theta)\right]
$,
namely
\begin{equation}\label{eq:gen_posterior_eta}
p_\eta(\theta\mid\mathcal D)
\propto
p(\theta)
\exp\left\{
-\eta\sum_{i=1}^N\bar\ell_i(\theta)
\right\},
\end{equation}
with $\eta>0$ a learning rate. At $\eta=1$, this is the target identified in Proposition~\ref{prop:generalized_posterior}. It differs from \eqref{eq:robust_theta_posterior} by the same Jensen gap that separates the SELBO from the exact ELBO. Were one unwilling to assume that the adversary knows the response, evaluating the score at the model's own prediction (rather than the actual label) would restore the ordinary case. We retain the observed response, so as to remain comparable with standard AML practice.

Label-dependent channels can also be accommodated within the
reactive model. In this case, the factorization
$p(\mathbf{x}\mid\phi)p(y\mid\mathbf{x},\theta)
p(\mathbf{x}'\mid\mathbf{x},y,\theta)$
remains acyclic after adding the edge $y\to\mathbf{x}'$, since
the label is generated from the clean input before the attack.
Prediction then requires joint inference over the latent clean
input and response. The ratio form
\eqref{eq:exact_reactive_ratio} still holds, with
$p(\mathbf{x}'_j\mid\mathbf{x}_j,y_j,\theta)$ in the numerator and
the label-averaged channel
$\sum_{y}p(y\mid\mathbf{x}_j,\theta)\,p(\mathbf{x}'_j\mid\mathbf{x}_j,y,\theta)$
(an integral over $y$ in regression) in place of
$p(\mathbf{x}'_j\mid\mathbf{x}_j,\theta)$ inside
$m_{\theta,\phi}(\mathbf{x}'_j)$, provided that the dominating measure of
Section~\ref{operations} does not depend on $y$ either. 

With this generalized framework in place, we denote the adversarial score by $s_\theta^{(a)}(\mathbf{z}; y)$, where $\mathbf{z}$ is a candidate perturbed input, $y$ is the response, and $a\in\{\mathrm{NLL},\mathrm{CW},\mathrm{ENT}\}$ indexes the three attack types we instantiate below. The likelihood-based score
$
s_\theta^{\mathrm{NLL}}(\mathbf{z};y)=-\log p(y\mid\mathbf{z},\theta)
$
drives the input towards misclassification, or towards large prediction errors in regression; it is the familiar cross-entropy objective under a categorical likelihood and, up to constants and scale, a squared-error loss under a homoscedastic Gaussian one. For classification, the margin score of \citet{carlini2017towards} (CW) is
$
s_\theta^{\mathrm{CW}}(\mathbf{z};y)
=
\max_{c\neq y}g_{\theta,c}(\mathbf{z})-g_{\theta,y}(\mathbf{z}),
$
with $g_{\theta,c}$ the logit of class $c$. Finally, the predictive-entropy score
$
s_\theta^{\mathrm{ENT}}(\mathbf{z})=-\sum_c p(c\mid\mathbf{z},\theta)\log p(c\mid\mathbf{z},\theta)
$
carries no response argument at all: it seeks inputs at which the model is maximally uncertain. Channels built on the ENT score are therefore label-independent, and preserve the strict Bayesian normalization discussed above.

Given such a score ${s}_\theta^{(a)}$, an attack-based channel moves the clean input along its ascent direction and then injects Gaussian noise, turning the deterministic attack into a proper distribution over corrupted inputs rather than a point mass. We consider two variants, differing only in how far the ascent is carried. The \emph{one-step} channel takes a single normalized gradient-ascent step, yielding the Gaussian kernel
\begin{equation*}
p_{\mathrm{OS},\sigma}(\mathbf{x}'\mid\mathbf{x},y,\theta)
=
\mathcal N\!\left(
\mathbf{x}';\;
\mathbf{x}
+
\epsilon\,
\frac{\nabla_{\mathbf{x}}s_\theta^{(a)}(\mathbf{x};y)}
{\left\|\nabla_{\mathbf{x}}s_\theta^{(a)}(\mathbf{x};y)\right\|_2},\;
\sigma^2 I
\right),
\end{equation*}
with the convention that the gradient term is zero whenever its norm vanishes, and with $y$ dropped throughout for label-independent scores. The \emph{multi-step} channel instead performs several projected gradient-ascent steps within an $L_2$ ball before the same Gaussian augmentation. Appendix~\ref{app:channel} provides the full projected-ascent recursion. The $L_2$ analogue of the fast gradient sign method (FGSM) and the $L_2$ projected gradient descent (PGD) attacks are, in the limit $\sigma\to 0$, the standard one-step and multi-step instances, respectively, of this construction applied to the likelihood score $s_\theta^{\mathrm{NLL}}$. 

\paragraph*{Mixture channels.}
To represent a more diverse, or partially unknown, adversary, as well as balance clean and attacked data, we combine component channels. Given kernels $p_k(\mathbf{x}'\mid\mathbf{x},y,\theta)$ and weights $\pi_k\ge 0$ with $\sum_{k=1}^K\pi_k=1$, the mixture channel is
\begin{equation*}
p_{\boldsymbol{\pi}}(\mathbf{x}'\mid\mathbf{x},y,\theta)
=
\sum_{k=1}^K \pi_k\, p_k(\mathbf{x}'\mid\mathbf{x},y,\theta),
\end{equation*}
which is sampled by first drawing a component $C\sim\mathrm{Categorical}(\pi_1,\ldots,\pi_K)$ and then $\mathbf{x}'\sim p_C(\cdot\mid\mathbf{x},y,\theta)$. 
As examples, one may pair the identity channel with a one-step attack through the two-component mixture $\rho\,p_{\mathrm{id},\sigma}+(1-\rho)\,p_{\mathrm{OS},\sigma}$, which places mass $\rho$ on clean augmentation, or blend several attack objectives to emulate an adversary whose strategy is itself uncertain. Both are instantiated in our experiments. In practice, such a mixture is computationally realized by assembling each mini-batch with the prescribed proportions of clean and attacked samples.

\paragraph*{Learned channels.}

Finally, rather than pre-specifying a fixed attack, the channel itself can be learned. A separate neural network (NN), parameterized by $\omega$, maps an input-label pair to a Gaussian perturbation distribution,
\begin{equation*}
p_{\omega}(\mathbf{x}'\mid\mathbf{x},y)
=
\mathcal N\!\left(
\mathbf{x}';\;
\mathbf{x}+\boldsymbol{\mu}_\omega(\mathbf{x},y),\;
\mathrm{diag}\!\big(\boldsymbol{\sigma}_\omega^2(\mathbf{x},y)\big)
\right),
\end{equation*}
whose mean shift $\boldsymbol{\mu}_\omega(\mathbf{x},y)$ and per-coordinate variances $\boldsymbol{\sigma}_\omega^2(\mathbf{x},y)$ are the network outputs. The generator parameters $\omega$ are trained in a generative-adversarial manner \citep{NIPS2014_5423}, by alternating updates that steer the perturbation towards high-loss regions with updates to the predictive model using the resulting samples. 
This construction is not tied to a prespecified attack family, although it introduces an auxiliary optimization problem and a further model specification of its own. Note that $p_\omega$ does not take $\theta$ as an argument: it is coupled to the predictive model only through this training, and $\omega$ is point-estimated rather than integrated over. Since training is a saddle-point procedure, the SELBO and its Gibbs-posterior characterization in Proposition~\ref{prop:generalized_posterior} apply only for a fixed $\omega$, for example at convergence. Appendix~\ref{app:channel} specifies the convolutional (image) and fully connected (tabular) generator architectures.

All four constructions are reparameterizable: a sample can be written as $\mathbf{x}'=h(\mathbf{x},\theta,\boldsymbol{\eta})$, with $y$ as an additional argument of $h$ for label-dependent channels and $\boldsymbol{\eta}$ drawn from a fixed base distribution, since each channel merely adds reparameterizable Gaussian noise to a (possibly attack-dependent) location. For mixtures, the base variable also includes the component indicator $C$ and, for a randomly started multi-step channel, the starting point $\mathbf{z}^0$. Since neither distribution depends on $\theta$, the representation remains valid. This is precisely the condition invoked for the pathwise gradient of the proactive objective in Section~\ref{sec:protect-train}. That gradient further requires $h$ to be differentiable in $\theta$ almost everywhere. For attack-based channels, this involves second derivatives of the score, and $h$ is discontinuous on some sets of parameter values (e.g., where the score gradient vanishes or the maximizing class in the CW score changes). The pathwise gradient then ignores these discontinuities and is, in general, biased.

%{\color{revtwo}Having made these channels concrete, we now show how several prominent existing defenses arise as limiting cases of our framework.}
%With these channels formally specified, we now show how several prominent existing defenses arise as limiting cases of our framework.

%%%%%%%%%%%%%%%%%%%%%%%%%%%%%%%%%%%%%%%%%%%%%%%%
\subsection{Recovering Existing Defenses}\label{subsec:at}

Both adversarially robust Bayesian constructions above recover several standard defenses as limiting cases. The common pattern is that a classical method is obtained by collapsing one or more Bayesian ingredients: replacing the posterior by a point estimate, the adversarial channel by a deterministic attack, or the posterior uncertainty over latent clean inputs by a fixed smoothing distribution. Let us make such connections explicit. Write $\ell_\theta(\mathbf{x},y) = - \log p(y\mid\mathbf{x},\theta)$
for the negative log-likelihood, and let $\mathcal B_\epsilon(\mathbf{x})$ denote the allowed perturbation set around $\mathbf{x}$.

\begin{restatable}[Adversarial training as a generalized-posterior maximum a posteriori (MAP) limit]{proposition}{atmaplimit}\label{prop:at_map_limit}
Let $\ell_\theta(\mathbf{z},y)=-\log p(y\mid\mathbf{z},\theta)$ and suppose that there exists a measurable map $(\mathbf{x},y,\theta)\mapsto A_\theta(\mathbf{x},y)$ such that
\(
A_\theta(\mathbf{x},y)
\in
\argmax_{\mathbf{z}\in\mathcal B_\epsilon(\mathbf{x})}
\ell_\theta(\mathbf{z},y)
\)
for every $(\mathbf{x},y,\theta)$.
Consider the deterministic, label-dependent channel
$
p^{\mathrm{AT}}(d\mathbf{x}'\mid\mathbf{x},y,\theta)
=
\delta_{A_\theta(\mathbf{x},y)}(d\mathbf{x}')
$.
The generalized posterior \eqref{eq:gen_posterior_eta} induced by this kernel is
\begin{equation*}
p_{\mathrm{AT},\eta}(\theta\mid\mathcal D)
\propto
p(\theta)
\exp\left\{
-\eta
\sum_{i=1}^N
\ell_\theta\left(A_\theta(\mathbf{x}_i,y_i),y_i\right)
\right\},
\qquad \eta>0 .
\end{equation*}
Under a flat prior, every MAP estimator of $p_{\mathrm{AT},\eta}$ solves the standard AT problem
\begin{equation}\label{eq:at_objective}
\theta_{\mathrm{AT}}
\in
\argmin_\theta
\sum_{i=1}^N
\max_{\mathbf{z}\in\mathcal B_\epsilon(\mathbf{x}_i)}
\ell_\theta(\mathbf{z},y_i).
\end{equation}
\end{restatable}

\noindent Thus, standard AT is recovered from the proactive generalized-Bayes construction through two degeneracies: the adversarial channel places all probability mass on a single worst-case perturbation, and posterior inference is replaced by a MAP estimate. Our proactive Bayesian defense relaxes both degeneracies by allowing a stochastic channel and retaining posterior uncertainty over $\theta$. Note that the flat prior matters: under a Gaussian prior, for instance, the MAP problem would add a weight-decay term to \eqref{eq:at_objective}. Moreover, the argument is not tied to a particular geometry: it holds for any perturbation set $\mathcal B_\epsilon(\mathbf{x})$ on which the maxima are attained by a measurable maximizer.

Interestingly, as an intermediate case, relaxing only the MAP degeneracy (i.e., retaining posterior uncertainty while keeping the deterministic worst-case channel) recovers the adversarially robust Gibbs posterior proposed by \citet{sabanayagam2025generalization}, in the Bayesian linear regression domain. Specifically, retaining the deterministic kernel $p^{\mathrm{AT}}(d\mathbf{x}'_i\mid\mathbf{x}_i,y_i,\theta)$ of Proposition~\ref{prop:at_map_limit} and setting $\eta=1$ in \eqref{eq:gen_posterior_eta} yields
\begin{equation*}
p_{\mathrm{AT},1}(\theta\mid\mathcal D)
\propto
p(\theta)
\exp\left\{-\sum_{i=1}^N \ell_\theta\left(A_\theta(\mathbf{x}_i,y_i), y_i\right) \right\},
\end{equation*}
which precisely mirrors their robust posterior formulation.

Other modern defenses can be recovered as variations of our framework. While standard AT relies entirely on the likelihood under attack, many recent defenses introduce auxiliary regularization terms to enforce smoother decision boundaries. In a Bayesian paradigm, regularization naturally corresponds to encoding structural prior beliefs about model behavior \citep{bishop2006pattern}. However, adversarial robustness is fundamentally a \textit{local} property, dictating that model output should remain invariant specifically within the neighborhoods of the training data $\mathcal{D}$. A traditional prior $p(\theta)$ cannot capture this geometry, as it should be strictly independent of the observed data. To rigorously connect our framework with these modern defenses, consider the generalized proactive posterior \citep{knoblauch2022optimization} that adds to \eqref{eq:gen_posterior_eta}, at $\eta=1$, a data-dependent penalty $\Omega(\theta,\mathcal D)\ge 0$ encoding a belief in local adversarial smoothness,
\begin{equation}\label{eq:penalized_generalized_posterior}
p_{\Omega}(\theta\mid\mathcal D)
\propto
p(\theta)
\exp\left\{
-
\sum_{i=1}^N
\mathbb{E}_{\mathbf{x}'_i\mid\mathbf{x}_i,y_i,\theta}
\left[
\ell_\theta(\mathbf{x}'_i,y_i)
\right]
-
\Omega(\theta,\mathcal D)
\right\}.
\end{equation}
Its MAP estimate is obtained by solving %minimizing:
\begin{equation*}
\argmin_\theta
\left\{
\sum_{i=1}^N
\mathbb{E}_{\mathbf{x}'_i\mid\mathbf{x}_i,y_i,\theta}
\left[
\ell_\theta(\mathbf{x}'_i,y_i)
\right]
+
\Omega(\theta,\mathcal D)
-
\log p(\theta)
\right\}.
\end{equation*}
This gives a formal generalized-Bayes interpretation of smoothness-regularized adversarial defenses. Note that \eqref{eq:penalized_generalized_posterior} weights the channel-averaged loss at $\eta=1$, the relative weight of the penalty being absorbed into $\Omega$ itself through the coefficients $\lambda$ and $\beta$ below. Calibrating both weights, as the generalized-Bayes literature requires \citep{knoblauch2022optimization}, is a separate question that we do not pursue here. 

The ALP and TRADES objectives are also recovered through different choices of
channel and penalty, under a flat parameter prior. For ALP, let
$a_\theta(\mathbf{x},y)$ be the attack rule used to construct each
adversarial counterpart and take
\[
p^{\mathrm{ALP}}(d\mathbf{x}'\mid\mathbf{x},y,\theta)
=
\tfrac12\delta_{\mathbf{x}}(d\mathbf{x}')
+
\tfrac12\delta_{a_\theta(\mathbf{x},y)}(d\mathbf{x}').
\]
The channel-averaged loss is therefore
\[
\frac12\sum_{i=1}^N
\left[
\ell_\theta(\mathbf{x}_i,y_i)
+
\ell_\theta(a_\theta(\mathbf{x}_i,y_i),y_i)
\right].
\]
Together with the logit-pairing penalty
\[
\Omega_{\mathrm{ALP}}(\theta,\mathcal{D})
=
\lambda\sum_{i=1}^N
\left\|
g_\theta(\mathbf{x}_i)
-
g_\theta(a_\theta(\mathbf{x}_i,y_i))
\right\|_2^2,
\]
this reproduces the ALP objective \citep{kannan2018adversarial} exactly, up to a rescaling of $\lambda$ that absorbs the sum or
average convention: the equal mixture weights average the clean
and adversarial losses, as the 50/50 mini-batches of ALP do. In the idealized worst-case formulation,
$a_\theta=A_\theta$ and the attacked component is
$p^{\mathrm{AT}}$. The ALP channel thus differs from that of
Proposition~\ref{prop:at_map_limit} only through its clean component,
which is what makes its likelihood term average the clean and
attacked losses.

For TRADES, use the identity channel
$p(d\mathbf{x}'\mid\mathbf{x},y,\theta)
=\delta_{\mathbf{x}}(d\mathbf{x}')$, that is, the $\sigma\to 0$
limit of $p_{\mathrm{id},\sigma}$,
and the penalty
\[
\Omega_{\mathrm{TRADES}}(\theta,\mathcal{D})
=
\beta\sum_{i=1}^N
\max_{\mathbf{z}\in\mathcal B_\epsilon(\mathbf{x}_i)}
\mathrm{KL}\!\left\{
p(\cdot\mid\mathbf{x}_i,\theta)
\,\middle\|\,
p(\cdot\mid\mathbf{z},\theta)
\right\}.
\]
The resulting MAP objective is the clean negative log-likelihood
plus the TRADES penalty \citep{zhang2019theoretically}, that is,
the TRADES objective itself, with the KL divergence taken in the
direction used by the TRADES implementation.
These constructions recover the corresponding training objectives
under matched attack and normalization conventions (weight decay in their costs, if any, corresponds to a Gaussian rather than a flat prior). Practical
implementations approximate the inner maximizations.

%
% With the deterministic channel fixed in this way, substituting these specific penalties into the generalized MAP objective reproduces the corresponding adversarial loss plus regularizer.
% \rev{These constructions reproduce the form of the ALP and TRADES objectives, a robust likelihood term plus a smoothness penalty, rather than the methods verbatim; TRADES, in particular, evaluates its likelihood term at the clean input and selects the perturbation by maximizing the KL divergence itself, not the negative log-likelihood.} 

The preceding examples concern proactive training objectives. The framework also recovers inference-time defenses. In particular, RS \citep{cohen2019certified} arises as a reactive limit in which model parameters are fixed, the latent clean-input prior is flat, and the adversarial channel is Gaussian. 
For clarity, recall that RS constructs a smoothed classifier from a base classifier $f$ by averaging its decisions under isotropic Gaussian perturbations. Its decision rule can be written as
$
g(\mathbf{x}'_j)
=
\argmax_c
\mathbb{P}_{\boldsymbol{\delta}\sim\mathcal N(0,\sigma^2 I)}
\left\{
f(\mathbf{x}'_j+\boldsymbol{\delta})=c
\right\}.
$
Thus, RS may be read as replacing a single prediction at $\mathbf{x}'_j$ by the predictive distribution induced by a Gaussian neighborhood of $\mathbf{x}'_j$. In our notation, the same averaging arises endogenously from Bayesian inversion of the adversarial channel: the Gaussian smoothing distribution is the posterior of the latent clean input once the parameter posterior is collapsed to a point estimate and the covariate prior is flattened.

\begin{restatable}[RS as a reactive limit]{proposition}{rslimit}\label{prop:rs_limit}
Consider the reactive model in Section~\ref{operations}, replacing the posterior over $(\theta,\phi)$ by the point estimate $(\theta_{\mathrm{MAP}},\phi_{\mathrm{MAP}})$. Further suppose that $p(\mathbf{x}\mid\phi_{\mathrm{MAP}})\propto 1$ on $\mathcal X=\mathbb R^d$ (an improper, limiting choice of the covariate model), and that the adversarial channel is isotropic Gaussian, $p(\mathbf{x}'\mid\mathbf{x},\theta_{\mathrm{MAP}}) = \mathcal N(\mathbf{x}';\mathbf{x},\sigma^2 I)$. Then, the reactive PPD in \eqref{eq:exact_reactive_nested} becomes
\begin{equation}\label{eq:rs_predictive}
p(y_j\mid\mathbf{x}'_j,\mathcal D)
\approx
\mathbb{E}_{\mathbf{x}\sim\mathcal N(\mathbf{x}'_j,\sigma^2 I)}
\left[
p(y_j\mid\mathbf{x},\theta_{\mathrm{MAP}})
\right],
\end{equation}
which is the RS predictive distribution.
For classification, returning the largest component of \eqref{eq:rs_predictive} gives the smoothed decision rule
\[
g(\mathbf{x}'_j)
=
\argmax_c
\mathbb{E}_{\mathbf{x}\sim\mathcal N(\mathbf{x}'_j,\sigma^2 I)}
\left[
p(c\mid\mathbf{x},\theta_{\mathrm{MAP}})
\right].
\]
When the base classifier is deterministic, $p(c\mid\mathbf{x},\theta_{\mathrm{MAP}})$ reduces to the indicator of $f(\mathbf{x})=c$, recovering the standard RS expression above.
\end{restatable}

%%%%%%%%%%%%%%%%%%%%%%%%%
\section{Experiments}

We conduct a series of experiments to demonstrate the advantages of our Bayesian defense framework. We evaluate both the proactive and the reactive defenses, but, given the test-time cost of the latter (Section~\ref{operations}), our analysis centers on the proactive ones, which we compare extensively against AT baselines. % \deltwo{Note though that we also offer a conceptual validation of the reactive approach.} \revii{The reactive defenses are also evaluated on these benchmarks, but only as a proof of concept: their test-time cost (see \rev{Appendix~\ref{app:exps}}) currently precludes large-scale use.} 

\subsection{Experimental Setup}\label{subsec:setup}

We evaluate our framework across both \textit{classification} and \textit{regression} tasks. In all cases, the underlying predictive model is a Bayesian neural network (BNN) trained with VI. For image classification, our main benchmark is CIFAR-10 \citep{krizhevsky2009learning}, with BNNs built on the ResNet-18 \citep{he2016deep} and Wide ResNet (WRN34-10) \citep{zagoruyko2016wide} architectures. The MNIST dataset \citep{MNIST} is reserved for a specific downstream selective-prediction task. For regression, we consider the Wine \citep{wine_data} and Energy Efficiency \citep{athanasios_tsanas_energy_2012} datasets, with a fully connected BNN.

Our proactive defenses differ in their adversarial channel (i.e., in the model of the adversary they assume). Instantiating the channels of Section~\ref{subsec:channels}, we consider the following ones:

\begin{itemize}
    \item \emph{Noise channel (\texttt{ID}):} The baseline attack-agnostic channel $p_{\mathrm{id},\sigma}$ from Section~\ref{subsec:channels}, which introduces isotropic Gaussian augmentation without an adversarial objective.
    
    \item \emph{Mixture channel (\texttt{MIX}):} Instantiates the general mixture channel $p_{\boldsymbol{\pi}}$ to simulate a diverse adversary. Concretely, \texttt{MIX} blends the identity channel with probabilistic attacks built from cross-entropy, CW, and predictive-entropy scores, and a stronger multi-step PGD5 component. We evaluate this channel using both ResNet-18 (\texttt{MIX}) and WRN34-10 (\texttt{MIX WRN}) backbones. For regression, \texttt{MIX} relies on the identity channel and a one-step squared-error objective.
    
    \item \emph{Learned channel (\texttt{NN}):} A separate neural network parameterizes the learned channel $p_{\omega}$ of Section~\ref{subsec:channels} (a convolutional generator for images and a fully connected one for tabular inputs), defining a specification-free adversary. In the regression experiments, \texttt{NN50} mixes this channel 50/50 with the identity channel.
\end{itemize}

Throughout, the proactive posteriors we fit are variational approximations of the generalized posterior \eqref{eq:gen_posterior_eta} at $\eta=1$ (for \texttt{NN}, at the final value of $\omega$): except for \texttt{ID}, the channels are constructed from the observed responses, so no ordinary posterior underlies them.

We compare these instances of our proactive defense against baseline methods (including AT) and our reactive defenses:

\begin{itemize}
    
    \item {\em Baselines (\texttt{BL}, \texttt{AT} and \texttt{AT PGD5}):} We benchmark against an undefended BNN trained on clean data (\texttt{BL}) and two BNNs trained with AT on a 50/50 mix of clean and deterministically attacked inputs: \texttt{AT} uses a one-step attack, matching the one-step channel of the \texttt{OS50} ablation below, and \texttt{AT PGD5} uses a 5-step PGD attack, matching the multi-step component of \texttt{MIX}.

    \item \emph{MAP Estimation (\texttt{MIX WRN MAP}):} To isolate the effect of posterior uncertainty, we train the WRN with the same \texttt{MIX} channel but fit a MAP point estimate instead of the variational posterior.
    
    \item \emph{Reactive defenses (\texttt{onPure} and \texttt{offPure}):} Finally, we assess the empirical online and offline reactive defenses of Equations~\eqref{onPure} and \eqref{offPure}.
\end{itemize}

For evaluation, we assess adversarial robustness using several attacks bounded in $L_2$ norm: single-step PGD (PGD1), multi-step PGD (50 iterations), entropy-based PGD (ENT), PGD$^+$ \citep{feng2024attacking} (25 iterations of PGD followed by 25 of entropy-based PGD), and the parameter-free ensemble AutoAttack \citep{croce2020reliable}.
AutoAttack combines gradient-based and gradient-free attacks, which reduces the risk of overestimating robustness through gradient masking. These evaluation attacks are distinct from the training channels: the channels encode the attacks a defense anticipates, whereas the evaluation attacks probe the trained models; some of them deliberately differ from every channel, to test robustness against unanticipated threats.
All evaluation attacks are white-box and target the predictive distribution that each model deploys. For every BNN, this distribution averages the predictions of a fixed set of posterior samples, which the attacker knows; the attacked predictor is thus a deterministic ensemble whose gradients the attacker computes exactly. This is the setting most favorable to the attacker, so the reported robustness is conservative: redrawing the posterior samples at each query would make a model harder to attack, although attackers can partly offset this by averaging gradients over draws \citep{athalye2018obfuscated}. The only exceptions are the reactive defenses on CIFAR-10, which are attacked by transfer, as detailed below.

Since our models deliver probabilistic predictions, and attacks may corrupt their uncertainty as well as their accuracy, we complement the usual deterministic metrics (Accuracy/RMSE) with the Negative Log-Likelihood (NLL) and the Expected Calibration Error (ECE), which assess the quality of predictive uncertainty. Tables report means with standard deviations over five random test splits, common to all methods, with best values highlighted in bold. Values whose one-standard-deviation interval overlaps that of the best value are also shown in bold. Further details are provided in Appendix~\ref{app:exps} and our code repository. Experiments ran on three NVIDIA A100 GPUs.

% In our experiments, we instantiate the adversarial channels from our framework into several proactive training regimes. We explore attack-based channels built from a one-step PGD attack, made probabilistic by perturbing the attacked input with Gaussian noise. This yields two models: OS, trained exclusively on these attacks, and OS50, trained on minibatches of 50\% attacked and 50\% clean data.  Second, our MIX model implements a mixture channel, which combines several different probabilistic attack types (PGD, PGD+, CW, etc.) and attack parameters to simulate a more diverse adversary. Third, we instantiate a learned generative channel in our NN and NN50 models, where a separate NN generates the attack distribution. As before, NN50 uses a 50/50 training mix while NN is trained solely on attacks. We benchmark these defenses against an undefended BNN trained on clean data (baseline, BL)  and a conventional  AT implementation, also trained on a 50/50 mix of clean and one-step attacked inputs. Finally, we also assess our two reactive defenses, the offline (offPure) and the online adaptive (onPure) models, corresponding to equations \eqref{offPure} and \eqref{onPure}, respectively.

\subsection{Case Study: Classification}\label{subsec:classification}

\paragraph*{Evaluating Defenses.}
On CIFAR-10, we focus on the two attacks that degrade accuracy the most, multi-step PGD and AutoAttack \citep{croce2020reliable}, whose results we show here. Those under PGD1 and ENT are reported in Appendix~\ref{app:exps}. Figures~\ref{fig:cifar10_pgd} and~\ref{fig:cifar10_autoattack} show, for each defense, how accuracy, NLL and ECE evolve as the perturbation budget $\epsilon$ grows, under PGD and AutoAttack, respectively.

\begin{figure*}[!t]
    \centering
    \includegraphics[width=0.32\linewidth]{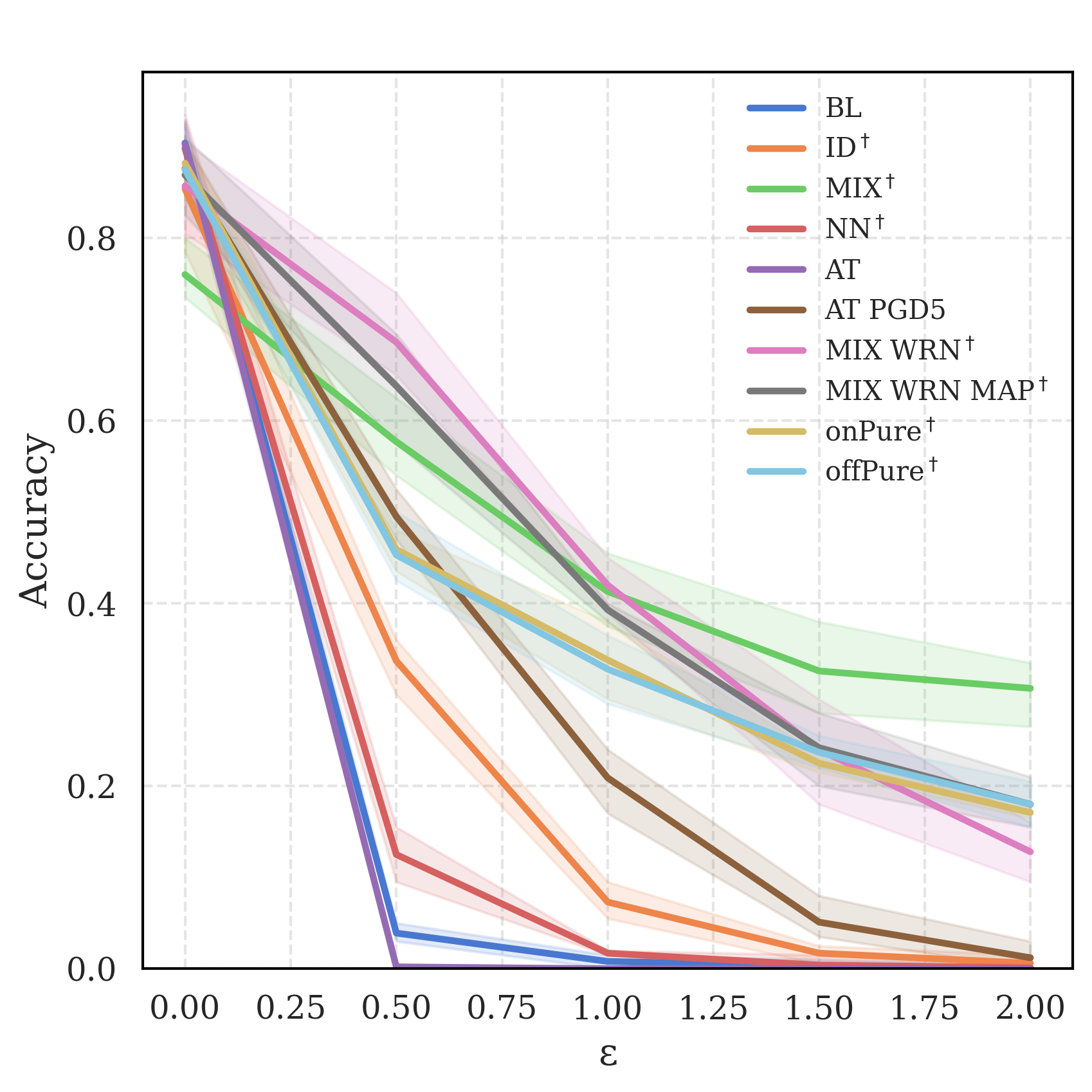}
    \includegraphics[width=0.32\linewidth]{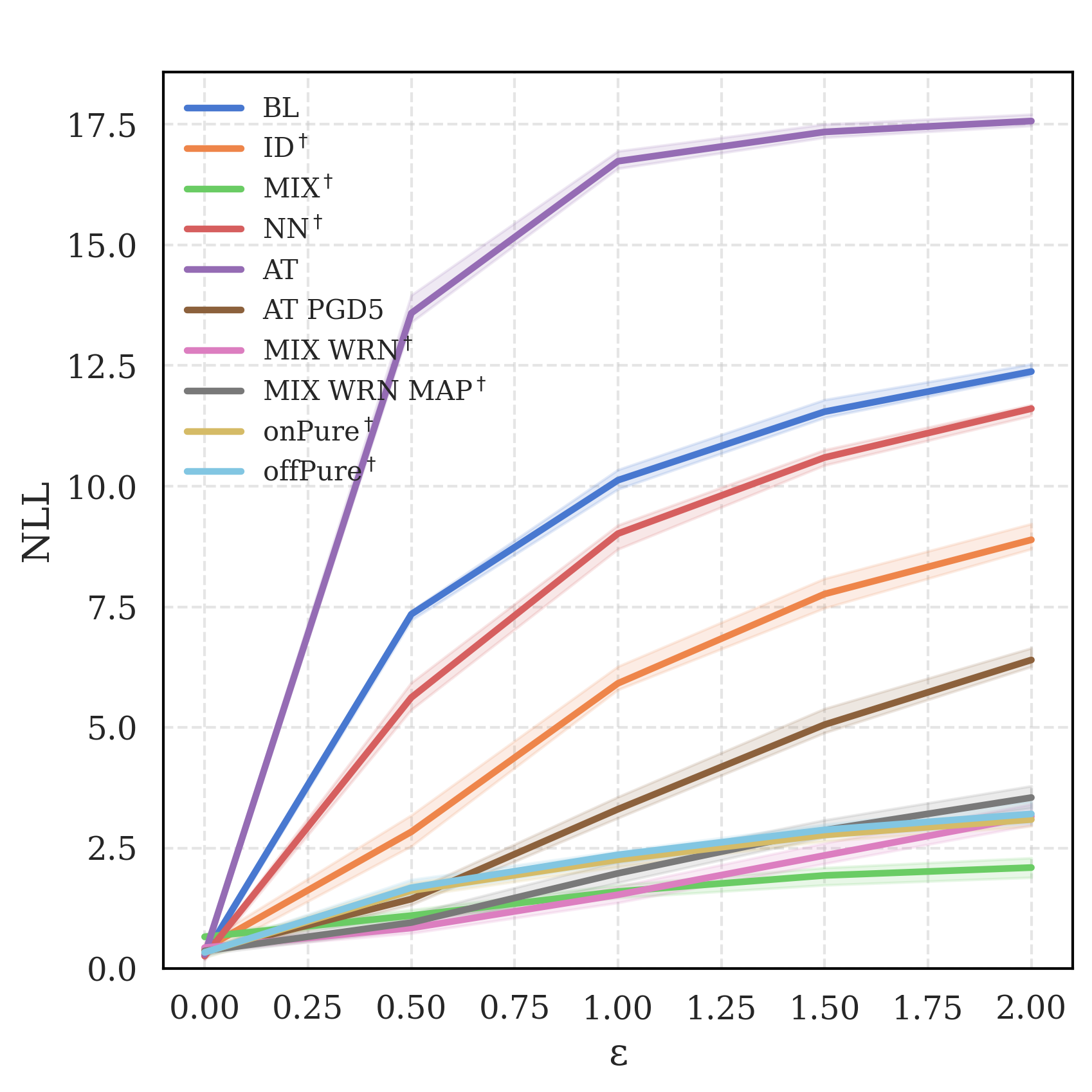}
    \includegraphics[width=0.32\linewidth]{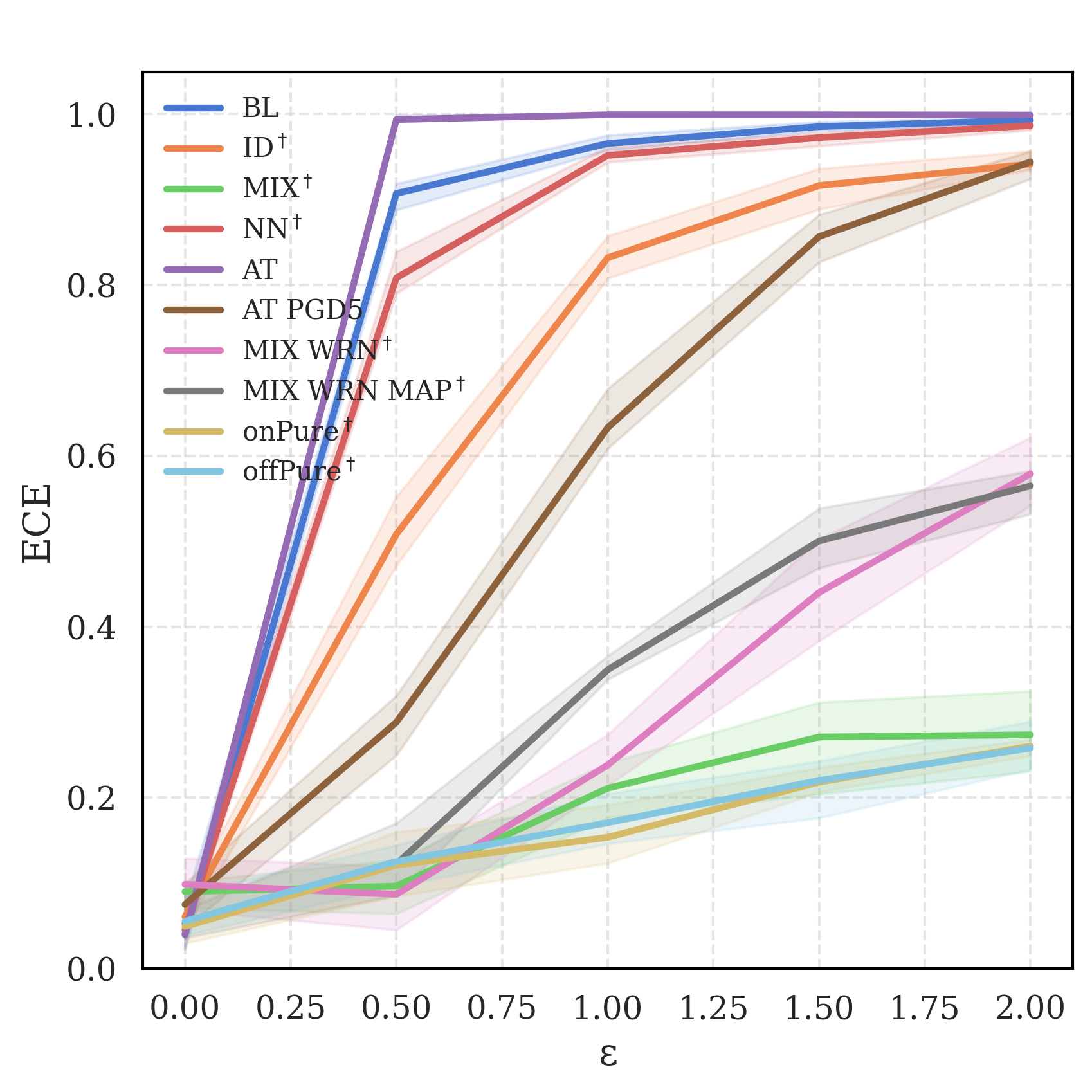}
    \caption{Performance on CIFAR-10 against PGD across varying perturbation budgets ($\epsilon$). From left to right: Accuracy (higher is better), NLL (lower is better), and ECE (lower is better). Shaded bands indicate $\pm 1$ SD over five random test splits. Our defenses are marked with $^\dagger$.}
    \label{fig:cifar10_pgd}
\end{figure*}

\begin{figure*}[!t]
    \centering
    \includegraphics[width=0.32\linewidth]{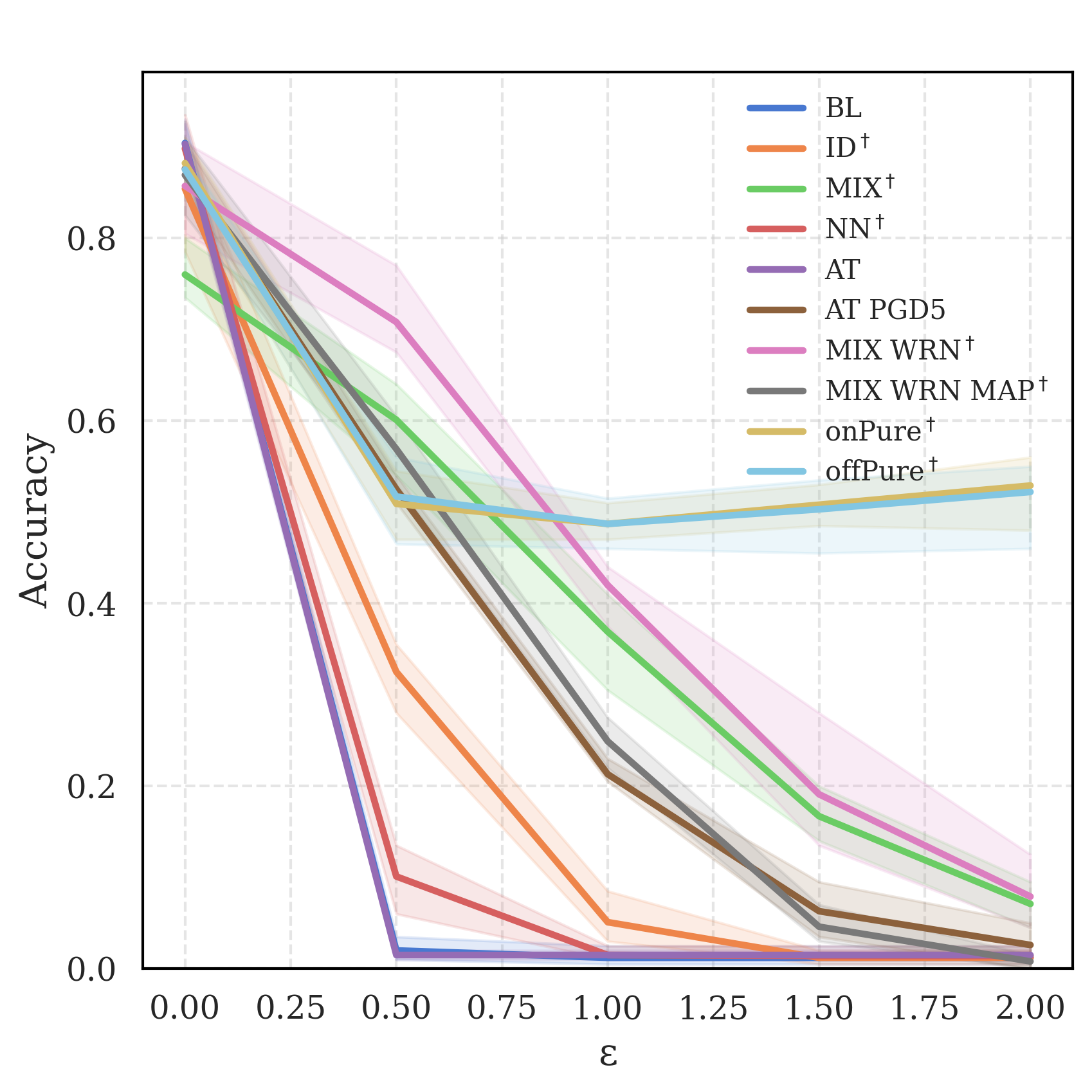}
    \includegraphics[width=0.32\linewidth]{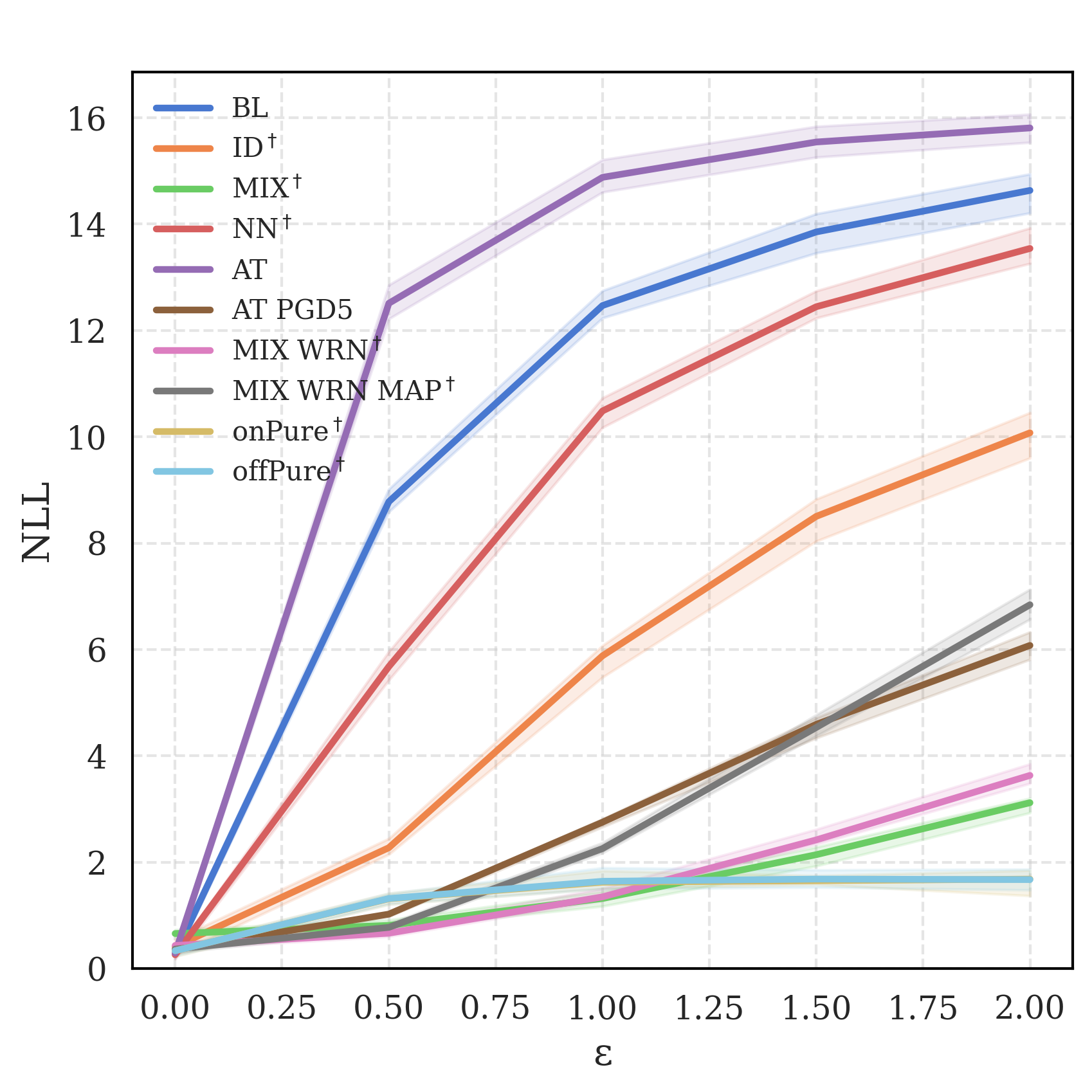}
    \includegraphics[width=0.32\linewidth]{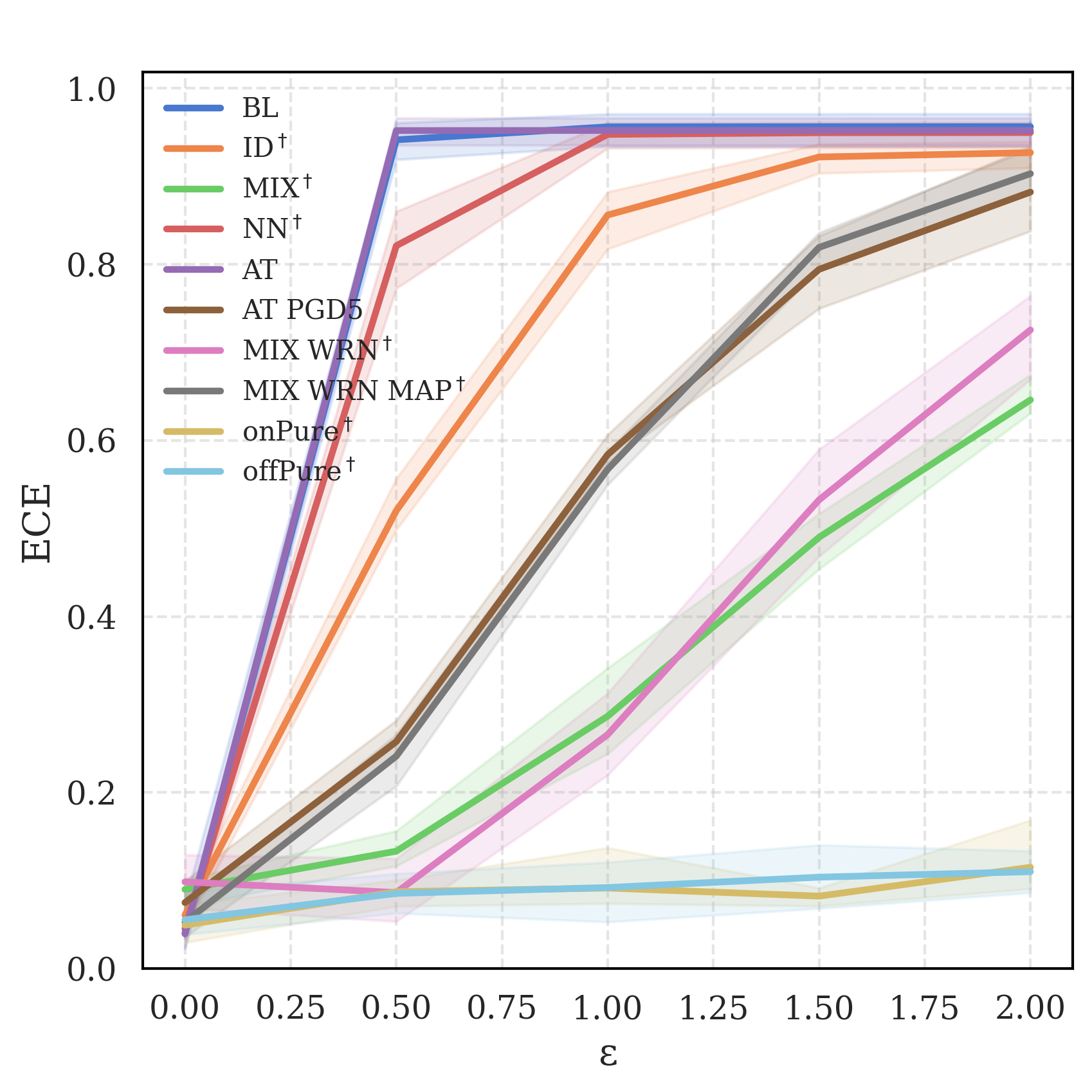}
    \caption{Performance on CIFAR-10 against AutoAttack across varying perturbation budgets ($\epsilon$). From left to right: Accuracy (higher is better), NLL (lower is better), and ECE (lower is better). Shaded bands indicate $\pm 1$ SD over five random test splits. Our defenses are marked with $^\dagger$.}
    \label{fig:cifar10_autoattack}
\end{figure*}

Under both attacks, the undefended baseline (\texttt{BL}) and one-step adversarial training (\texttt{AT}) lose nearly all accuracy by $\epsilon=0.5$, and become confidently wrong: their ECE, which ranges from 0 (perfect calibration) to 1 (worst), already approaches 1 at that budget. While the baseline noise channel (\texttt{ID}) provides minor regularization over \texttt{BL}, it lacks the adversarial guidance required to withstand strong iterative attacks, decaying rapidly across all metrics. Similarly, the learned channel (\texttt{NN}) struggles to generalize against multi-step attacks on CIFAR-10, closely tracking the degradation of the baselines.
Training against a 5-step attack (\texttt{AT PGD5}) prevents the collapse of \texttt{AT}, but its accuracy decreases steadily from the smallest budget, to about 0.5 at $\epsilon=0.5$ and 0.2 at $\epsilon=1.0$ under both attacks, while its ECE rises steeply.
In contrast, our proactive defenses equipped with the mixture channel (\texttt{MIX} and \texttt{MIX WRN}) degrade much more slowly. They retain higher accuracy than \texttt{AT PGD5} at every nonzero budget under both attacks, albeit at a lower clean accuracy for \texttt{MIX}, while keeping NLL and ECE well below those of the baselines. Although \texttt{AT PGD5} is trained with the same 5-step PGD attack that \texttt{MIX} includes among its components, both differ most in calibration: at $\epsilon=1.0$, the ECE of \texttt{MIX} is about half that of \texttt{AT PGD5} under AutoAttack, and about a third under PGD.

These results are also competitive with the published state of the art for deterministic, non-Bayesian networks, whose parameters are point estimates. At $\epsilon=0.5$, \texttt{MIX WRN} retains 0.71 AutoAttack accuracy (Table~\ref{tab:abl_acc}), in line with the 66--74\% that $L_2$ AT reaches on such networks without extra or synthetic data, as reported on RobustBench \citep{croce2021robustbench}, e.g., 73.7\% for AWP with the same WRN34-10 backbone \citep{wu2020adversarial} and 67.7\% for \citet{rice2020overfitting}. To our knowledge, prior Bayesian defenses on CIFAR-10 report $L_\infty$ PGD-type or certified results, but no $L_2$ AutoAttack evaluation \citep{liu2019advbnn, wicker2021bayesian, doan2022bayesian}.

The reactive defenses (\texttt{onPure} and \texttt{offPure}) are evaluated differently. Differentiating iterative attacks through their sampling loop is computationally prohibitive on CIFAR-10, so we craft the adversarial examples against \texttt{BL} and then feed them to the reactive defenses, which purify them before predicting. Under PGD (Figure~\ref{fig:cifar10_pgd}) they track the proactive \texttt{MIX} models in accuracy and calibration, and under AutoAttack (Figure~\ref{fig:cifar10_autoattack}) they keep accuracy near $0.5$ up to $\epsilon=2.0$ with the flattest ECE curve, at a higher test-time cost. Since these perturbations are crafted against \texttt{BL} rather than against the purified predictor, these curves should be read as an upper bound on the robustness of the reactive defenses rather than compared directly with the white-box results.

Importantly, these results empirically validate the advantage of Bayesian marginalization over point estimation: under AutoAttack (Figure~\ref{fig:cifar10_autoattack}), the variational posterior (\texttt{MIX WRN}) consistently outperforms its MAP counterpart (\texttt{MIX WRN MAP}) across all perturbation budgets, retaining substantially higher accuracy while avoiding the sharp inflation in NLL and ECE suffered by \texttt{MIX WRN MAP}. A similar calibration advantage holds under PGD (Figure~\ref{fig:cifar10_pgd}), where \texttt{MIX WRN} achieves superior accuracy at moderate budgets ($\epsilon \le 1.25$) and lower NLL across the entire perturbation range, as well as lower ECE up to $\epsilon=1.5$. This confirms that marginalizing over a posterior distribution provides better calibration and protection against adversarial misspecification than relying on a single point estimate.

\paragraph*{Ablation Study.} 

Next, we conduct an ablation study on CIFAR-10 to isolate the sources of robustness in our proactive defense, investigating channel stochasticity, adversarial diversity, multi-step training components, architectural capacity, and Bayesian marginalization. Tables~\ref{tab:abl_acc}, \ref{tab:abl_nll} and \ref{tab:abl_ece} report all models at $\epsilon=0.5$.
Besides the models above, the ablation includes \texttt{OS50}, a 50/50 mixture of the identity channel and a one-step channel on the CW score; \texttt{MIX OS}, which removes the PGD5 component from \texttt{MIX}; \texttt{MIX AA}, whose channel imitates the attacks combined by AutoAttack: it mixes the identity channel with a gradient-based multi-step attack on the CW margin score and the gradient-free Square attack \citep{andriushchenko2020square}, one of AutoAttack's components (since Square is a random search, its output is treated as fixed when differentiating the training objective); \texttt{MIX ConvNet}, the \texttt{MIX} channel on a shallow ConvNet; and \texttt{BL WRN}, the undefended baseline on the WRN34-10 backbone.

Each comparison below contrasts models that share everything except the factor under study. For brevity, we quote AutoAttack accuracy and ECE at $\epsilon=0.5$; accuracy under PGD and the NLL under both attacks tell the same story in every comparison (Tables~\ref{tab:abl_acc}--\ref{tab:abl_ece}).
\begin{itemize}
    \item \emph{Channel stochasticity} (\texttt{AT} vs.\ \texttt{OS50}). Both train on a 50/50 mix of clean and one-step attacked inputs; \texttt{OS50} only adds the Gaussian noise of its channels. The deterministic \texttt{AT} collapses (accuracy 0.02, ECE 0.95), whereas \texttt{OS50} does not (0.30, 0.52), in line with the role of noise in preventing catastrophic overfitting of one-step AT \citep{wong2020fast, de2022make}.
    \item \emph{Diversity of attack objectives} (\texttt{OS50} vs.\ \texttt{MIX OS}). Adding one-step channels on the cross-entropy and entropy scores to the CW one raises accuracy from 0.30 to 0.38 and lowers ECE from 0.52 to 0.42.
    \item \emph{Multi-step component} (\texttt{MIX OS} vs.\ \texttt{MIX}). Adding the PGD5 component raises accuracy from 0.38 to 0.60 and cuts ECE from 0.42 to 0.13, at the cost of clean accuracy (0.85 to 0.76). Instead, tailoring the channel to the evaluation attack (\texttt{MIX AA}) does not help (0.29, 0.48).
    \item \emph{Stochastic, diverse channel versus deterministic multi-step training} (\texttt{AT PGD5} vs.\ \texttt{MIX}). Both use the same 5-step attack, but \texttt{MIX} embeds it in a stochastic, diverse channel. \texttt{AT PGD5} avoids the collapse of \texttt{AT}, yet falls short of \texttt{MIX} (0.53 vs.\ 0.60), most clearly in ECE (0.26 vs.\ 0.13).
    \item \emph{Capacity} (\texttt{MIX ConvNet}, \texttt{MIX} and \texttt{MIX WRN}). With the same channel, accuracy grows from 0.28 (shallow ConvNet) to 0.60 (ResNet-18) and 0.71 (WRN34-10), while ECE falls from 0.43 to 0.13 and 0.09. Capacity alone does not explain this gain: without a channel, \texttt{BL WRN} collapses (0.04, 0.92).
    \item \emph{Bayesian marginalization} (\texttt{MIX WRN} vs.\ \texttt{MIX WRN MAP}). With the same channel and architecture, the variational posterior outperforms the MAP point estimate (0.71 vs.\ 0.57), with a much lower ECE (0.09 vs.\ 0.24).
\end{itemize}
To sum up, a diverse stochastic channel, a multi-step component and posterior marginalization all contribute to high accuracy and reliable calibration under attack.

\begin{table}[h]
\centering
\caption{Accuracy on CIFAR-10 under different attack strategies with perturbation budget $\epsilon=0.5$. SD in parentheses, best in bold. Our defenses are marked with $^\dagger$.}
\label{tab:abl_acc}
\begin{tabular}{lcccc}
\hline
\textbf{Model} & \textbf{Clean} & \textbf{AutoAttack} & \textbf{PGD} & \textbf{ENT} \\
\hline
\texttt{BL}          & \textbf{0.90 (0.02)} & 0.02 (0.01) & 0.04 (0.01) & 0.38 (0.03) \\
\texttt{ID}$^\dagger$          & 0.85 (0.04) & 0.33 (0.03) & 0.34 (0.02) & 0.64 (0.04) \\
\texttt{OS50}$^\dagger$        & 0.84 (0.03) & 0.30 (0.03) & 0.40 (0.02) & 0.62 (0.03) \\
\texttt{MIX OS}$^\dagger$      & 0.85 (0.03) & 0.38 (0.03) & 0.44 (0.02) & 0.68 (0.04) \\
\texttt{MIX}$^\dagger$         & 0.76 (0.02) & 0.60 (0.04) & 0.58 (0.03) & 0.75 (0.02) \\
\texttt{MIX AA}$^\dagger$      & 0.80 (0.02) & 0.29 (0.03) & 0.35 (0.02) & 0.61 (0.04) \\
\texttt{NN}$^\dagger$          & \textbf{0.90 (0.02)} & 0.10 (0.03) & 0.13 (0.02) & 0.42 (0.02) \\
\texttt{MIX ConvNet}$^\dagger$ & 0.79 (0.02) & 0.28 (0.03) & 0.31 (0.02) & 0.59 (0.05) \\
\texttt{BL WRN}      & \textbf{0.91 (0.02)} & 0.04 (0.01) & 0.06 (0.01) & 0.47 (0.02) \\
\texttt{MIX WRN}$^\dagger$     & \textbf{0.86 (0.04)} & \textbf{0.71 (0.03)} & \textbf{0.69 (0.03)} & \textbf{0.83 (0.02)} \\
\texttt{MIX WRN MAP}$^\dagger$ & \textbf{0.87 (0.03)} & 0.57 (0.03) & \textbf{0.64 (0.04)} & \textbf{0.83 (0.01)} \\
\texttt{onPure}$^\dagger$      & \textbf{0.88 (0.02)} & 0.51 (0.03) & 0.46 (0.02) & 0.45 (0.03) \\
\texttt{offPure}$^\dagger$     & 0.88 (0.01) & 0.52 (0.03) & 0.45 (0.03) & 0.44 (0.03) \\
\texttt{AT}          & \textbf{0.90 (0.02)} & 0.02 (0.00) & 0.00 (0.00) & 0.41 (0.03) \\
\texttt{AT PGD5}     & \textbf{0.88 (0.02)} & 0.53 (0.01) & 0.50 (0.02) & \textbf{0.84 (0.03)} \\
\hline
\end{tabular}
\end{table}

\begin{table}[h]
\centering
\caption{NLL on CIFAR-10 under different attack strategies with perturbation budget $\epsilon=0.5$. SD in parentheses, best in bold. Our defenses are marked with $^\dagger$.}
\label{tab:abl_nll}
\begin{tabular}{lcccc}
\hline
\textbf{Model} & \textbf{Clean} & \textbf{AutoAttack} & \textbf{PGD} & \textbf{ENT} \\
\hline
\texttt{BL}          & \textbf{0.28 (0.06)} & 8.78 (0.14) & 7.35 (0.07) & 1.68 (0.07) \\
\texttt{ID}$^\dagger$          & 0.42 (0.09) & 2.28 (0.11) & 2.84 (0.21) & 0.97 (0.08) \\
\texttt{OS50}$^\dagger$        & 0.42 (0.07) & 2.30 (0.13) & 2.08 (0.14) & 0.97 (0.07) \\
\texttt{MIX OS}$^\dagger$      & 0.40 (0.06) & 1.74 (0.08) & 1.82 (0.14) & 0.87 (0.06) \\
\texttt{MIX}$^\dagger$         & 0.66 (0.05) & 0.82 (0.06) & 1.10 (0.07) & 0.81 (0.05) \\
\texttt{MIX AA}$^\dagger$      & 0.57 (0.06) & 1.96 (0.09) & 2.18 (0.17) & 1.05 (0.09) \\
\texttt{NN}$^\dagger$          & \textbf{0.26 (0.04)} & 5.68 (0.17) & 5.62 (0.20) & 1.40 (0.06) \\
\texttt{MIX ConvNet}$^\dagger$ & 0.57 (0.05) & 1.79 (0.10) & 2.21 (0.11) & 1.10 (0.04) \\
\texttt{BL WRN}      & \textbf{0.25 (0.04)} & 7.98 (0.25) & 5.66 (0.05) & 1.45 (0.09) \\
\texttt{MIX WRN}$^\dagger$     & 0.44 (0.04) & \textbf{0.67 (0.06)} & \textbf{0.84 (0.07)} & \textbf{0.61 (0.04)} \\
\texttt{MIX WRN MAP}$^\dagger$ & 0.37 (0.04) & 0.78 (0.05) & \textbf{0.96 (0.11)} & \textbf{0.60 (0.03)} \\
\texttt{onPure}$^\dagger$      & \textbf{0.34 (0.06)} & 1.32 (0.07) & 1.62 (0.10) & 1.44 (0.05) \\
\texttt{offPure}$^\dagger$     & \textbf{0.34 (0.06)} & 1.32 (0.07) & 1.68 (0.11) & 1.51 (0.05) \\
\texttt{AT}          & \textbf{0.29 (0.06)} & 12.51 (0.21) & 13.59 (0.20) & 1.68 (0.05) \\
\texttt{AT PGD5}     & 0.36 (0.05) & 1.03 (0.03) & 1.44 (0.10) & 0.75 (0.04) \\
\hline
\end{tabular}
\end{table}

\begin{table}[h]
\centering
\caption{ECE on CIFAR-10 under different attack strategies with perturbation budget $\epsilon=0.5$ (lower is better). SD in parentheses, best in bold. Our defenses are marked with $^\dagger$.}
\label{tab:abl_ece}
\begin{tabular}{lcccc}
\hline
\textbf{Model} & \textbf{Clean} & \textbf{AutoAttack} & \textbf{PGD} & \textbf{ENT} \\
\hline
\texttt{BL}          & \textbf{0.04 (0.02)} & 0.94 (0.02) & 0.91 (0.01) & 0.18 (0.03) \\
\texttt{ID}$^\dagger$          & \textbf{0.06 (0.01)} & 0.52 (0.02) & 0.51 (0.03) & 0.12 (0.02) \\
\texttt{OS50}$^\dagger$        & \textbf{0.05 (0.01)} & 0.52 (0.03) & 0.39 (0.01) & \textbf{0.09 (0.01)} \\
\texttt{MIX OS}$^\dagger$      & \textbf{0.06 (0.02)} & 0.42 (0.03) & 0.33 (0.02) & 0.13 (0.03) \\
\texttt{MIX}$^\dagger$         & 0.09 (0.01) & 0.13 (0.02) & \textbf{0.10 (0.03)} & 0.18 (0.01) \\
\texttt{MIX AA}$^\dagger$      & 0.08 (0.02) & 0.48 (0.01) & 0.38 (0.03) & 0.12 (0.02) \\
\texttt{NN}$^\dagger$          & \textbf{0.05 (0.01)} & 0.82 (0.03) & 0.81 (0.02) & \textbf{0.08 (0.02)} \\
\texttt{MIX ConvNet}$^\dagger$ & 0.07 (0.01) & 0.43 (0.03) & 0.38 (0.02) & 0.15 (0.03) \\
\texttt{BL WRN}      & \textbf{0.05 (0.01)} & 0.92 (0.01) & 0.77 (0.02) & \textbf{0.09 (0.02)} \\
\texttt{MIX WRN}$^\dagger$     & 0.10 (0.02) & \textbf{0.09 (0.02)} & \textbf{0.09 (0.03)} & 0.18 (0.02) \\
\texttt{MIX WRN MAP}$^\dagger$ & \textbf{0.05 (0.01)} & 0.24 (0.02) & \textbf{0.12 (0.03)} & 0.17 (0.01) \\
\texttt{onPure}$^\dagger$      & \textbf{0.05 (0.01)} & \textbf{0.09 (0.01)} & \textbf{0.12 (0.03)} & 0.10 (0.01) \\
\texttt{offPure}$^\dagger$     & \textbf{0.06 (0.02)} & \textbf{0.09 (0.02)} & \textbf{0.13 (0.02)} & 0.13 (0.01) \\
\texttt{AT}          & \textbf{0.04 (0.01)} & 0.95 (0.01) & 0.99 (0.00) & \textbf{0.07 (0.02)} \\
\texttt{AT PGD5}     & 0.08 (0.02) & 0.26 (0.01) & 0.29 (0.02) & 0.27 (0.02) \\
\hline
\end{tabular}
\end{table}

\paragraph*{Downstream Task: Selective Accuracy.}
Finally, we assess robustness on a downstream task that relies on predictive uncertainty, against attacks that target the uncertainty itself, a threat specific to models that quantify it, such as Bayesian ones \citep{feng2024attacking, arce2025evasion}. The task is selective prediction for out-of-distribution detection. A BNN trained on MNIST receives a balanced test set of MNIST digits and FashionMNIST images \citep{xiao2017online}, which it should not classify, and abstains on the half of the inputs with the highest predictive entropy. Selective accuracy is then the fraction of retained inputs that are MNIST digits classified correctly, so it rewards both accurate predictions and reliable uncertainty. The attacker acts on both fronts: MNIST digits are attacked with PGD$^+$, which induces misclassification while increasing entropy, and FashionMNIST images with an entropy attack that decreases entropy, so that they evade detection.

Figure~\ref{fig:sel_acc} reports selective accuracy as the attack budget grows. All defenses improve on \texttt{BL} at every budget, except one-step \texttt{AT}, which falls below it already at $\epsilon=0.5$ (0.10 vs.\ 0.27). The learned proactive channel (\texttt{NN}) attains the highest selective accuracy without attack and at small budgets (0.75 at $\epsilon=0$ and 0.58 at $\epsilon=0.5$). Under stronger perturbations, the reactive defenses (\texttt{onPure} and \texttt{offPure}), which purify the inputs of the standardly trained \texttt{BL}, are the most robust, retaining about 0.41 at $\epsilon=1.0$ and 0.27 at $\epsilon=2.0$, against at most 0.33 and 0.06 for the proactive models.

\begin{figure}[h]
    \centering
    \includegraphics[width=.32\linewidth]{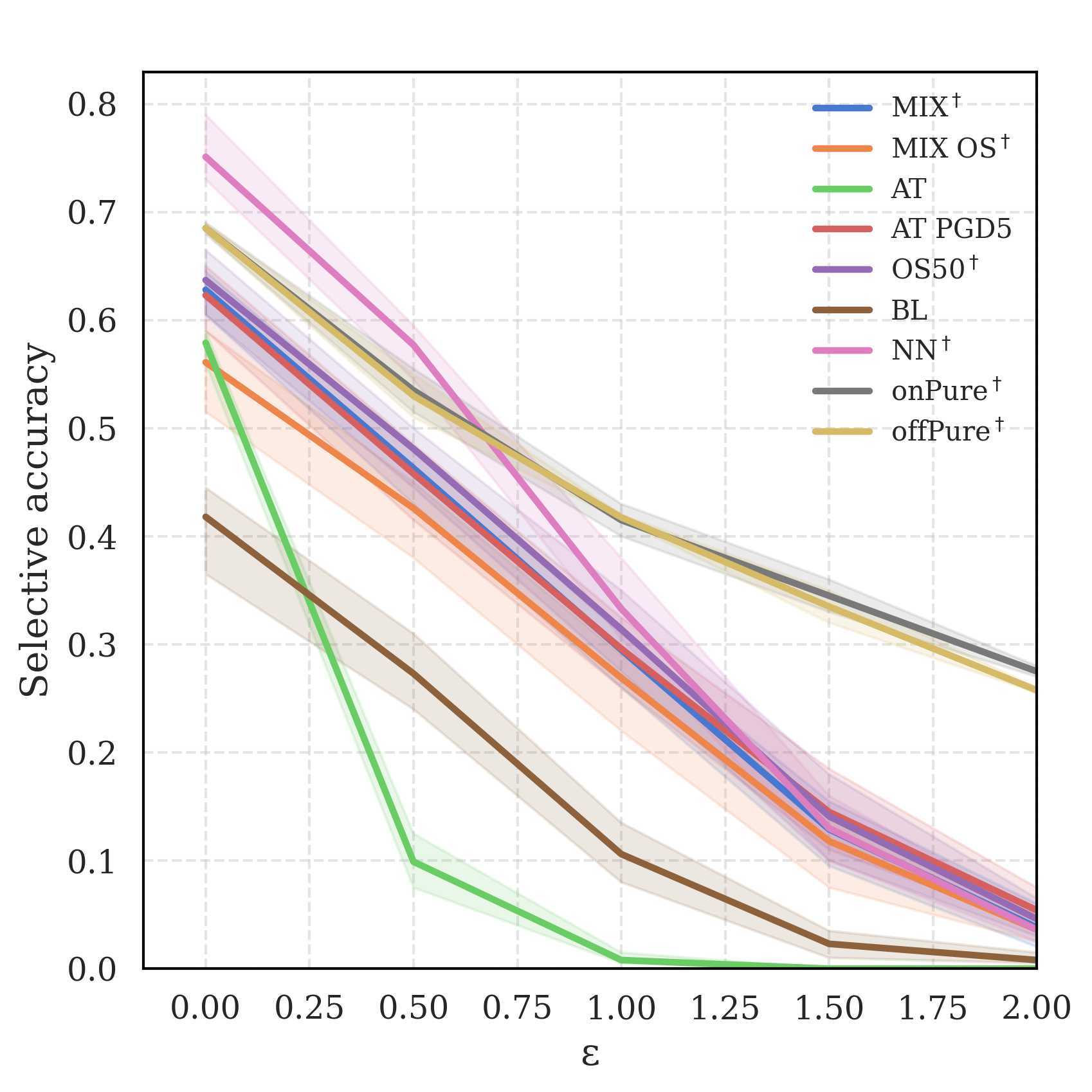}
    \caption{Selective accuracy on balanced mixture of MNIST and FashionMNIST across varying attack budgets ($\epsilon$). Samples with lowest predictive uncertainty (bottom 50\%)~are~kept. Our defenses are marked with $^\dagger$.}% against PGD$^+$ for MNIST and entropy-PGD attack for FashionMNIST.}
    \label{fig:sel_acc}
\end{figure}

\subsection{Case Study: Regression}

We extend our assessment to regression tasks using the Wine and Energy datasets. Tables \ref{tab:reg_rmse_wine_energy} and \ref{tab:reg_nll_wine_energy} confirm our methodology's effectiveness beyond classification. The undefended baseline exhibits catastrophic failure under attack for both datasets, with RMSE increasing dramatically (e.g., 0.75 to over 15 on Wine) and NLL values exploding, indicating a complete loss of predictive reliability. \texttt{AT} mitigates this failure, but remains far behind our proactive defenses (e.g., RMSE between 6.5 and 8.6 under attack on Wine). On Wine, \texttt{NN50} achieves the best overall performance, maintaining reasonable clean RMSE with modest degradation under attack (RMSE from 0.79 to 0.96 under the worst attack) and well-calibrated predictions (NLL around 1.4). \texttt{MIX} performs competitively, achieving the lowest clean RMSE (0.74) but showing slightly higher attack sensitivity. On Energy, \texttt{MIX} and \texttt{NN50} perform similarly, with \texttt{MIX} slightly stronger in RMSE and both comparable~in~NLL. The purification methods behave differently: on both datasets, their clean RMSE exceeds that of the proactive defenses, but it barely changes under attack, remaining competitive with the best proactive defenses. Their calibration, however, differs: \texttt{onPure} is well calibrated on Wine, whereas \texttt{offPure} suffers from calibration issues on both datasets.

\begin{table*}[!tb]
\centering
\caption{RMSE  for regression tasks on Wine and Energy datasets at $\epsilon=2.0$. Our defenses are marked with $^\dagger$.}
\label{tab:reg_rmse_wine_energy}
\setlength{\tabcolsep}{3pt}
\resizebox{\textwidth}{!}{
\begin{tabular}{l|ccccc|ccccc}
\hline
\multicolumn{1}{c}{\textbf{Model}}  & \multicolumn{5}{c}{\textbf{Wine}} & \multicolumn{5}{c}{\textbf{Energy}} \\
\hline
& \textbf{Clean} & \textbf{PGD1} & \textbf{PGD} & \textbf{PGD$^+$} & \textbf{ENT} 
& \textbf{Clean} & \textbf{PGD1} & \textbf{PGD} & \textbf{PGD$^+$} & \textbf{ENT} \\
\hline
\texttt{BL}      & \textbf{0.75 (0.04)} & 15.32 (0.05) & 15.19 (0.00) & 22.47 (0.14) & 15.63 (0.03) 
        & \textbf{2.32 (1.34)} & 8.14 (0.15) & 8.55 (0.21) & 9.73 (0.12) & 8.11 (0.36) \\
\texttt{MIX}$^\dagger$     & \textbf{0.74 (0.02)} & 1.11 (0.00) & 1.09 (0.01) & 1.53 (0.01) & 1.12 (0.03) 
        & \textbf{2.27 (1.01)} & \textbf{2.95 (0.18)} & \textbf{2.93 (0.10)} & 4.69 (0.21) & \textbf{2.95 (0.29)} \\
\texttt{NN50}$^\dagger$    & \textbf{0.79 (0.06)} & \textbf{0.93 (0.06)} & \textbf{0.90 (0.06)} & \textbf{0.96 (0.06)} & \textbf{0.89 (0.06)} 
        & \textbf{2.25 (0.79)} & 3.86 (0.16) & 3.71 (0.16) & 5.91 (0.31) & 3.86 (0.17) \\
\texttt{onPure}$^\dagger$  & 1.11 (0.03) & 1.07 (0.02) & 1.11 (0.01) & 1.08 (0.03) & 1.08 (0.01) 
        & \textbf{2.96 (0.13)} & \textbf{3.03 (0.30)} & \textbf{3.19 (0.38)} & \textbf{3.17 (0.12)} & \textbf{3.10 (0.24)} \\
\texttt{offPure}$^\dagger$ & 1.13 (0.06) & 1.13 (0.08) & 1.15 (0.04) & 1.14 (0.02) & 1.10 (0.03) 
        & \textbf{2.85 (0.21)} & \textbf{2.88 (0.26)} & \textbf{2.98 (0.38)} & \textbf{3.18 (0.15)} & \textbf{3.13 (0.33)} \\
\texttt{AT}      & \textbf{0.76 (0.02)} & 6.75 (0.03) & 6.46 (0.09) & 8.60 (0.05) & 6.70 (0.02) 
        & \textbf{2.41 (1.09)} & 6.25 (0.34) & 6.26 (0.24) & 8.59 (0.36) & 6.26 (0.49) \\
\hline
\end{tabular}
}
\end{table*}

\begin{table*}[!tb]
\centering
\caption{NLL  for regression tasks on Wine and Energy datasets at $\epsilon=2.0$. Our defenses are marked with $^\dagger$.}
\label{tab:reg_nll_wine_energy}
\setlength{\tabcolsep}{3pt}
\resizebox{\textwidth}{!}{
\begin{tabular}{l|ccccc|ccccc}
\hline
\multicolumn{1}{c}{\textbf{Model}}  & \multicolumn{5}{c}{\textbf{Wine}} & \multicolumn{5}{c}{\textbf{Energy}} \\
\hline
& \textbf{Clean} & \textbf{PGD1} & \textbf{PGD} & \textbf{PGD$^+$} & \textbf{ENT} 
& \textbf{Clean} & \textbf{PGD1} & \textbf{PGD} & \textbf{PGD$^+$} & \textbf{ENT} \\
\hline
\texttt{BL}      & \textbf{1.16 (0.06)} & 173.51 (1.97) & 149.05 (0.95) & 316.88 (3.57) & 185.29 (1.23) 
        & \textbf{0.64 (0.03)} & 172.29 (4.56) & 160.89 (7.53) & 221.56 (0.98) & 168.03 (4.11) \\
\texttt{MIX}$^\dagger$     & \textbf{1.12 (0.02)} & 1.72 (0.03) & 1.77 (0.07) & 2.56 (0.08) & 1.69 (0.04) 
        & 1.50 (0.07) & \textbf{3.47 (0.28)} & \textbf{3.33 (0.15)} & 6.78 (0.68) & \textbf{3.49 (0.34)} \\
\texttt{NN50}$^\dagger$    & 1.22 (0.05) & \textbf{1.36 (0.05)} & \textbf{1.34 (0.06)} & \textbf{1.40 (0.06)} & \textbf{1.34 (0.05)} 
        & 1.75 (0.05) & \textbf{3.52 (0.19)} & \textbf{3.48 (0.21)} & \textbf{5.59 (0.33)} & \textbf{3.65 (0.20)} \\
\texttt{onPure}$^\dagger$  & 1.25 (0.06) & \textbf{1.30 (0.06)} & \textbf{1.35 (0.08)} & \textbf{1.42 (0.07)} & \textbf{1.32 (0.06)} 
        & 6.09 (0.63) & 7.41 (0.85) & 7.23 (0.74) & 7.97 (0.86) & 7.28 (0.68) \\
\texttt{offPure}$^\dagger$ & 6.69 (0.33) & 6.94 (0.38) & 7.04 (0.51) & 7.48 (0.49) & 6.99 (0.57) 
        & 30.52 (3.14) & 36.98 (4.21) & 36.11 (3.69) & 39.66 (4.20) & 36.32 (3.44) \\
\texttt{AT}      & 1.15 (0.01) & 21.70 (0.67) & 22.88 (0.74) & 30.91 (0.71) & 20.43 (0.68) 
        & 1.33 (0.05) & 14.00 (1.25) & 13.09 (1.60) & 23.19 (1.09) & 14.78 (2.13) \\
\hline
\end{tabular}
}
\end{table*}

%%%%%%%%%%%%%%%%%%%%%%%%%%%%%%%
%\addtolength{\textheight}{-.5in}
\section{Conclusions}

This paper set out to protect Bayesian predictive models against adversarial attacks, an underexplored problem still under foundational debate, and to do so in a Bayesian way, explicitly modeling uncertainty about the attacker. To this end, we have introduced a statistically rigorous Bayesian framework. By modeling the adversary's actions through a stochastic adversarial channel, our framework makes all probabilistic assumptions transparent and exploits the structure of attacks, which generic robust Bayesian methods treat agnostically, yielding two complementary strategies: a \textit{reactive defense}, providing a principled foundation for adversarial purification, and a \textit{proactive defense}, generalizing AT by incorporating uncertainty about the attack. We formally prove that prominent defenses like AT and RS are recovered as limiting cases of our framework. Since attacks on Bayesian models can corrupt their uncertainty as well as their predictions, we evaluated our defenses on both fronts. Our empirical results validate the proactive approach, showcasing that explicitly modeling adversarial uncertainty confers superior robustness in both model accuracy and the quality of predictive distributions: our proactive defenses outperform AT, including AT against a multi-step adversary, in robustness and calibration, and, on CIFAR-10, a Bayesian wide residual network trained with our approach matches the $L_2$ AutoAttack accuracy of adversarially trained deterministic networks.

Our work opens several avenues for future research. For reactive defenses, the primary challenge is computational efficiency; promising directions include exploring likelihood-free inference methods to make the online reactive defense practical at scale. For proactive defenses, future work involves developing more sophisticated learned adversarial channels, for instance by training an amortized generative adversary. Another direction is to move beyond the i.i.d.\ setting, to temporal models and settings in which predictions are made sequentially as data accumulate. There, the adversarial channel could itself be updated over time: once an attack is detected, the observed corruptions become data with which to refine our beliefs about the adversary, so that the defense adapts to the threats actually faced. 
From a theoretical perspective, PAC-Bayesian generalization bounds for adversarial robustness within our framework remain to be derived.
More broadly, a key next step is to move beyond adversarially robust prediction towards    adversarially robust decision-making, integrating the rich, reliable predictive distributions produced by our framework into decision-theoretic pipelines in high-stakes applications within safety-critical domains.

\subsubsection*{Data Availability}

Datasets can be accessed at:  MNIST (\url{http://yann.lecun.com/exdb/mnist/}), FashionMNIST (\url{https://github.com/zalandoresearch/fashion-mnist}), CIFAR-10 (\url{https://www.cs.toronto.edu/~kriz/cifar.html}), Wine Quality and Energy Efficiency (\url{https://archive.ics.uci.edu/}), and California Housing (\url{https://lib.stat.cmu.edu/datasets/}). Code and hyperparameters are available at \ifanon\url{https://anonymous.4open.science/r/advDef}\else\url{https://github.com/pablogarciarce/advDef}\fi.

\FloatBarrier

\subsubsection*{Broader Impact Statement}

Our contribution is defensive and inferential: we formalize adversarial robustness for Bayesian predictive models by treating the perturbation as a stochastic channel within an explicit probabilistic model, and study the proactive and reactive defenses that this formulation induces. We expect it to be most useful in high-stakes, safety-critical domains, where Bayesian predictive models are typically deployed and reliable predictions and uncertainty estimates under attack matter most. Yet instantiating the framework requires specifying concrete attack channels, and our evaluation includes entropy-targeting attacks that degrade uncertainty estimates rather than accuracy alone, so the material is dual-use. The attack constructions we employ (single-step and multi-step PGD, margin-based and predictive-entropy objectives, PGD$^+$ and AutoAttack) are standard and already public, and the generator of our learned channel is trained as part of the defense rather than evaluated as a standalone attack. On balance, since the attacks involved are already public, while making the assumed threat explicit helps practitioners judge when a defense applies, we expect the benefits of this work to outweigh its risks.

%% TMLR NOTE: the sample instructs authors to add acknowledgments "only once
%% your submission is accepted and deanonymized". The \ifanon switch below
%% does exactly that; delete the placeholder line if you prefer no marker at all.
\ifanon

\else
  \subsubsection*{Acknowledgments}
Research supported by the AFOSR-EOARD award FA8655-21-1-7042, the Spanish Ministry of Science program PID2021-124662OB-I00 and the 2026 Leonardo Grant for Scientific Research and Cultural Creation from the BBVA Foundation. The BBVA Foundation accepts no responsibility for the opinions, statements and contents included in the project and/or the results thereof, which are entirely the responsibility of the authors.
PGA is a staff member hired under the Generation D initiative, promoted by Red.es, an organisation attached to the Ministry for Digital Transformation and the Civil Service, for the attraction and retention of talent through grants and training contracts, financed by the Recovery, Transformation and Resilience Plan through the European Union's Next Generation funds.
\fi

\bibliographystyle{tmlr}
\bibliography{biblio}

@article{bissiri2016general,
  title={A general framework for updating belief distributions},
  author={Bissiri, Pier Giovanni and Holmes, Chris C and Walker, Stephen G},
  journal={Journal of the Royal Statistical Society Series B: Statistical Methodology},
  volume={78},
  number={5},
  pages={1103--1130},
  year={2016},
  publisher={Oxford University Press}
}

@article{miller2019robust,
  title={Robust {B}ayesian inference via coarsening},
  author={Miller, Jeffrey W and Dunson, David B},
  journal={Journal of the American Statistical Association},
  year={2019},
volume = {114},
number = {527},
pages = {1113--1125},
publisher = {Taylor & Francis},
doi = {10.1080/01621459.2018.1469995}
}

@book{berger2000bayesian,
  title={Robust {B}ayesian Analysis},
  author={R{\'\i}os Insua, David and Ruggeri, Fabrizio},
  year={2000},
  publisher={Springer}
}

@inproceedings{kingma2014adam,
  title={Adam: A method for stochastic optimization},
  author={Kingma, Diederik P. and Ba, Jimmy},
  booktitle={International Conf. on Learning Representations},
  year={2015}
}

@inproceedings{tramer2017ensemble,
  title={Ensemble adversarial training: Attacks and defenses},
  author={Tram{\`e}r, Florian and Kurakin, Alexey and Papernot, Nicolas and Goodfellow, Ian and Boneh, Dan and McDaniel, Patrick},
  booktitle={International Conf. on Learning Representations},
  year={2018}
}

@article{knoblauch2022optimization,
  title={An optimization-centric view on Bayes' rule: Reviewing and generalizing variational inference},
  author={Knoblauch, Jeremias and Jewson, Jack and Damoulas, Theodoros},
  journal={JMLR},
  volume={23},
  number={132},
  pages={1--109},
  year={2022}
}

@inproceedings{wicker2021bayesian,
  title={Bayesian inference with certifiable adversarial robustness},
  author={Wicker, Matthew and Laurenti, Luca and Patane, Andrea and Chen, Zhuotong and Zhang, Zheng and Kwiatkowska, Marta},
  booktitle={International Conf. on Artificial Intelligence and Statistics},
  pages={2431--2439},
  year={2021},
  organization={PMLR}
}

@inproceedings{goodfellow2014explaining,
  title={Explaining and harnessing adversarial examples},
  author={Goodfellow, Ian J and Shlens, Jonathon and Szegedy, Christian},
  booktitle={International Conf. on Learning Representations},
  year={2015}
}

@inproceedings{de2021adversarial,
  title={Adversarial robustness guarantees for random deep neural networks},
  author={De Palma, Giacomo and Kiani, Bobak and Lloyd, Seth},
  booktitle={International Conf. on Machine Learning},
  pages={2522--2534},
  year={2021},
  organization={PMLR}
}

@article{feng2024attacking,
  title={Attacking {B}ayes: On the Adversarial Robustness of {B}ayesian Neural Networks},
  author={Feng, Yunzhen and Rudner, Tim G.J. and Tsilivis, Nikolaos and Kempe, Julia},
  journal={Transactions on Machine Learning Research},
  year={2024}
}

@book{joseph,
  title={{Adversarial Machine Learning}},
  author={Joseph, Anthony and Nelson, Blaine and Rubinstein, Benjamin and Tygar, J.D.},
  year={2019},
  publisher={Cambridge University Press}
}

@book{vorobeichikantar,
   title={{Adversarial Machine Learning}},
   author={Vorobeychik, Yevgeniy and Kantarcioglu, Murat},
   year={2018},
   publisher={Morgan \& Claypool}
}

@inproceedings{burda2015importance,
  title={Importance weighted autoencoders},
  author={Burda, Yuri and Grosse, Roger and Salakhutdinov, Ruslan},
  booktitle={International Conf. on Learning Representations},
  year={2016}
}

@article{kannan2018adversarial,
  title={Adversarial logit pairing},
  author={Kannan, Harini and Kurakin, Alexey and Goodfellow, Ian},
  journal={arXiv preprint arXiv:1803.06373},
  year={2018}
}

@inproceedings{adversarialClassification2004,
 author = {Dalvi, N. and  Domingos, P. and Mausam and Sumit, S. and Verma, D},
 title = {{Adversarial classification}},
 booktitle ={{Proc. of the Tenth ACM SIGKDD International Conf. on Knowledge Discovery and Data Mining}},
 series = {KDD '04},
 year = {2004},
 isbn = {1-58113-888-1},
 pages = {{99--108}},
 numpages = {10},
 acmid = {1014066},
}

@inproceedings{ye2018bayesian,
  title={Bayesian adversarial learning},
  author={Ye, Nanyang and Zhu, Zhanxing},
  booktitle={Proc. of the 32nd International Conf. on Neural Information Processing Systems},
  pages={6892--6901},
  year={2018},
  organization={Curran Associates Inc.}
}

@inproceedings{madry2018towards,
title={Towards Deep Learning Models Resistant to Adversarial Attacks},
author={Aleksander Madry and Aleksandar Makelov and Ludwig Schmidt and Dimitris Tsipras and Adrian Vladu},
booktitle={International Conf. on Learning Representations},
year={2018},
}

@misc{MNIST,
  title={{THE MNIST DATABASE of handwritten digits}},
  author={LeCun, Yan and Cortes, Corinna and Burges, Christopher},
  howpublished = {\url{http://yann.lecun.com/exdb/mnist/}},
  year={1998}
}

@book{bishop2006pattern,
  title={{Pattern Recognition and Machine Learning}},
  author={Bishop, Christopher M},
  year={2006},
  publisher={{Springer}}
}

@incollection{NIPS2014_5423,
title = {Generative Adversarial Nets},
author = {Goodfellow, Ian and Pouget-Abadie, Jean and Mirza, Mehdi and Xu, Bing and Warde-Farley, David and Ozair, Sherjil and Courville, Aaron and Bengio, Yoshua},
booktitle = {Advances in Neural Information Processing Systems 27},
pages = {2672--2680},
year = {2014},
publisher = {Curran Associates, Inc.}
}

@inproceedings{he2016deep,
  title={Deep residual learning for image recognition},
  author={He, Kaiming and Zhang, Xiangyu and Ren, Shaoqing and Sun, Jian},
  booktitle={Proc. of the IEEE Conf. on computer vision and pattern recognition},
  pages={770--778},
  year={2016}
}

@inproceedings{carlini2017towards,
  title={Towards evaluating the robustness of neural networks},
  author={Carlini, Nicholas and Wagner, David},
  booktitle={2017 IEEE Symposium on Security and Privacy (SP)},
  pages={39--57},
  year={2017},
  organization={IEEE}
}

@InProceedings{carreau2025poisoning,
  title = 	 {Poisoning {B}ayesian Inference via Data Deletion and Replication},
  author =       {Carreau, Matthieu and Naveiro, Roi and Caballero, William N.},
  booktitle = 	 {Proc. of The 28th International Conf. on Artificial Intelligence and Statistics},
  pages = 	 {883--891},
  year = 	 {2025},
  volume = 	 {258},
  month = 	 {03--05 May},
  publisher =    {PMLR}
}

@article{caballero2025posterior,
author = {William N. Caballero and Roi Naveiro and Brian J. Lunday},
title = {{Posterior Attraction with Exponential-Family Likelihoods and Their Conjugate Priors}},
journal = {Bayesian Analysis},
publisher = {International Society for Bayesian Analysis},
pages = {1 -- 29},
year = {2025}
}

@article{grunwald2017inconsistency,
author = {Peter Gr{\"u}nwald and Thijs van Ommen},
title = {{Inconsistency of Bayesian Inference for Misspecified Linear Models, and a Proposal for Repairing It}},
volume = {12},
journal = {Bayesian Analysis},
number = {4},
publisher = {International Society for Bayesian Analysis},
pages = {1069 -- 1103},
year = {2017},
doi = {10.1214/17-BA1085},
URL = {https://doi.org/10.1214/17-BA1085}
}

@article{riosInsua2023,
author = {R{\'\i}os Insua, David and Naveiro, Roi and Gallego, V{\'\i}ctor and Poulos, Jason},
title = {Adversarial Machine Learning: {B}ayesian Perspectives},
journal = {Journal of the American Statistical Association},
volume = {118},
number = {543},
pages = {2195--2206},
year  = {2023},
publisher = {Taylor & Francis},
doi = {10.1080/01621459.2023.2183129},
}

@misc{athanasios_tsanas_energy_2012,
	title = {Energy Efficiency},
	url = {https://archive.ics.uci.edu/dataset/242},
howpublished = {UCI Machine Learning Repository},
	author = {Tsanas, Athanasios and Xifara, Angeliki},
	year = {2012},
}

@misc{wine_data,
  author       = {Cortez, Paulo and Cerdeira, António and Almeida, Fernando and Matos, Telmo and Reis, José},
  title        = {{Wine Quality}},
  year         = {2009},
  howpublished = {UCI Machine Learning Repository},
  doi         = {https://doi.org/10.24432/C56S3T}
}

@article{gallego2024protecting,
  title={Protecting classifiers from attacks},
  author={Gallego, V{\'\i}ctor and Naveiro, Roi and Redondo, Alberto and R{\'\i}os Insua, David and Ruggeri, Fabrizio},
  journal={Statistical Science},
  volume={39},
  number={3},
  pages={449--468},
  year={2024},
  publisher={Institute of Mathematical Statistics}
}

@InProceedings{arce2025evasion,
  title = 	 {Evasion Attacks Against {B}ayesian Predictive Models},
  author =       {Arce, Pablo G. and Naveiro, Roi and Insua, David R\'{i}os},
  booktitle = 	 {Proc. of the 41st Conf. on Uncertainty in Artificial Intelligence},
  pages = 	 {184--202},
  year = 	 {2025},
  editor = 	 {Chiappa, Silvia and Magliacane, Sara},
  volume = 	 {286},
  month = 	 {21--25 Jul},
  publisher =    {PMLR}
}

@inproceedings{cai2018curriculum,
  title={Curriculum adversarial training},
  author={Cai, Qi-Zhi and Du, Min and Liu, Chang and Song, Dawn},
  booktitle={Proceedings of the Twenty-Seventh International Joint Conference on Artificial Intelligence},
  year={2018}
}

@article{dong2024enhancing,
  title={Enhancing Adversarial Robustness via Uncertainty-Aware Distributional Adversarial Training},
  author={Dong, Junhao and Qu, Xinghua and Wang, Z Jane and Ong, Yew-Soon},
  journal={arXiv preprint arXiv:2411.02871},
  year={2024}
}

@inproceedings{lin2024adversarial,
  title={Adversarial training on purification ({AToP}): Advancing both robustness and generalization},
  author={Lin, Guang and Li, Chao and Zhang, Jianhai and Tanaka, Toshihisa and Zhao, Qibin},
  booktitle={International Conf. on Learning Representations},
  year={2024}
}

@inproceedings{nie2022diffusion,
  title={Diffusion models for adversarial purification},
  author={Nie, Weili and Guo, Brandon and Huang, Yujia and Xiao, Chaowei and Vahdat, Arash and Anandkumar, Anima},
  booktitle={International Conf. on Machine Learning},
  year={2022}
}

@inproceedings{cohen2019certified,
  title={Certified adversarial robustness via randomized smoothing},
  author={Cohen, Jeremy and Rosenfeld, Elan and Kolter, Zico},
  booktitle={International Conf. on Machine Learning},
  pages={1310--1320},
  year={2019},
  organization={PMLR}
}

@article{balaji2019instance,
  title={Instance adaptive adversarial training: Improved accuracy tradeoffs in neural nets},
  author={Balaji, Yogesh and Goldstein, Tom and Hoffman, Judy},
  journal={arXiv preprint arXiv:1910.08051},
  year={2019}
}

@inproceedings{papamakarios2019sequential,
  title={Sequential neural likelihood: Fast likelihood-free inference with autoregressive flows},
  author={Papamakarios, George and Sterratt, David and Murray, Iain},
  booktitle={The 22nd International Conf. on artificial intelligence and statistics},
  pages={837--848},
  year={2019},
  organization={PMLR}
}

@article{zhao2023survey,
  title={A survey of large language models},
  author={Zhao, Wayne Xin and Zhou, Kun and Li, Junyi and Tang, Tianyi and Wang, Xiaolei and Hou, Yupeng and Min, Yingqian and Zhang, Beichen and Zhang, Junjie and Dong, Zican and others},
  journal={arXiv preprint arXiv:2303.18223},
  volume={1},
  number={2},
  year={2023}
}

@inproceedings{zhang2019theoretically,
  title={Theoretically principled trade-off between robustness and accuracy},
  author={Zhang, Hongyang and Yu, Yaodong and Jiao, Jiantao and Xing, Eric and El Ghaoui, Laurent and Jordan, Michael},
  booktitle={International Conf. on Machine Learning},
  pages={7472--7482},
  year={2019},
  organization={PMLR}
}

@article{blei2017variational,
  title={Variational inference: A review for statisticians},
  author={Blei, David M and Kucukelbir, Alp and McAuliffe, Jon D},
  journal={Journal of the American statistical Association},
  volume={112},
  number={518},
  pages={859--877},
  year={2017},
  publisher={Taylor \& Francis}
}

@inproceedings{ranganath2014black,
  title={Black box variational inference},
  author={Ranganath, Rajesh and Gerrish, Sean and Blei, David},
  booktitle={Proc. of the Seventeenth International Conf. on Artificial Intelligence and Statistics},
  pages={814--822},
  year={2014},
  organization={PMLR}
}

@article{sabanayagam2025generalization,
  title={Generalization Certificates for Adversarially Robust {B}ayesian Linear Regression},
  author={Sabanayagam, Mahalakshmi and Tsuchida, Russell and Ong, Cheng Soon and Ghoshdastidar, Debarghya},
  journal={arXiv preprint arXiv:2502.14298},
  year={2025}
}

@book{murphy2023probabilistic,
  title={Probabilistic machine learning: Advanced topics},
  author={Murphy, Kevin P},
  year={2023},
  publisher={MIT press}
}

@article{xiao2017online,
  author={Xiao, Han and Rasul, Kashif and Vollgraf, Roland},
  title={{Fashion-MNIST}: a novel image dataset for benchmarking machine learning algorithms},
  journal={arXiv preprint arXiv:1708.07747},
  year={2017}
}

@article{cranmer2020frontier,
  title={The frontier of simulation-based inference},
  author={Cranmer, Kyle and Brehmer, Johann and Louppe, Gilles},
  journal={Proc. of the National Academy of Sciences},
  volume={117},
  number={48},
  pages={30055--30062},
  year={2020},
  publisher={National Academy of Sciences}
}

@article{pace1997sparse,
  title={Sparse spatial autoregressions},
  author={Pace, R Kelley and Barry, Ronald},
  journal={Statistics \& Probability Letters},
  volume={33},
  number={3},
  pages={291--297},
  year={1997},
  publisher={Elsevier}
}

@techreport{krizhevsky2009learning,
  title={Learning multiple layers of features from tiny images},
  author={Krizhevsky, Alex},
  institution={University of Toronto},
  year={2009}
}

@inproceedings{zagoruyko2016wide,
  title={Wide residual networks},
  author={Zagoruyko, Sergey and Komodakis, Nikos},
  booktitle={Proceedings of the British Machine Vision Conference},
  year={2016}
}

@inproceedings{croce2020reliable,
  title={Reliable evaluation of adversarial robustness with an ensemble of diverse parameter-free attacks},
  author={Croce, Francesco and Hein, Matthias},
  booktitle={International Conf. on machine learning},
  pages={2206--2216},
  year={2020},
  organization={PMLR}
}

@article{gallego2026opponent,
  title={Opponent aware reinforcement learning},
  author={Gallego, V{\'\i}ctor and Naveiro, Roi and Insua, David R{\'\i}os and G{\'o}mez-Ullate, David},
  journal={European Journal of Operational Research},
  year={2026},
  publisher={Elsevier}
}

@article{arnold1989compatible,
  title={Compatible conditional distributions},
  author={Arnold, Barry C. and Press, S. James},
  journal={Journal of the American Statistical Association},
  volume={84},
  number={405},
  pages={152--156},
  year={1989},
  doi={10.1080/01621459.1989.10478750}
}

@inproceedings{he2019probgan,
  title={{ProbGAN}: Towards probabilistic {GAN} with theoretical guarantees},
  author={He, Hao and Wang, Hao and Lee, Guang-He and Tian, Yonglong},
  booktitle={International Conf. on Learning Representations},
  year={2019}
}

@inproceedings{saatci2017bayesian,
  title={Bayesian {GAN}},
  author={Saatci, Yunus and Wilson, Andrew Gordon},
  booktitle={Advances in Neural Information Processing Systems},
  volume={30},
  pages={3622--3631},
  year={2017}
}

@article{heckerman2000dependency,
  title={Dependency networks for inference, collaborative filtering, and data visualization},
  author={Heckerman, David and Chickering, David Maxwell and Meek, Christopher and Rounthwaite, Robert and Kadie, Carl},
  journal={Journal of Machine Learning Research},
  volume={1},
  pages={49--75},
  year={2000}
}

@article{chen2015behaviour,
  title={Behaviour of the {G}ibbs sampler when conditional distributions are potentially incompatible},
  author={Chen, Shyh-Huei and Ip, Edward H.},
  journal={Journal of Statistical Computation and Simulation},
  volume={85},
  number={16},
  pages={3266--3275},
  year={2015},
  doi={10.1080/00949655.2014.968159}
}

@inproceedings{wong2020fast,
  title     = {Fast is better than free: Revisiting adversarial training},
  author    = {Wong, Eric and Rice, Leslie and Kolter, J. Zico},
  booktitle = {International Conf. on Learning Representations},
  year      = {2020},
}

@article{de2022make,
  title={Make some noise: Reliable and efficient single-step adversarial training},
  author={de Jorge Aranda, Pau and Bibi, Adel and Volpi, Riccardo and Sanyal, Amartya and Torr, Philip and Rogez, Gr{\'e}gory and Dokania, Puneet},
  journal={Advances in Neural Information Processing Systems},
  volume={35},
  pages={12881--12893},
  year={2022}
}

@inproceedings{andriushchenko2020square,
  title={Square attack: a query-efficient black-box adversarial attack via random search},
  author={Andriushchenko, Maksym and Croce, Francesco and Flammarion, Nicolas and Hein, Matthias},
  booktitle={European conference on computer vision},
  pages={484--501},
  year={2020},
  organization={Springer}
}

@inproceedings{liu2019advbnn,
  title     = {Adv-{BNN}: Improved Adversarial Defense through Robust {B}ayesian Neural Network},
  author    = {Liu, Xuanqing and Li, Yao and Wu, Chongruo and Hsieh, Cho-Jui},
  booktitle = {International Conf. on Learning Representations},
  year      = {2019}
}

@inproceedings{athalye2018obfuscated,
  title={Obfuscated gradients give a false sense of security: Circumventing defenses to adversarial examples},
  author={Athalye, Anish and Carlini, Nicholas and Wagner, David},
  booktitle={International Conf. on Machine Learning},
  pages={274--283},
  year={2018},
  organization={PMLR}
}

@inproceedings{doan2022bayesian,
  title     = {{B}ayesian Learning with Information Gain Provably Bounds Risk for a Robust Adversarial Defense},
  author    = {Doan, Bao Gia and Abbasnejad, Ehsan and Shi, Javen Qinfeng and Ranasinghe, Damith C.},
  booktitle = {Proceedings of the 39th International Conf. on Machine Learning},
  series    = {PMLR},
  volume    = {162},
  pages     = {5309--5323},
  year      = {2022}
}

@inproceedings{croce2021robustbench,
  title     = {{RobustBench}: a standardized adversarial robustness benchmark},
  author    = {Croce, Francesco and Andriushchenko, Maksym and Sehwag, Vikash and Debenedetti, Edoardo and Flammarion, Nicolas and Chiang, Mung and Mittal, Prateek and Hein, Matthias},
  booktitle = {Thirty-fifth Conference on Neural Information Processing Systems Datasets and Benchmarks Track},
  year      = {2021}
}

@inproceedings{wu2020adversarial,
  title     = {Adversarial Weight Perturbation Helps Robust Generalization},
  author    = {Wu, Dongxian and Xia, Shu-Tao and Wang, Yisen},
  booktitle = {Advances in Neural Information Processing Systems},
  volume    = {33},
  pages     = {2958--2969},
  year      = {2020}
}

@inproceedings{rice2020overfitting,
  title     = {Overfitting in adversarially robust deep learning},
  author    = {Rice, Leslie and Wong, Eric and Kolter, J. Zico},
  booktitle = {Proceedings of the 37th International Conf. on Machine Learning},
  series    = {PMLR},
  volume    = {119},
  pages     = {8093--8104},
  year      = {2020}
}

@article{naveiro2026steering,
  title={Steering {B}ayesian Inference and Decisions through Data Poisoning},
  author={Naveiro, Roi and Caballero, William N. and Maro{\~n}as, Juan},
  journal={arXiv preprint arXiv:2611.19557},
  year={2026}
}

%% ==========================================================================
%% APPENDIX. Formerly the standalone Supplementary Material document of the
%% Bayesian Analysis submission, merged here as TMLR convention places the
%% appendix after the references in a single PDF. The text is reproduced
%% BYTE-FOR-BYTE; only the SM-N cross-references in the main body were
%% retargeted to \ref{} (shown in the review colour).
%%
%% \clearpage starts the appendix on a fresh page and flushes any pending
%% floats. The official TMLR sample places \appendix immediately after
%% \bibliographystyle with no break, but TMLR routes submissions by main-body
%% length (2-week reviews up to 12 pages before references, 4 weeks above),
%% so a clean boundary between the main paper and the appendix is worth having.
%% ==========================================================================
\clearpage
\appendix

\paragraph{Appendix.} The Appendix contains: proofs of all theoretical results (Appendix~\ref{app:proofs}); a counterexample showing that the reactive and proactive predictive distributions differ (Appendix~\ref{sec:counterexample}); an empirical validation of our surrogate variational bound (Appendix~\ref{app:tightness}); a consistent gradient estimator for the exact objective (Appendix~\ref{app:unbiased_grad}); the specification of the adversarial channels used in our experiments (Appendix~\ref{app:channel}); additional experimental details, computational costs, and results (Appendix~\ref{app:exps}); and a proof that the conditional distributions defining the Bayesian Adversarial Learning sampler are incompatible, so that the Gibbs sampler they define does not target a valid posterior (Appendix~\ref{sec:bal}).

%%%%%%%%%%%%%%%%%%%%%%%%%%%%%%%%
\section{Proofs}\label{app:proofs}

\reactiveppd*
\begin{proof}
The chain rule, together with the conditional independence of $\mathbf{x}_j$ from $\mathcal D$ given $(\mathbf{x}'_j,\theta,\phi)$, gives
\[
p(\mathbf{x}_j,\theta,\phi\mid \mathbf{x}'_j,\mathcal D)
=
p(\mathbf{x}_j\mid \mathbf{x}'_j,\theta,\phi)
p(\theta,\phi\mid \mathbf{x}'_j,\mathcal D).
\]
Substitution into \eqref{eq:op_predictive} provides \eqref{eq:exact_reactive_nested}.

To obtain the ratio form \eqref{eq:exact_reactive_ratio}, conditionally on $(\theta,\phi)$, the joint density/mass of $(y_j,\mathbf{x}'_j)$ is obtained by marginalizing the latent clean input
\[
p(y_j,\mathbf{x}'_j\mid \theta,\phi)
=
\int
p(y_j\mid \mathbf{x}_j,\theta)
p(\mathbf{x}'_j\mid \mathbf{x}_j,\theta)
p(\mathbf{x}_j\mid \phi)
\,d\mathbf{x}_j.
\]
Averaging with respect to the posterior $p(\theta,\phi\mid \mathcal D)$ gives
\[
p(y_j,\mathbf{x}'_j\mid \mathcal D)
=
\mathbb{E}_{(\theta,\phi)\mid \mathcal D}
\left[
\mathbb{E}_{\mathbf{x}_j\mid \phi}
\left\{
p(y_j\mid \mathbf{x}_j,\theta)
p(\mathbf{x}'_j\mid \mathbf{x}_j,\theta)
\right\}
\right].
\]
Similarly,
\[
p(\mathbf{x}'_j\mid \mathcal D)
=
\mathbb{E}_{(\theta,\phi)\mid \mathcal D}
\left[
m_{\theta,\phi}(\mathbf{x}'_j)
\right].
\]
Dividing $p(y_j,\mathbf{x}'_j\mid \mathcal D)$ by $p(\mathbf{x}'_j\mid \mathcal D)$ yields \eqref{eq:exact_reactive_ratio}.
\end{proof}

\proactivefactorization*

\begin{proof}
Integrating \eqref{eq:proactive_joint} over the latent adversarial inputs gives
\[
p(\mathcal D\mid\theta,\phi)
=
\prod_{i=1}^N
\left[
p(\mathbf{x}_i\mid\phi)
\int
p(y_i\mid\mathbf{x}'_i,\theta)
p(\mathbf{x}'_i\mid\mathbf{x}_i,\theta)
\,d\mathbf{x}'_i
\right]
=
\prod_{i=1}^N
p(\mathbf{x}_i\mid\phi)
r_\theta(y_i\mid\mathbf{x}_i).
\]
Multiplying by the independent prior $p(\theta)p(\phi)$ separates all terms involving $\theta$ from all terms involving $\phi$, proving \eqref{eq:proactive_posterior}--\eqref{eq:robust_theta_posterior} after normalization. For prediction, Figure~\ref{fig:train} implies $y_j\perp(\phi,\mathcal D)\mid(\mathbf{x}'_j,\theta)$, so
\[
p(y_j\mid\mathbf{x}'_j,\mathcal D)
=
\iint
p(y_j\mid\mathbf{x}'_j,\theta)
p(\theta,\phi\mid\mathbf{x}'_j,\mathcal D)
\,d\theta\,d\phi
=
\int
p(y_j\mid\mathbf{x}'_j,\theta)
p(\theta\mid\mathbf{x}'_j,\mathcal D)
\,d\theta.
\]
Using $p(\theta\mid\mathbf{x}'_j,\mathcal D)\approx p_{\mathrm{R}}(\theta\mid\mathcal D)$ gives \eqref{eq:proactive_ppd}.
\end{proof}

\genpos*

\begin{proof}
For any $q\in\mathcal Q$, the SELBO can be written as
\[
\tilde{\mathcal L}(q)
=
\int_\Theta q(\theta)\{-\bar L_{\mathcal D}(\theta)\}\,d\theta
-
\int_\Theta q(\theta)
\log\frac{q(\theta)}{p(\theta)}
\,d\theta,
\]
where the second integral is finite since $q\in\mathcal Q$, and the first one lies in $[-\infty,\infty)$: by the Donsker--Varadhan variational inequality, $\int_\Theta q(\theta)\{-\bar L_{\mathcal D}(\theta)\}^{+}\,d\theta\le \mathrm{KL}\{q(\theta)\|p(\theta)\}+\log\{1+Z(\mathcal D)\}<\infty$, with $\{\cdot\}^{+}$ the positive part. From \eqref{eq:generalized_posterior},
\[
-
\bar L_{\mathcal D}(\theta)
=
\log q^*(\theta\mid\mathcal D)
-
\log p(\theta)
+
\log Z(\mathcal D).
\]
Substituting this expression into the previous one and canceling the $\log p(\theta)$ terms gives
\[
\tilde{\mathcal L}(q)
=
\log Z(\mathcal D)
-
\int_\Theta
q(\theta)
\log\frac{q(\theta)}{q^*(\theta\mid\mathcal D)}
\,d\theta,
\]
which proves \eqref{eq:surrogate_kl_identity}. The bound and uniqueness statement follow from Gibbs' inequality, the equality being attained since $q^*\in\mathcal Q$ by assumption.
\end{proof}

\atmaplimit*

\begin{proof}
Under the deterministic channel, the channel-averaged loss defining \eqref{eq:gen_posterior_eta} collapses to a single evaluation,
\[
\bar\ell_i(\theta)
=
\mathbb{E}_{\mathbf{x}'_i\mid\mathbf{x}_i,y_i,\theta}
\left[\ell_\theta(\mathbf{x}'_i,y_i)\right]
=
\ell_\theta\left(A_\theta(\mathbf{x}_i,y_i),y_i\right)
=
\max_{\mathbf{z}\in\mathcal B_\epsilon(\mathbf{x}_i)}\ell_\theta(\mathbf{z},y_i),
\]
the last equality holding by definition of $A_\theta$. Substituting into \eqref{eq:gen_posterior_eta}, and noting the constant flat prior, maximizing $p_{\mathrm{AT},\eta}$ is equivalent to minimizing $\sum_i \max_{\mathbf{z}\in\mathcal B_\epsilon(\mathbf{x}_i)}\ell_\theta(\mathbf{z},y_i)$, which is \eqref{eq:at_objective}. The learning rate $\eta>0$ scales the objective and hence leaves its minimizers unchanged.
\end{proof}

\rslimit*

\begin{proof}
With fixed $(\theta_{\mathrm{MAP}},\phi_{\mathrm{MAP}})$, \eqref{eq:exact_reactive_nested} reduces to an expectation over the latent clean input. Bayes' rule gives
\(
p(\mathbf{x}\mid\mathbf{x}'_j,\theta_{\mathrm{MAP}},\phi_{\mathrm{MAP}})
\propto
\mathcal N(\mathbf{x}'_j;\mathbf{x},\sigma^2I)
p(\mathbf{x}\mid\phi_{\mathrm{MAP}}).
\)
Under the flat prior, this posterior is $\mathcal N(\mathbf{x};\mathbf{x}'_j,\sigma^2I)$. This follows from the symmetry of the Gaussian kernel in $(\mathbf{x},\mathbf{x}'_j)$: once the flat prior contributes no additional weighting, the normalized kernel in $\mathbf{x}$ is centered at the observed corrupted input $\mathbf{x}'_j$. Substituting this distribution into the inner expectation of \eqref{eq:exact_reactive_nested} yields \eqref{eq:rs_predictive}. If $p(\mathbf{x}\mid\phi_{\mathrm{MAP}})$ is bounded and approximately constant over the region where $\mathcal N(\mathbf{x}'_j,\sigma^2I)$ places most of its mass (for large $d$, a thin shell of radius about $\sigma\sqrt d$ around $\mathbf{x}'_j$), the same argument yields \eqref{eq:rs_predictive} approximately.
\end{proof}

%%%%%%%%%%%%%%%%%%%%%%%%%%%%%%%%
\section{Distinctness of the Reactive and Proactive Predictive Distributions}
\label{sec:counterexample}

Using a counterexample, this section proves that the PPDs derived from the proactive (during training) and reactive (during operations) defense strategies are, in general, different. We construct a simple linear-Gaussian model where most relevant quantities can be computed analytically, revealing differences between both approaches.

\paragraph*{Model Setup.}
Define a simple model with the following components: a clean univariate covariate model, $x \sim \mathcal{N}(0, \sigma_x^2)$ with known variance $\sigma_x^2$ (thus $\phi = \sigma_x^2 $); a known linear adversarial channel independent of model parameters, $x' = (1+\varepsilon)x + \nu$ where $\nu \sim \mathcal{N}(0, \sigma_\delta^2)$ is independent of $x$, with known $\sigma_\delta^2{}\ge 0$; a linear model, $y \mid x, \theta \sim \mathcal{N}(\theta x, \sigma_y^2)$ with known $\sigma_y^2$; and a Gaussian prior on the single unknown parameter, $\theta \sim \mathcal{N}(0, \sigma_\theta^2)$. Assume that $\sigma_x^2,\sigma_y^2,\sigma_\theta^2>0$ and $(1+\varepsilon)^2\sigma_x^2+\sigma_\delta^2>0$, so that $x'$ is not degenerate. When $\sigma_\delta^2=0$, the channel is deterministic and falls outside the density assumption of Section~\ref{operations}, but the reactive PPD remains defined through its nested form, $p(x\mid x')$ below is then read as the point mass at $x'/(1+\varepsilon)$, and all expressions coincide with the $\sigma_\delta^2\to0$ limits of those for $\sigma_\delta^2>0$.

\paragraph*{Strategy 1: Reactive Defense.}
Suppose the model is trained on a single clean data pair $\mathcal{D} = \{(x_1, y_1)\}$.
Upon observing the possibly corrupted covariate $x'$, the robust PPD in equation~\eqref{eq:op_predictive} from the main text can be written as
\begin{align*}
    p(y \mid x', \mathcal{D}) = \iint p(y \mid x, \theta) \, p(x \mid x') \, p(\theta \mid \mathcal{D})\,dx \,d\theta.
\end{align*}
In this setup, the posterior on the parameter is Gaussian, $p(\theta \mid \mathcal{D}) = \mathcal{N}(\theta; \mu_{\text{RE}}, v_{\text{RE}})$, with
\begin{align*}
  \mu_{\text{RE}}=\frac{\sigma_\theta^{2}x_1 y_1}{\sigma_y^{2}+\sigma_\theta^{2}x_1^{2}},
  \quad
  v_{\text{RE}}=\frac{\sigma_\theta^{2}\sigma_y^{2}}
         {\sigma_y^{2}+\sigma_\theta^{2}x_1^{2}}.
\end{align*}
The posterior $p(x \mid x')$ for the latent clean covariate is also Gaussian, with mean $\mathbb{E}[x \mid x'] = \lambda x'$ and variance $\operatorname{Var}(x \mid x') = \tau^2$, where
\begin{align*}
    \lambda =
  \frac{(1+\varepsilon)\sigma_x^{2}}
       {(1+\varepsilon)^{2}\sigma_x^{2}+\sigma_\delta^{2}},
  \qquad
  \tau^{2} = (1-\lambda(1+\varepsilon))\sigma_x^{2}
  {}=\frac{\sigma_x^{2}\sigma_\delta^{2}}
       {(1+\varepsilon)^{2}\sigma_x^{2}+\sigma_\delta^{2}}.
\end{align*}
 When $\sigma_\delta^2>0$, so that $\tau^2>0$, the robust PPD is non-Gaussian. In all cases, its first two moments are
\begin{align*}
    \mathbb{E}_{\text{RE}}[y \mid x', \mathcal{D}] &= \lambda x' \mu_{\text{RE}}, \\
    \operatorname{Var}_{\text{RE}}[y \mid x', \mathcal{D}] &= \sigma_y^2 + \lambda^2 x'^2 v_{\text{RE}} + \tau^2(\mu_{\text{RE}}^2 + v_{\text{RE}}).
\end{align*}

\paragraph*{Strategy 2: Proactive Defense.}
This strategy embeds the adversarial channel within the training likelihood. Following the channel-marginalized likelihood of Proposition~\ref{prop:proactive_factorization}, the resulting \emph{adversary-aware} likelihood is $r_\theta(y_1\mid x_1) = \mathcal{N}(y_1; (1+\varepsilon)\theta x_1, \, \sigma_y^2 + \theta^2 \sigma_\delta^2)$. As a consequence, when $\sigma_\delta^2>0$, the posterior $p_{\mathrm{R}}(\theta \mid \mathcal{D})$ is non-Gaussian. At test time, the PPD obtained by integrating over this posterior is
\[
    p_{\mathrm{R}}(y \mid x', \mathcal{D}) = \int \mathcal{N}(y; \theta x', \sigma_y^2) \, p_{\mathrm{R}}(\theta \mid \mathcal{D}) \, d\theta.
\]
Denoting the mean and variance of this posterior as $\mu_{\text{PR}}$ and $v_{\text{PR}}$, the predictive moments are
\begin{align*}
    \mathbb{E}_{\text{PR}}[y \mid x', \mathcal{D}] &= x' \mu_{\text{PR}}, \\
    \operatorname{Var}_{\text{PR}}[y \mid x', \mathcal{D}] &= \sigma_y^2 + x'^2 v_{\text{PR}}.
\end{align*}

\paragraph*{Conclusion.} Both PPDs coincide for all $x'$ if and only if $\sigma_\delta^2=0$ and $(1+\varepsilon)^2=1$. In particular, they differ whenever the channel is noisy. Indeed, assume that both PPDs were identical for all $x'$. Then, from the condition of equal variances, we must have
$$\sigma_y^2 + \lambda^2 x'^2 v_{\text{RE}} + \tau^2(\mu_{\text{RE}}^2 + v_{\text{RE}}) = \sigma_y^2 + x'^2 v_{\text{PR}}.$$
Rearranging this expression yields
$$x'^2 (v_{\text{PR}} - \lambda^2 v_{\text{RE}}) = \tau^2(\mu_{\text{RE}}^2 + v_{\text{RE}}).$$
For this equality to hold for all possible test inputs $x'$, the left-hand side, which is a function of $x'$, must equal the right-hand side, which is a constant. This is only possible if both sides are zero. For the right-hand side to be zero, given that $\mu_{\text{RE}}^2 + v_{\text{RE}} > 0$, it is necessary that $\tau^2 = 0$, that is, $\sigma_\delta^2=0$. This condition corresponds to the degenerate case of a noise-free, perfectly invertible adversarial channel. It is necessary but not sufficient. When $\sigma_\delta^2=0$, we have $\lambda=1/(1+\varepsilon)$ and $\tau^2=0$, the proactive likelihood $\mathcal N(y_1;(1+\varepsilon)\theta x_1,\sigma_y^2)$ yields a Gaussian posterior, and both PPDs are normal, namely
\[
\mathcal N\!\left(y;\,\frac{x'\mu_{\text{RE}}}{1+\varepsilon},\,\sigma_y^2+\frac{x'^2 v_{\text{RE}}}{(1+\varepsilon)^2}\right)
\quad\text{and}\quad
\mathcal N\!\left(y;\,x'\mu_{\text{PR}},\,\sigma_y^2+x'^2 v_{\text{PR}}\right),
\]
with
\[
\mu_{\text{PR}}=\frac{\sigma_\theta^{2}(1+\varepsilon)x_1 y_1}{\sigma_y^{2}+\sigma_\theta^{2}(1+\varepsilon)^2x_1^{2}},
\qquad
v_{\text{PR}}=\frac{\sigma_\theta^{2}\sigma_y^{2}}{\sigma_y^{2}+\sigma_\theta^{2}(1+\varepsilon)^2x_1^{2}}.
\]
These coincide for all $x'$ if and only if $v_{\text{PR}}=v_{\text{RE}}/(1+\varepsilon)^2$ and $\mu_{\text{PR}}=\mu_{\text{RE}}/(1+\varepsilon)$. Since $\sigma_y^2>0$, the variance condition holds exactly when $(1+\varepsilon)^2=1$, in which case $\mu_{\text{PR}}=(1+\varepsilon)\mu_{\text{RE}}=\mu_{\text{RE}}/(1+\varepsilon)$, so the mean condition holds as well. Therefore, except for a noise-free channel with $(1+\varepsilon)^2=1$, which leaves the input unchanged or flips its sign, both predictive distributions are distinct.

\section{Empirical Validation of the Variational Bound}
\label{app:tightness}

As discussed in the main text, we optimize a tractable lower bound of the ELBO. To assess whether this formulation sacrifices robustness compared to tighter bounds, we compare our objective against a tighter lower bound inspired by Importance Weighted Autoencoders (IWAE) \citep{burda2015importance}. By drawing $K$ samples from the adversarial channel, we define a refined objective $\mathcal{L}_{\text{IWAE}, K}$ that is non-decreasing in $K$, coincides with the SELBO for $K=1$, and approaches the exact ELBO \eqref{ELBONITA} as $K\to\infty$:
\begin{equation*}
    \mathcal{L}_{\text{IWAE}, K}(\psi) = 
    \mathbb{E}_{\theta \sim q_\psi(\theta)} \left[ \sum_{i=1}^N \mathbb{E}_{\mathbf{x}'_{i, 1:K} \sim p(\cdot|\mathbf{x}_i, \theta)} \left[ \log \left( \frac{1}{K} \sum_{k=1}^K p(y_i \mid \mathbf{x}'_{ik}, \theta) \right) \right] \right] - \text{KL}(q_\psi(\theta) \| p(\theta)).
\end{equation*}

%\paragraph*{Empirical Validation.}
We trained a BNN on MNIST using our MIX adversarial channel under both our proposed SELBO corresponding to equation~\eqref{eq:surrogate_elbo} in the main text and the tighter IWAE objective. For the surrogate ELBO, we varied the number $K$ of MC samples used to estimate the inner expectation (effectively reducing gradient variance), and evaluated clean accuracy and adversarial accuracy against a PGD attack ($L_2$ norm, $\epsilon=2.0$).

The results in Table~\ref{tab:loose_vs_iwae} suggest that our proposed SELBO objective yields robustness comparable to the tighter IWAE-based bound on this benchmark. The differences between both objectives should be read with caution, as the standard deviations are large. In particular, at $K=1$ both objectives coincide, so the difference between the two $K=1$ entries (0.93/0.67 versus 0.96/0.71) is attributable to run-to-run variability. Notably, increasing $K$ in the surrogate objective, which leaves its expectation unchanged and only reduces gradient variance, appears to improve performance and stability, suggesting that the primary bottleneck for training is gradient variance, rather than the Jensen gap itself.

% \begin{table}[h]
% \centering
% \caption{Comparison of Loose vs. IWAE objectives on MNIST ($\epsilon=2.0$). Values report mean (std).}
% \label{tab:loose_vs_iwae}
% \begin{tabular}{lc|cc}
% \hline
% \textbf{Method} & \textbf{K} & \textbf{Clean Accuracy} & \textbf{Adversarial Accuracy} \\
% \hline
% Loose & 1 & 0.93 (0.09) & 0.67 (0.27) \\
% Loose & 5 & 0.98 (0.02) & 0.77 (0.15) \\
% Loose & 10 & 0.98 (0.02) & 0.76 (0.19) \\
% Loose & 50 & \textbf{0.99 (0.01)} & \textbf{0.83 (0.14)} \\
% \hline
% IWAE & 1 & 0.96 (0.05) & 0.71 (0.24) \\
% IWAE & 5 & 0.97 (0.02) & 0.63 (0.14) \\
% IWAE & 10 & 0.97 (0.04) & 0.71 (0.02) \\
% IWAE & 50 & 0.98 (0.01) & 0.74 (0.09) \\
% \hline
% \end{tabular}
% \end{table}

\begin{table}[h]
\centering
\caption{Comparison of surrogate ELBO vs. IWAE objectives on MNIST ($\epsilon=2.0$). Values report mean (std), with best highlighted in bold.}
\label{tab:loose_vs_iwae}
\begin{tabular}{c|cc|cc}
\hline
 & \multicolumn{2}{c|}{\textbf{Surrogate ELBO}} & \multicolumn{2}{c}{\textbf{IWAE}} \\
\textbf{K} & \textbf{Clean Accuracy} & \textbf{Adversarial Accuracy} & \textbf{Clean Accuracy} & \textbf{Adversarial Accuracy} \\
\hline
1  & 0.93 (0.09)          & 0.67 (0.27)          & 0.96 (0.05) & 0.71 (0.24) \\
5  & 0.98 (0.02)          & 0.77 (0.15)          & 0.97 (0.02) & 0.63 (0.14) \\
10 & 0.98 (0.02)          & 0.76 (0.19)          & 0.97 (0.04) & 0.71 (0.02) \\
50 & \textbf{0.99 (0.01)} & \textbf{0.83 (0.14)} & 0.98 (0.01) & 0.74 (0.09) \\
\hline
\end{tabular}
\end{table}

%\FloatBarrier
\section{Derivation of the Consistent Gradient Estimator}\label{app:unbiased_grad}

To make our proactive defense operational, we optimize a tractable lower bound of the ELBO. One could alternatively employ a consistent gradient estimator for the true ELBO. With a reparameterizable posterior $\theta=f(\psi,\boldsymbol{\epsilon}),\ \boldsymbol{\epsilon}\sim p(\boldsymbol{\epsilon})$ (written $f_\psi(\boldsymbol{\epsilon})$ in Section~\ref{sec:protect-train}), the gradient of the actual objective becomes
\begin{equation*} 
        \nabla_\psi \mathcal{L}(\psi) = \mathbb{E}_{\boldsymbol{\epsilon} \sim p(\boldsymbol{\epsilon})} \Bigg[ \sum_{i=1}^N \nabla_\psi \log \mathbb{E}_{\mathbf{x}'_i \sim p(\cdot|\mathbf{x}_i,f(\psi,\boldsymbol{\epsilon}))} \left[ p(y_i|\mathbf{x}'_i, f(\psi, \boldsymbol{\epsilon})) \right] \Bigg]  - 
         \nabla_\psi \text{KL}(q_\psi(\theta) \| p(\theta)).
\end{equation*}
The log term is problematic because its pathwise gradient is a ratio of expectations. If we assume the adversarial channel is also reparameterizable, $\mathbf{x}'=h(\mathbf{x},\theta,\boldsymbol{\eta}),\ \boldsymbol{\eta}\sim p(\boldsymbol{\eta})$, then
\begin{equation*}
    \nabla_\psi \log \mathbb{E}_{\mathbf{x}'_i \sim p(\cdot|\mathbf{x}_i,f(\psi,\boldsymbol{\epsilon}))} p(y_i|\mathbf{x}'_i, f(\psi, \boldsymbol{\epsilon}))  = \frac{\mathbb{E}_{\boldsymbol{\eta} \sim p(\boldsymbol{\eta})} \left[ \nabla_\psi p(y_i | h(\mathbf{x}_i, f(\psi,\boldsymbol{\epsilon}), \boldsymbol{\eta}),  f(\psi, \boldsymbol{\epsilon})) \right]}{\mathbb{E}_{\boldsymbol{\eta} \sim p(\boldsymbol{\eta})} \left[ p(y_i | h(\mathbf{x}_i, f(\psi,\boldsymbol{\epsilon}), \boldsymbol{\eta}),  f(\psi, \boldsymbol{\epsilon})) \right]}.
\end{equation*}
Here $\nabla_\psi$ is the total pathwise derivative, which also propagates through the channel output $\mathbf{x}'_i=h(\mathbf{x}_i,f(\psi,\boldsymbol{\epsilon}),\boldsymbol{\eta})$, under appropriate regularity conditions, see end of Section~\ref{subsec:channels}. Treating $\mathbf{x}'_i$ as fixed (a stop-gradient) would drop this term and, in general, yield an inconsistent estimator.
We propose a biased but consistent estimator built directly from samples: for each gradient step we (i) draw $S$ samples $\boldsymbol{\epsilon}_s\sim p(\boldsymbol{\epsilon})$, (ii) compute the corresponding parameters $\theta_s=f(\psi,\boldsymbol{\epsilon}_s)$, (iii) for each data point $i$ draw $K$ adversarial noises $\boldsymbol{\eta}_{isk}\sim p(\boldsymbol{\eta})$ and construct $\mathbf{x}'_{isk}=h(\mathbf{x}_i,\theta_s,\boldsymbol{\eta}_{isk})$, (iv) evaluate $p_{isk}=p(y_i|\mathbf{x}'_{isk},\theta_s)$ and the total pathwise gradients $\nabla_\psi\log p_{isk}$, differentiating through both $\theta_s$ and $\mathbf{x}'_{isk}$, and (v) plug these into the ratio. This yields
\begin{equation*}
    \widehat{g}_{i,s}(\psi)=\sum_{k=1}^K\tilde{w}_{isk}\,\nabla_\psi\log p_{isk},\qquad 
    \tilde{w}_{isk}=\frac{p_{isk}}{\sum_{k'=1}^K p_{isk'}}.
\end{equation*}
Since $\nabla_\psi\log\big(K^{-1}\sum_{k=1}^K p_{isk}\big)=\sum_{k=1}^K\tilde{w}_{isk}\nabla_\psi\log p_{isk}$, then $\widehat{g}_{i,s}$ is exactly the single-sample reparameterization gradient estimate of the $i$-th term of the importance-weighted bound $\mathcal{L}_{\text{IWAE},K}$ of Appendix~\ref{app:tightness}.
Therefore, the full approximate gradient estimator of the ELBO over a mini-batch $\mathcal{B}$ is
\begin{equation*}
    \nabla_\psi\mathcal{L}(\psi)\;\approx\;\frac{N}{|\mathcal{B}|S}\sum_{s=1}^S\sum_{i\in\mathcal{B}}\widehat{g}_{i,s}(\psi)-\frac{1}{S}\sum_{s=1}^S\nabla_\psi\big[\log q_\psi(\theta_s)-\log p(\theta_s)\big].
\end{equation*}
Under these conditions, the estimator is unbiased for $\nabla_\psi\mathcal{L}_{\text{IWAE},K}(\psi)$. As an estimator of $\nabla_\psi\mathcal{L}(\psi)$, its bias depends on $K$ alone, whereas $S$ only affects its variance; it is thus consistent as $K,S\to\infty$ (with the full data set).

%%%%%%%%%%%%%%%%%%%%%%%%%%%%%%%%%

\section{Adversarial Channels}\label{app:channel}

The adversarial channel taxonomy, together with the baseline noise, one-step, mixture, and learned channels, is presented in Section~\ref{subsec:channels} of the main text. For completeness, this section collects the implementation details omitted there, namely the iterative attack channel and the architectures of the learned generative channels. As in the main text, $s_\theta^{(a)}(\mathbf{z};y)$ denotes the objective that attack type $a$ seeks to maximize (e.g., for the standard likelihood-based attack, the negative log-likelihood $-\log p(y\mid\mathbf{z},\theta)$, substituting the model prediction $\hat{y}$ for $y$ when the label is unavailable), and each channel moves $\mathbf{x}$ in a direction that increases $s_\theta^{(a)}$ before adding Gaussian augmentation.

\subsection{Iterative Attack Channel}

The PGD channel is an iterative attack. Let
$\mathcal{B}_\epsilon(\mathbf{x})
=
\{\mathbf{z}:\|\mathbf{z}-\mathbf{x}\|_2\le \epsilon\}
$
be the allowed perturbation set. Starting from $\mathbf{z}^0\in\mathcal{B}_\epsilon(\mathbf{x})$, often sampled randomly, PGD performs $T$ projected ascent steps,
\[
\mathbf{z}^{t+1}
\leftarrow
\Pi_{\mathcal{B}_\epsilon(\mathbf{x})}
\left(
\mathbf{z}^t
+
\alpha
\frac{\nabla_{\mathbf{z}}s_\theta^{(a)}(\mathbf{z}^t;y)}
     {\left\|\nabla_{\mathbf{z}}
       s_\theta^{(a)}(\mathbf{z}^t;y)\right\|_2}
\right),
\qquad t=0,\ldots,T-1,
\]
and then samples
\[
\mathbf{x}'=\mathbf{z}^T+\boldsymbol{\eta},
\qquad
\boldsymbol{\eta}\sim\mathcal{N}(0,\sigma^2 I).
\]
Here $\alpha$ is the step size and $\Pi_{\mathcal{B}_\epsilon(\mathbf{x})}$ denotes projection onto the $L_2$ perturbation ball around the original input. PGD is more computationally expensive than the one-step channel but produces stronger adversarial examples for both training and evaluation. As with the one-step channel in the main text, the PGD channel can serve as a component of a mixture with the identity channel.

% [The general mixture channel and the concrete MIX specification are now given in Section~3.2 and the Experimental Setup of the main text.]

\subsection{Learned Adversarial Perturbation Models}

As one instantiation of the adversarial channel detailed in Section~\ref{subsec:channels} of the main text, we develop learned adversarial models that generate perturbations conditioned on both input and label. These generative models parameterize perturbation distributions rather than producing deterministic ones, enabling gradient-based AT against stochastic attacks. 
Training follows a GAN-style procedure \citep{NIPS2014_5423}. At each batch iteration, the adversarial model is first updated to maximize a loss function associated with model performance (cross-entropy for classification and RMSE for regression). Subsequently, the main model is adversarially trained using samples from $p_\omega(\mathbf{x}'\mid\mathbf{x},y)$ generated by the adversarial model. Concretely, each mini-batch triggers two updates: holding the predictive model fixed, we take one gradient ascent step in $\omega$ on that loss, so that the generator concentrates perturbation mass where the current model performs worst; holding $\omega$ fixed, we then take the usual variational step with $p_\omega$ in the role of the adversarial channel. Perturbations are drawn through the reparameterization $\mathbf{x}'=\mathbf{x}+\boldsymbol{\mu}_\omega+\boldsymbol{\sigma}_\omega\odot\boldsymbol{\eta}$, where $\boldsymbol{\eta}\sim\mathcal{N}(0,I)$ and $\odot$ denotes the element-wise product, so that both updates are differentiable in the generator outputs. Note that $\omega$ is point-estimated rather than assigned a prior and integrated over.

\paragraph*{Convolutional Adversarial Model.}
For MNIST and CIFAR-10, we design a convolutional perturbation generator preserving spatial structure. Given input image $\mathbf{x}$ and label $y$ (one-hot encoded), the image is processed through two convolutional layers (16 filters, 3×3 kernels, same padding) with ReLU activations, followed by a third convolutional layer matching the input channel dimension. The spatial features are flattened, combined with a learned label projection via element-wise addition, and reshaped to the original dimensions. This produces a mean perturbation $\boldsymbol{\mu}$ added to the input  $\mathbf{x}$ and per-pixel log-standard deviations $\log\boldsymbol{\sigma}$ from an additional convolutional layer, jointly parameterizing a Gaussian perturbation distribution.

\paragraph*{Fully Connected Adversarial Model.}
For regression datasets, we employ a fully connected perturbation generator that concatenates a learned projection of the scalar label $y$ with input features $\mathbf{x}$. This representation is processed through two hidden layers (128 units each) with ReLU activations, followed by a final layer matching the input dimensionality. Analogously to the convolutional variant, this produces a mean perturbation $\boldsymbol{\mu}$ (added to $\mathbf{x}$) and per-feature log-standard deviations $\log\boldsymbol{\sigma}$ through separate projections, parameterizing a Gaussian perturbation distribution for diverse, label-conditioned adversarial perturbations.

%%%%%%%%%%%%%%%%%%%%%%%%%%%%%%%%%
\section{Additional Information on Experiments}\label{app:exps}

\subsection{Experimental Setup}

This subsection details the experimental setup used throughout our experiments. Full hyperparameter specifications, including learning rates, batch sizes, decay rates, and number of training iterations, are available in the repository \ifanon\url{https://anonymous.4open.science/r/advDef}\else\url{https://github.com/pablogarciarce/advDef}\fi. Besides the datasets of the main text, we use the California Housing dataset \citep{pace1997sparse} for the timing comparison and the additional regression results below.

\subsubsection{Model Architectures}

We transform the deterministic architectures described below into Bayesian neural networks by placing prior distributions over all network weights and biases. Specifically, we employ factorized Gaussian priors with zero mean and unit variance $p(\theta) = \mathcal{N}(0, I)$.

Inference is performed via variational inference, optimizing a mean-field variational posterior distribution $q_\psi(\theta)$ to approximate the actual posterior $p(\theta|\mathcal{D})$ (for the proactive defenses, the generalized posterior \eqref{eq:gen_posterior_eta} at $\eta=1$, see Proposition~\ref{prop:generalized_posterior}), using the Adam optimizer \citep{kingma2014adam} with exponential learning rate decay.  The family is fully factorized Gaussian, $q_\psi(\theta)=\prod_k \mathcal{N}(\theta_k;\mu_k,\sigma_k^2)$, with one factor per weight and bias: each carries a learned mean $\mu_k$ and a learned standard deviation $\sigma_k$, the latter parameterized in unconstrained form and mapped to the positive reals, so that $\psi$ collects two variational parameters per network parameter. We maximize the corresponding variational objective (the ELBO, or the SELBO in \eqref{eq:surrogate_elbo} for the proactive defenses) by stochastic gradient ascent, drawing Monte Carlo samples of $\theta$ per mini-batch through the reparameterization $\theta_k=\mu_k+\sigma_k\,\epsilon_k$ with $\boldsymbol{\epsilon}\sim\mathcal{N}(0,I)$, the reparameterizable form assumed in Section~\ref{sec:protect-train}. Predictions average the response likelihood $p(y\mid\mathbf{x},\theta)$ over samples $\theta_s\sim q_\psi(\theta)$, a Monte Carlo approximation of the PPD. Note that the evaluation attacks are computed against this same posterior-averaged predictive distribution, not against a single posterior sample.

\paragraph*{Classification Model.}
For CIFAR-10 classification, we primarily employ a Bayesian \textbf{ResNet-18} \citep{he2016deep} backbone across our proactive channels, baselines, and reactive purification models. To evaluate the impact of increased architectural capacity, we additionally employ a Bayesian \textbf{WideResNet-34-10} (\texttt{WRN34-10}) \citep{zagoruyko2016wide} for \texttt{BL WRN}, \texttt{MIX WRN}, and its point-estimate counterpart \texttt{MIX WRN MAP}. For MNIST (and the \texttt{MIX ConvNet} architectural ablation on CIFAR-10), we employ a shallower convolutional neural network comprising two convolutional blocks followed by a fully connected output layer. Each block consists of a convolutional layer (16 and 32 filters, respectively, with $3 \times 3$ kernels), ReLU activation, and $2 \times 2$ average pooling, after which the flattened feature maps are passed through a dense layer producing logits for the 10 classes.

\paragraph*{Regression Model.}
For regression tasks, we employ a fully connected feedforward architecture with two hidden layers of 16 neurons each with ReLU activations. The input is processed through these hidden layers before a final dense layer outputs a scalar prediction.

%\FloatBarrier
\subsection{Computational Overhead}

As the adversarial robustification methodologies present distinct trade-offs between training and inference efficiency,
let us evaluate the computational costs associated with each robustification approach.

The proactive defense modifies the training objective, resulting in moderately increased training time with no test-time overhead beyond standard Bayesian inference. Conversely, the reactive defense preserves standard training costs but introduces substantial test-time computation and memory overhead. The inference overhead depends mainly on the number of posterior samples ($S$) and the size of the training set ($N$), both of which enter the weighted estimate. The estimate uses the full training set. To reduce test-time cost, we precompute the channel outputs $\mathbf{x}_i \mapsto \mathbf{x}'_i$ for every training input and posterior sample once, and cache them, so that attacks are not recomputed for each test input. With our default configuration ($S=5$, full training set), the reactive defense incurs substantial overhead for high-dimensional inputs (over 50× for CIFAR-10 images) but remains competitive for low-dimensional regression tasks (approximately 5× overhead). Additionally, the reactive defense requires maintaining the training set and the cached channel outputs in memory simultaneously, leading to considerable memory requirements that scale with input dimensionality.

Table \ref{tab:computational_times} summarizes training and inference times across datasets and methodologies. 
Robust Training is timed for \texttt{MIX OS}\footnote{\texttt{MIX OS} adversarial channel places probability mass $(0.4, 0.2, 0.2, 0.2)$ on the identity channel and on one-step channels built, respectively, from the cross-entropy, CW, and predictive-entropy objectives for classification, and mass $(0.4, 0.6)$ on the identity and on a one-step squared-error channel for regression, fixing $\epsilon=1$ in all one-step components.} and Robust Inference for \texttt{onPure}. Proactive defenses with one-step channels and the same mixture composition have similar costs, as the score used by each component barely affects the time of a training step. Multi-step components are more expensive: each additional attack step adds roughly one forward and one backward pass of the network for the samples it perturbs, so the PGD5 component of \texttt{MIX} increases training time accordingly. \texttt{MIX AA} is far more expensive still, as its Square component requires many model queries per input, and it brought no robustness gain. On the regression datasets, where the networks are small, standard and robust training times are comparable. Overall, the proactive defense is preferable when low-latency inference is required or when inputs are high-dimensional, whereas the reactive defense is viable for low-dimensional problems, where its inference overhead is modest and training is standard. 

\begin{table}[h]
\centering
\caption{Training and inference times (in ms) across datasets, measured on a single NVIDIA A100. Training times are reported per batch iteration, and inference times per sample, both averaged over 100 iterations. Standard Training refers to the undefended BNN baseline (\texttt{BL}), trained by ordinary stochastic variational inference on clean data, and Robust Training to the proactive defense, instantiated by \texttt{MIX OS}. Standard Inference refers to the ordinary posterior predictive evaluation, and Robust Inference to the online reactive defense \texttt{onPure}~\eqref{onPure}, at its default configuration of $S=5$ posterior samples and the full training set with cached channel outputs. }
\label{tab:computational_times}
\begin{tabular}{lcccc}
\hline
\textbf{Dataset} & \textbf{Standard Training} & \textbf{Robust Training} & \textbf{Standard Inference} & \textbf{Robust Inference} \\
\hline
MNIST & 1.85 & 2.35 & 0.59 & 9.22\\
CIFAR-10 & 28.48 & 71.56 & 8.22 & 452.19 \\
Wine & 0.75 & 0.70 & 0.63 & 3.39  \\
Energy & 0.74 & 0.66 & 0.65  & 3.47 \\
Housing & 0.74 & 0.68 & 0.62 & 3.44 \\
\hline
\end{tabular}
\end{table}

\subsection{Additional Results}

%\FloatBarrier

\paragraph*{CIFAR-10 under PGD1 and ENT.}
Section~\ref{subsec:classification} compares all defenses under the two attacks that degrade accuracy the most, multi-step PGD and AutoAttack. Here we complete that comparison with the remaining two attacks of our evaluation suite: the one-step PGD1 attack and the entropy attack ENT, which targets predictive uncertainty rather than accuracy.

Figures~\ref{fig:SEP_onestep} and~\ref{fig:SEP_entropy_pgd} present CIFAR-10 performance metrics across varying perturbation strengths ($\epsilon$) for the one-step PGD1 and entropy maximization (ENT) attacks, respectively. Overall, both largely follow the trends in the main text (Figures~\ref{fig:cifar10_pgd} and~\ref{fig:cifar10_autoattack}), with two notable distinctions. First, under the weaker PGD1 attack (Figure~\ref{fig:SEP_onestep}), the point-estimate MAP variant (\texttt{MIX WRN MAP}) achieves the strongest overall performance, outperforming the variational posterior (\texttt{MIX WRN}) in accuracy and NLL, although \texttt{MIX WRN} remains better calibrated in ECE at most budgets, while deterministic \texttt{AT} still collapses to zero accuracy despite \texttt{BL} retaining moderate resilience, consistent with catastrophic overfitting. Second, the ENT attack (Figure~\ref{fig:SEP_entropy_pgd}) clearly separates both reactive variants in calibration: \texttt{onPure} maintains lower NLL and ECE than \texttt{offPure} across perturbation budgets, where \texttt{offPure} degrades to the worst calibration among all models, confirming the benefit of test-time adaptation.

\begin{figure*}[!t]
    \centering
    \begin{subfigure}[b]{0.32\linewidth}
        \centering
        \includegraphics[width=\linewidth]{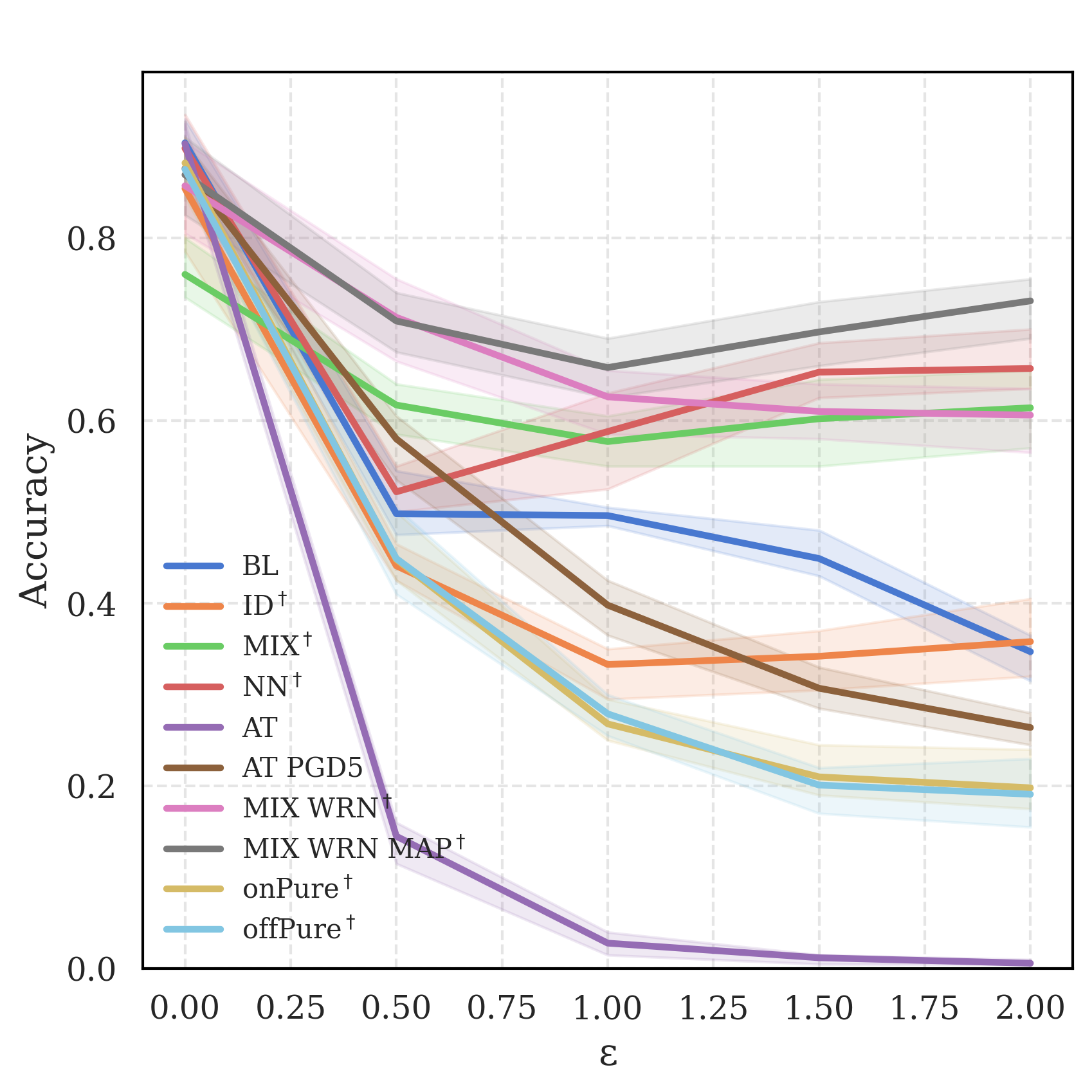}
        \caption{Accuracy}
    \end{subfigure}
    \hfill
    \begin{subfigure}[b]{0.32\linewidth}
        \centering
        \includegraphics[width=\linewidth]{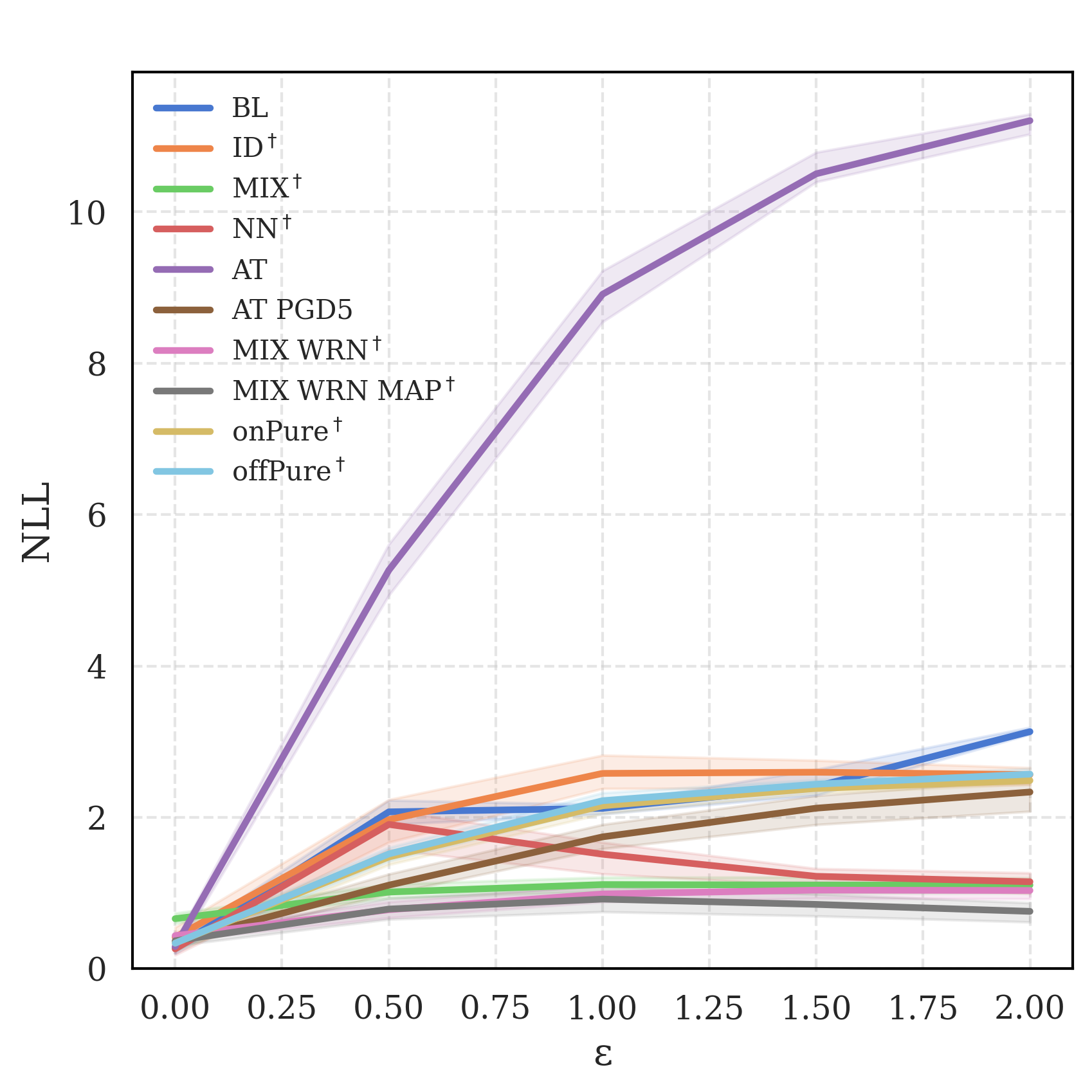}
        \caption{NLL}
    \end{subfigure}
    \hfill
    \begin{subfigure}[b]{0.32\linewidth}
        \centering
        \includegraphics[width=\linewidth]{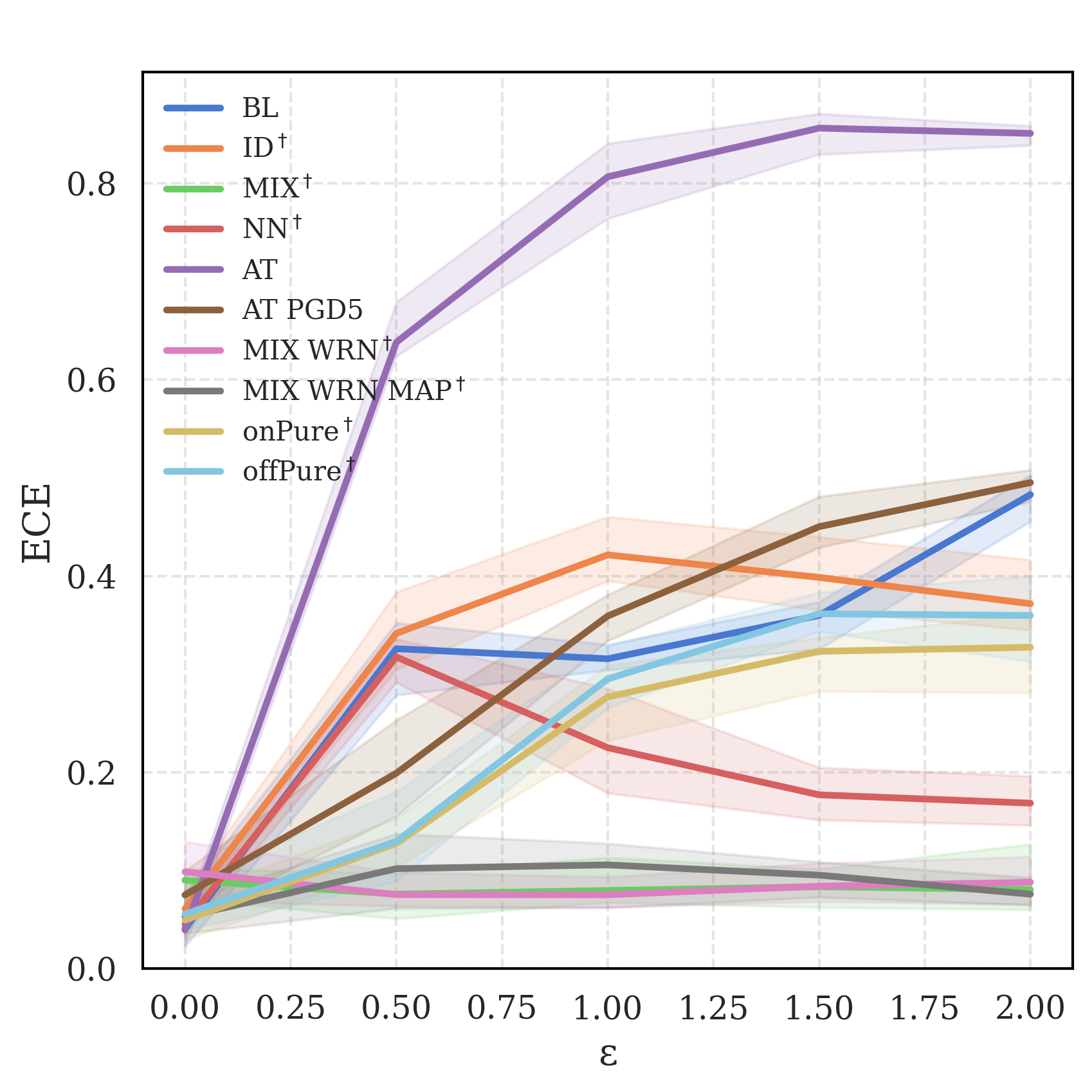}
        \caption{ECE}
    \end{subfigure}
    \caption{Performance on CIFAR-10 against PGD1 attack across varying perturbation strengths ($\epsilon$). Shaded bands indicate $\pm 1$ SD over five random test splits. Our defenses are marked with $^\dagger$.}
    \label{fig:SEP_onestep}
\end{figure*}

\begin{figure*}[!t]
    \centering
    \begin{subfigure}[b]{0.32\linewidth}
        \centering
        \includegraphics[width=\linewidth]{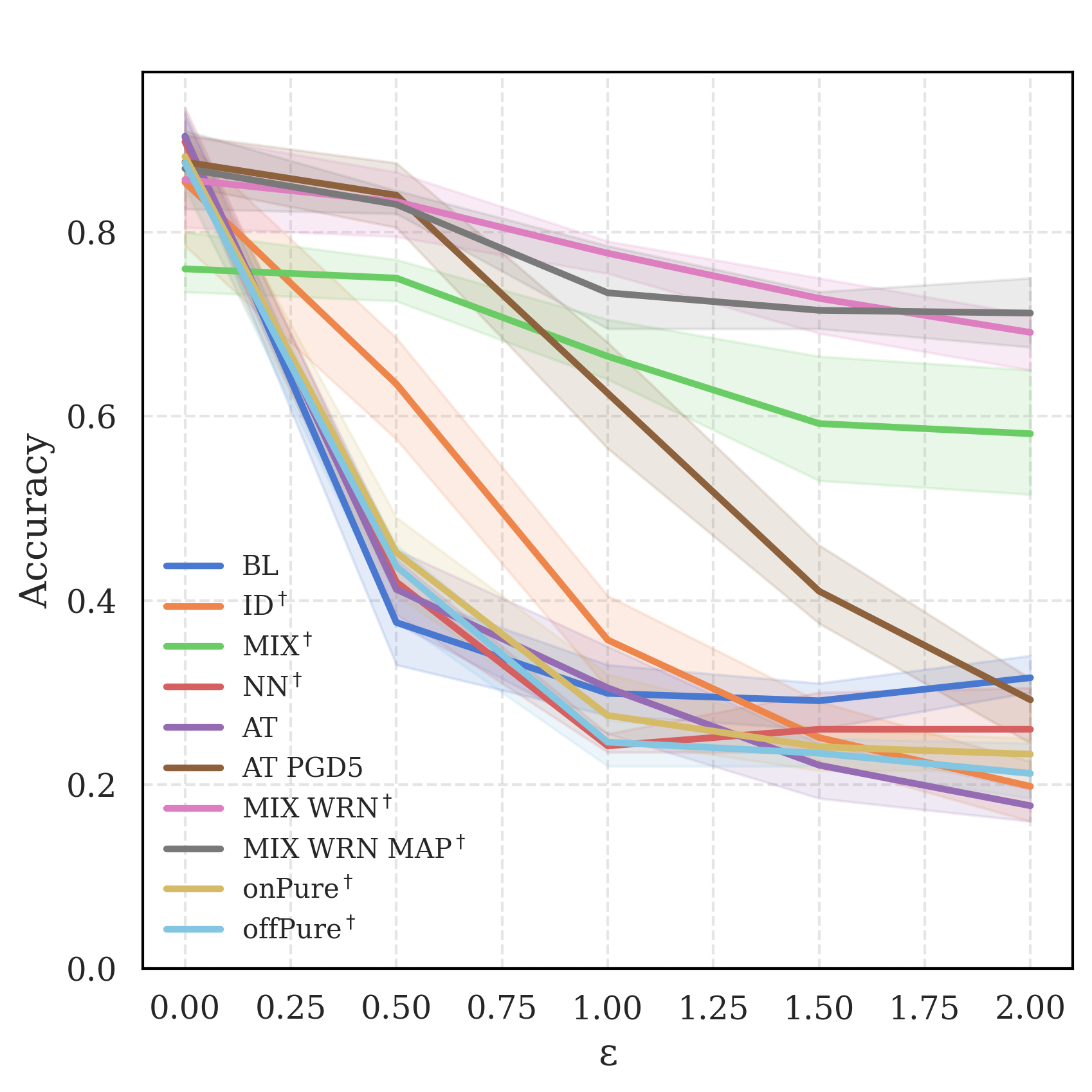}
        \caption{Accuracy}
    \end{subfigure}
    \hfill
    \begin{subfigure}[b]{0.32\linewidth}
        \centering
        \includegraphics[width=\linewidth]{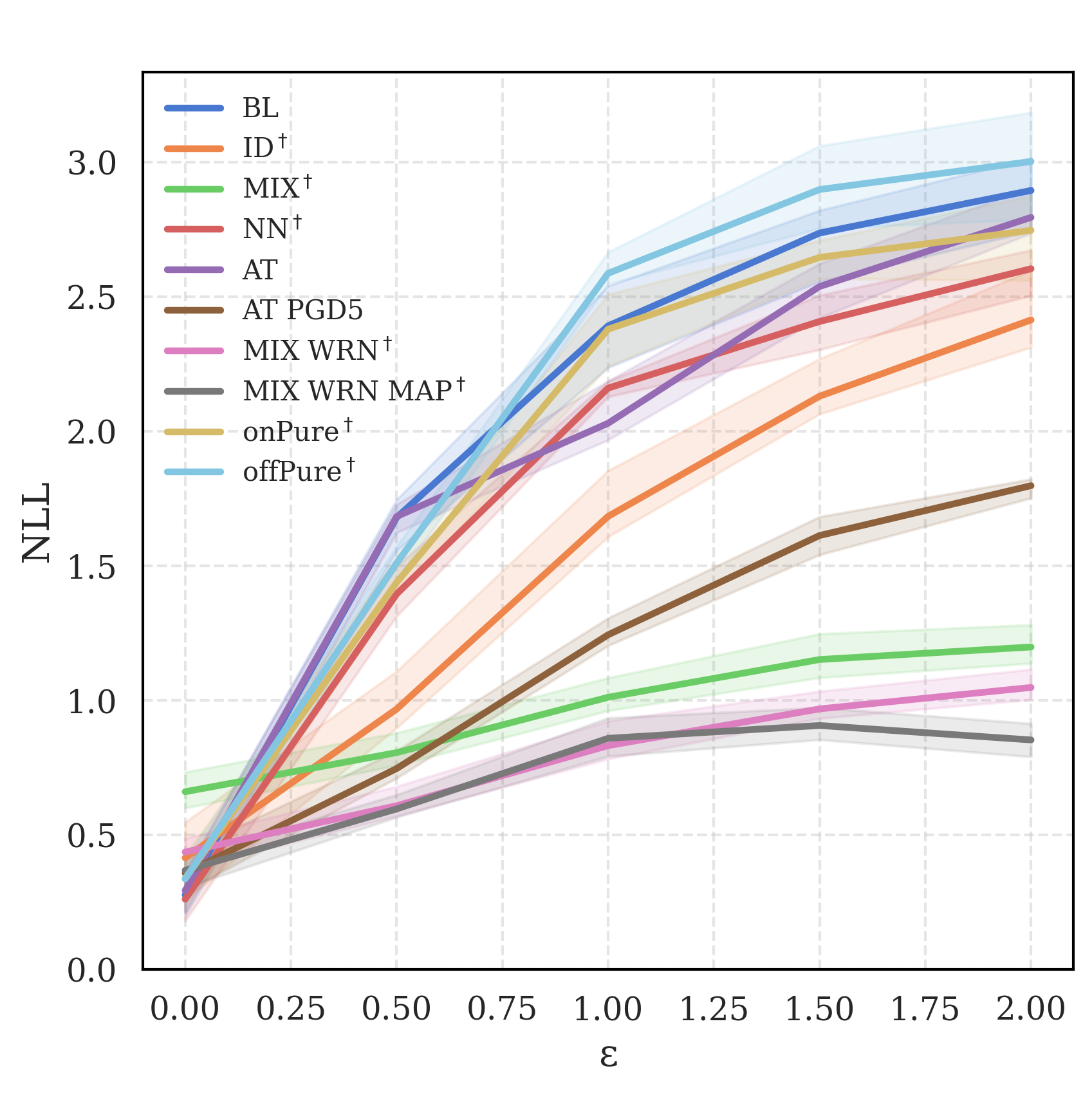}
        \caption{NLL}
    \end{subfigure}
    \hfill
    \begin{subfigure}[b]{0.32\linewidth}
        \centering
        \includegraphics[width=\linewidth]{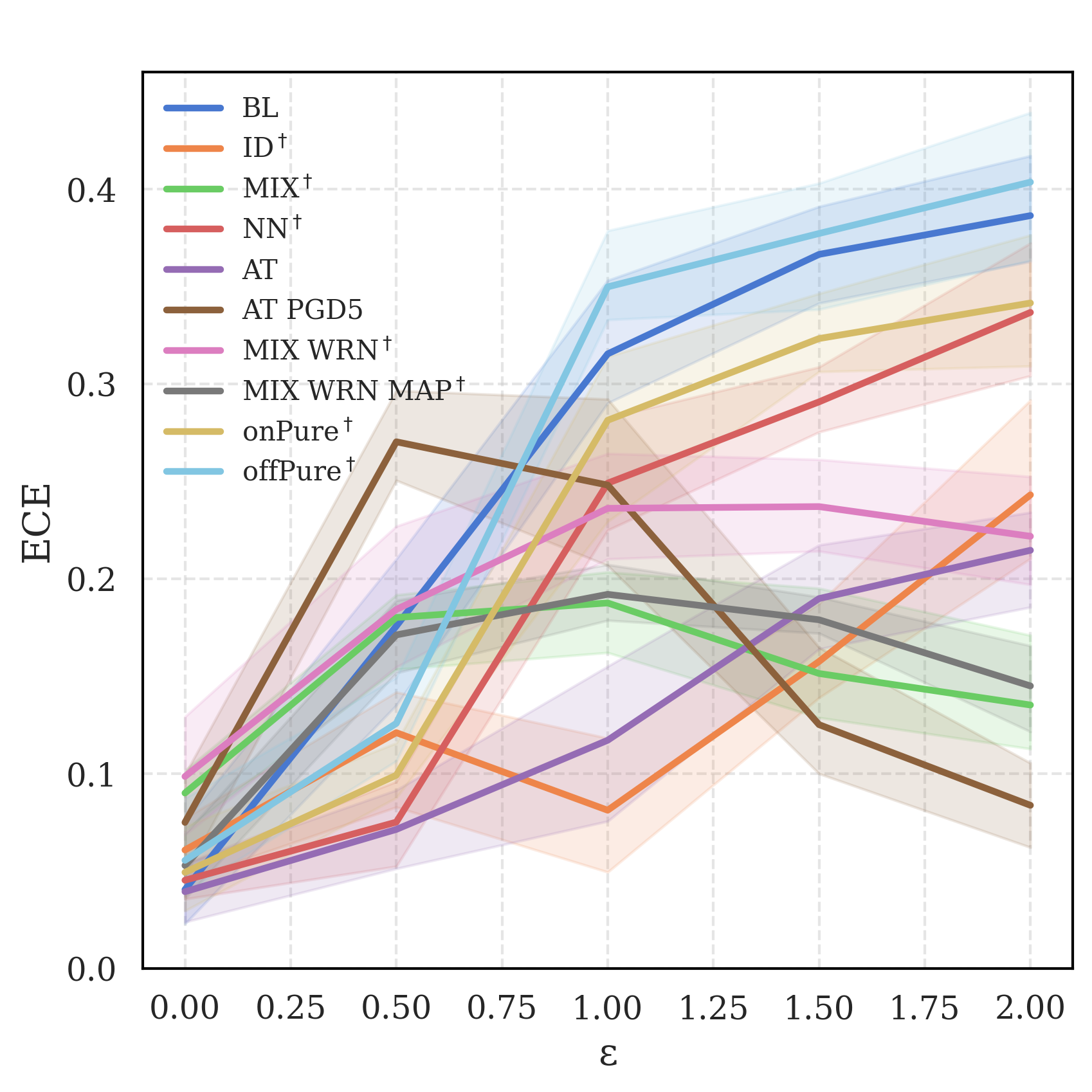}
        \caption{ECE}
    \end{subfigure}
    \caption{Performance on CIFAR-10 against the Entropy Maximization (ENT) attack across varying perturbation strengths ($\epsilon$). Shaded bands indicate $\pm 1$ SD over five random test splits. Our defenses are marked with $^\dagger$.}
    \label{fig:SEP_entropy_pgd}
\end{figure*}

\paragraph*{Qualitative examples on MNIST.}
Figures~\ref{fig:extra_bl} and~\ref{fig:extra_mix} show illustrative examples of adversarial attacks and the corresponding predictive distributions on MNIST. The clean digits are common to both figures; the adversarial images are not, as every evaluation attack is white-box and crafted separately for each defense, against that defense's own PPD. Each model is therefore confronted only with perturbations optimized against itself, a more demanding comparison than transferring one set of adversarial images across models. The \texttt{BL} model is easily misled by most perturbations, often producing incorrect predictions or displaying high uncertainty even when the correct class remains among the top outputs. In contrast, the \texttt{MIX} model shows clear resilience to these attacks, maintaining reliable predictions. Dispersion alone does not indicate degraded calibration, since a diffuse predictive distribution is an appropriate response for a model that is frequently wrong once attacked. What these examples expose is that \texttt{BL} either concentrates on an incorrect class, which is a genuine calibration failure, or spreads its mass without recovering the true label. The \texttt{MIX} model instead preserves a sharper and more consistent predictive distribution, concentrated on the true label across most attack types.

\begin{figure}[ht]
\begin{multicols}{2}
    \includegraphics[width=\linewidth]{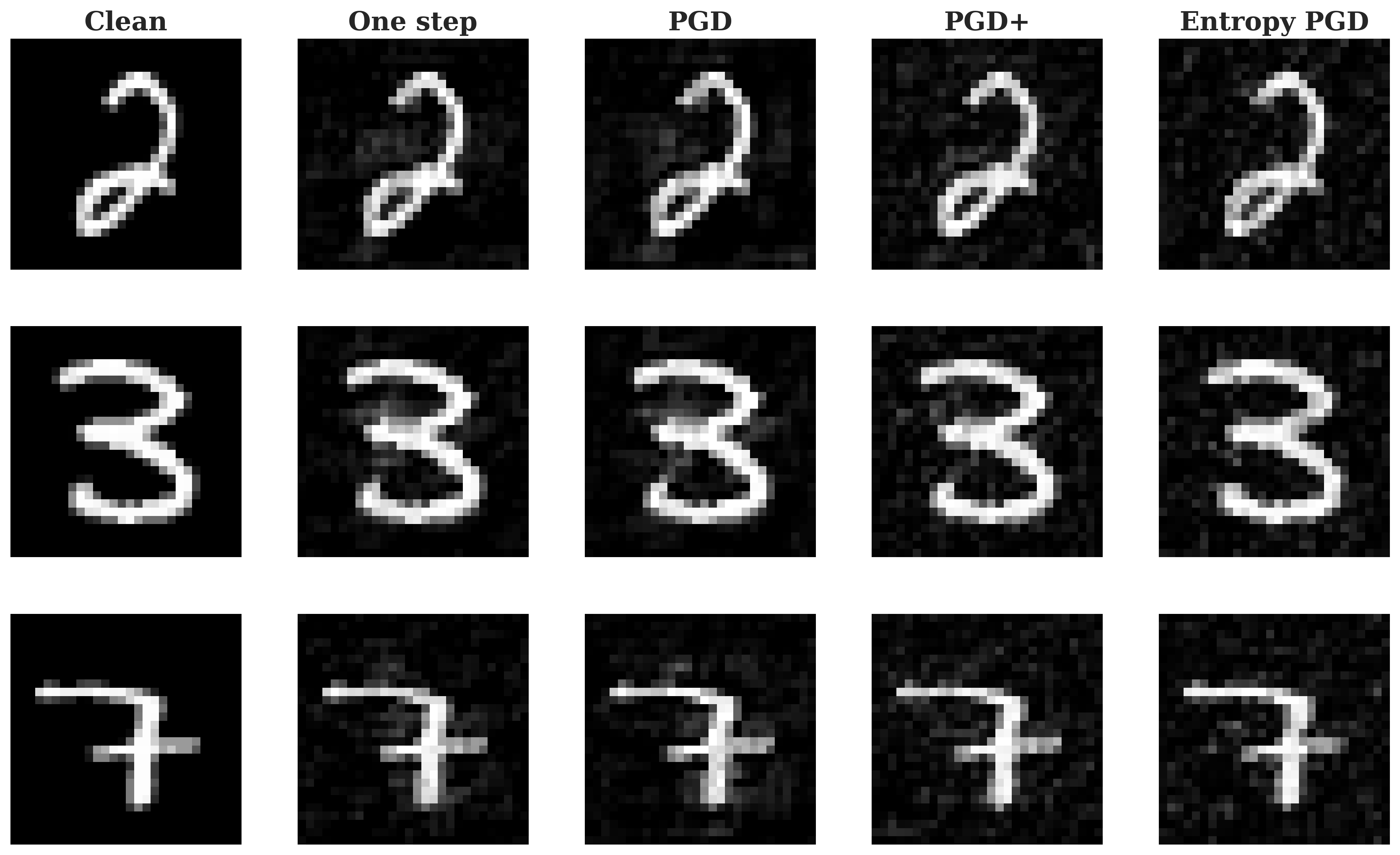}
    \subcaption{Adversarial Examples: Input images generated by different attacks.}\label{subfig:bl_ex}\par 
    \includegraphics[width=\linewidth]{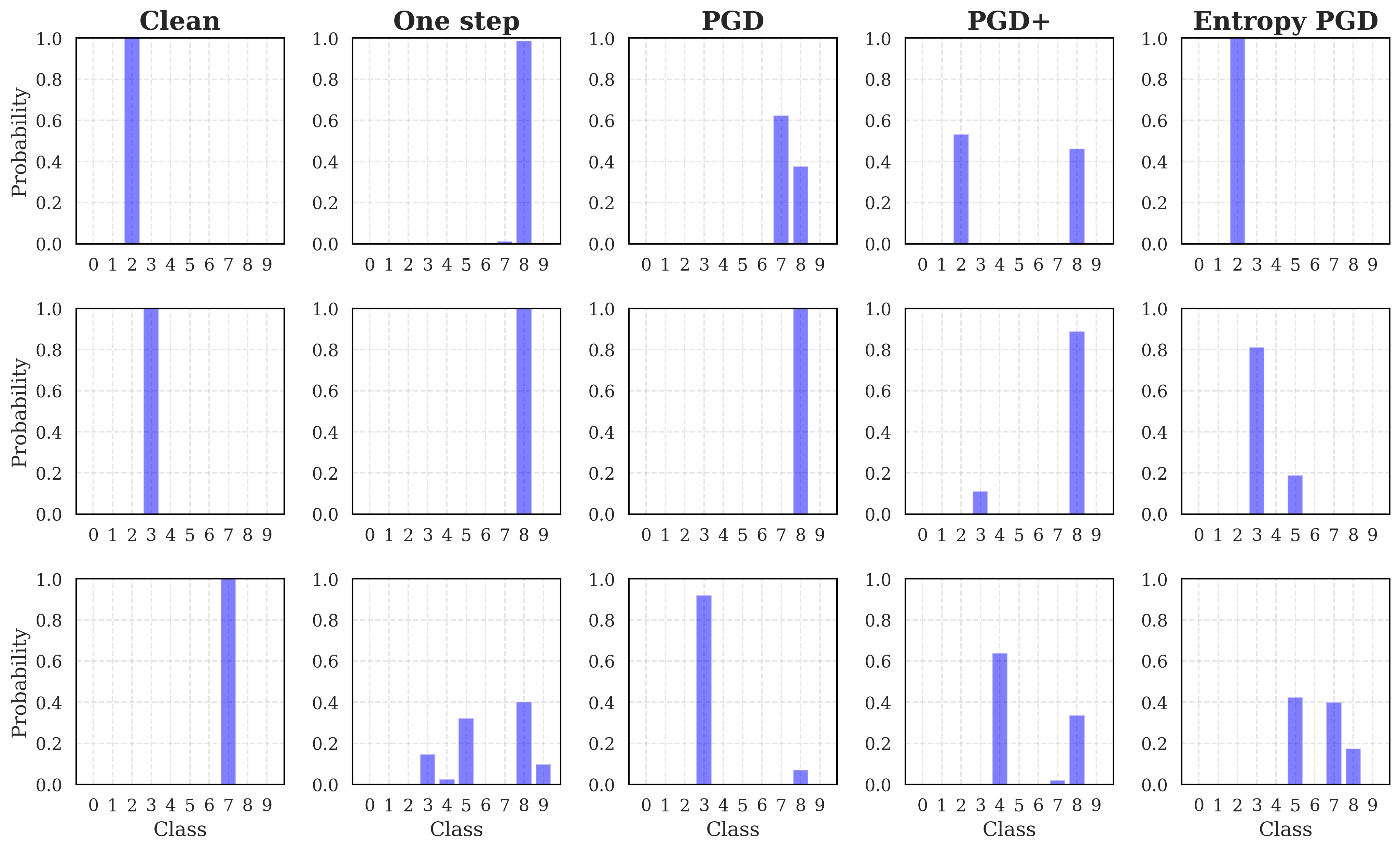}
    \subcaption{Predictive Distributions. Note the high confidence in incorrect classes or dispersed probability mass under attack.}\label{subfig:bl_distr}\par 
    \end{multicols}
\caption{Qualitative results for the undefended \texttt{BL} model on MNIST.}
\label{fig:extra_bl}
\end{figure}

\begin{figure}[ht]
\begin{multicols}{2}
    \includegraphics[width=\linewidth]{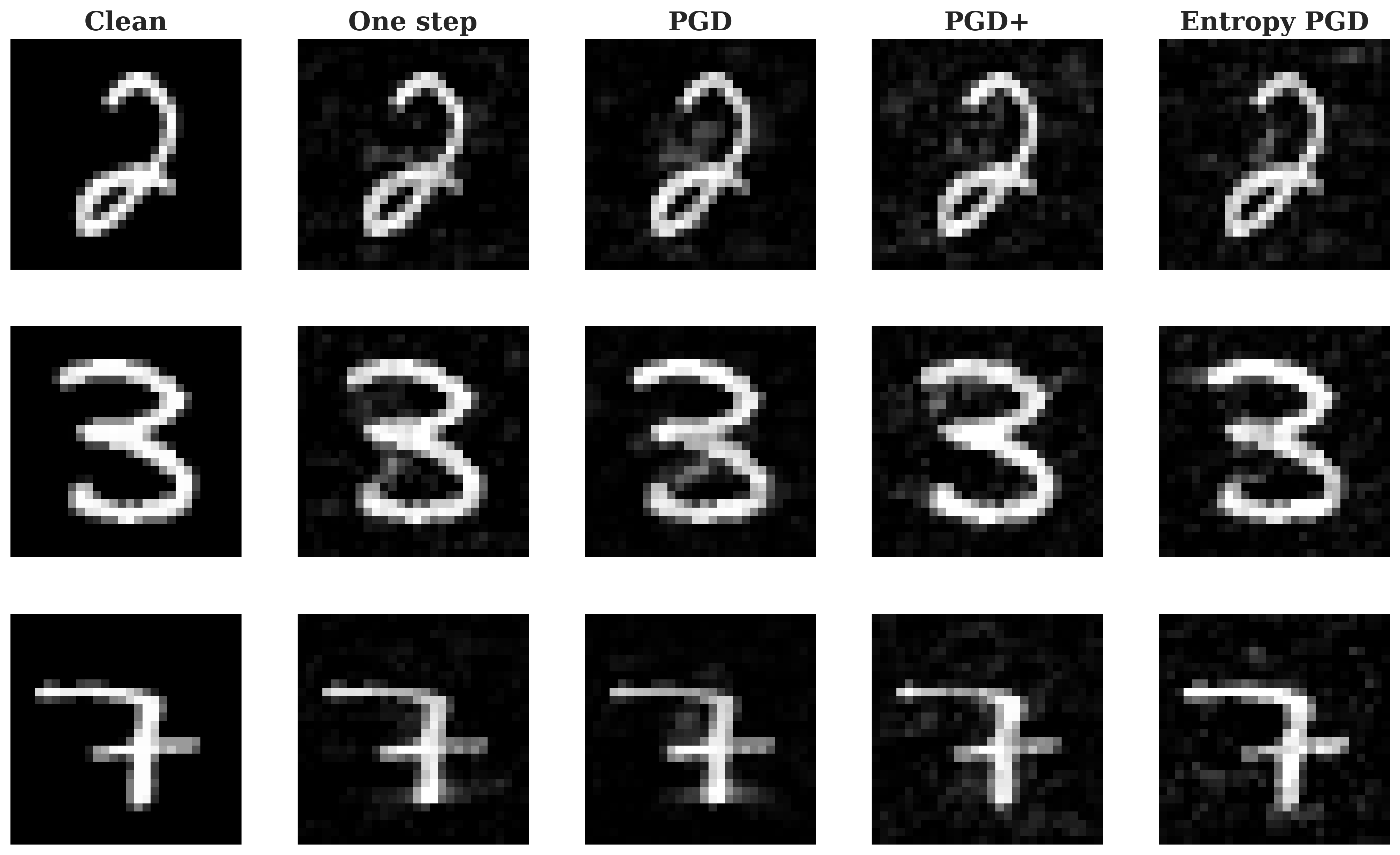}
    \subcaption{Adversarial examples.}\label{subfig:mix_ex}\par 
    \includegraphics[width=\linewidth]{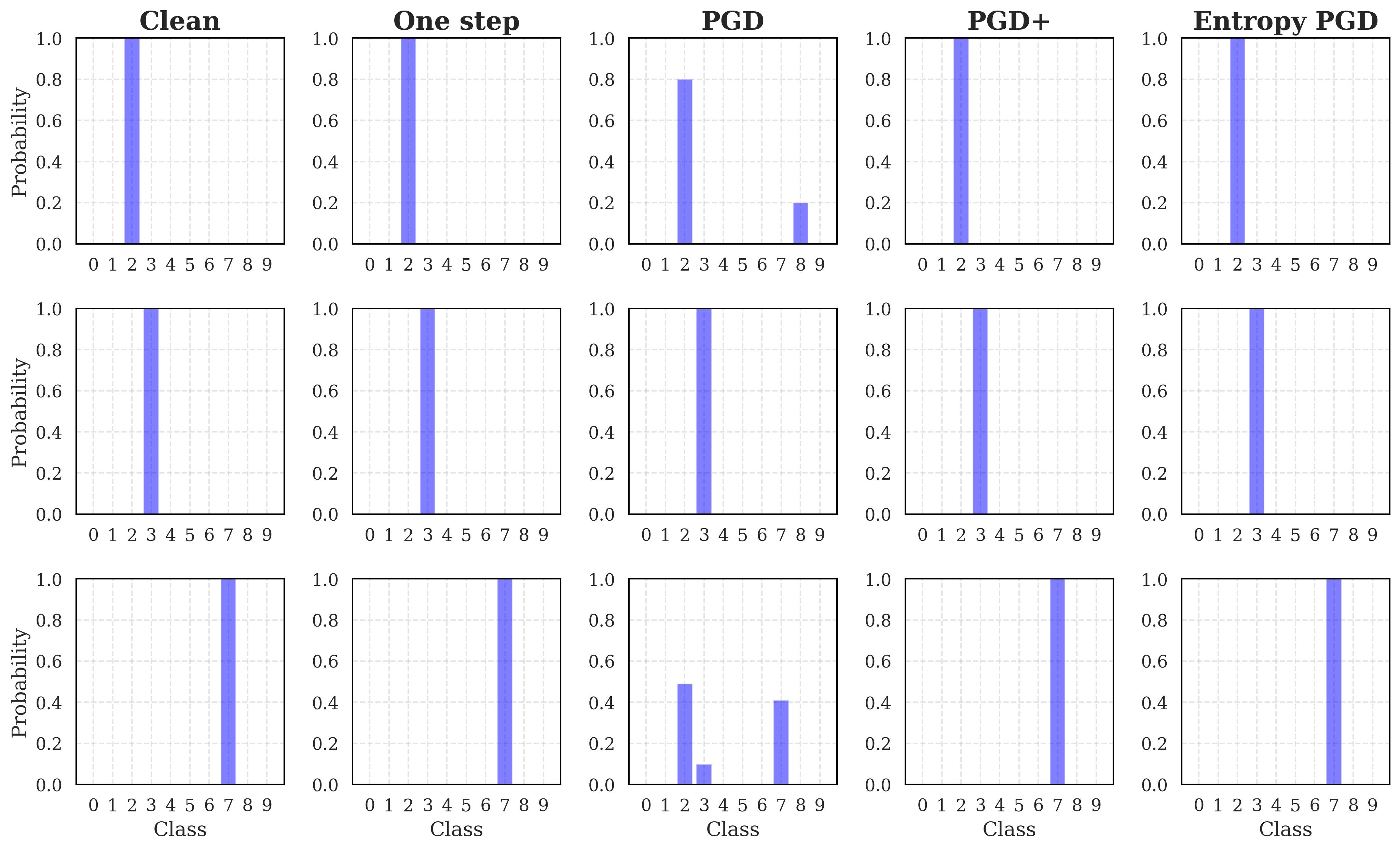}
    \subcaption{Predictive distributions. The model maintains sharp, correct predictive distributions across most attack types.}\label{subfig:mix_distr}\par 
    \end{multicols}
\caption{Qualitative results for the \texttt{MIX} (Proactive) model on MNIST.}
\label{fig:extra_mix}
\end{figure}

\paragraph*{Regression on California Housing.}
Table~\ref{tab:extra_regression} presents additional regression results complementing those in the main text. Overall, the same qualitative patterns hold. The \texttt{BL} model performs well on clean data but degrades sharply under attack, while \texttt{AT} exhibits unstable behavior, performing worse than \texttt{BL} in every adversarial setting. \texttt{NN50} maintains the most consistent performance across attacks, achieving the lowest NLLs and stable RMSE values. \texttt{MIX} improves on \texttt{BL} in NLL under every attack, and in RMSE under the strongest one (PGD$^+$), but trails \texttt{NN50} on this dataset. As before, purification-based methods display clear trade-offs: \texttt{onPure} achieves reasonable RMSEs yet suffers from moderate calibration issues, whereas \texttt{offPure} remains poorly calibrated despite stable errors. 
% These results confirm that the observed trends extend naturally to regression tasks.

\begin{table*}[!h]
\centering
\caption{RMSE and NLL on the California Housing dataset at $\epsilon=2$ under different attacks. SD in parentheses, best in bold. Our defenses are marked with $^\dagger$.}
\label{tab:extra_regression}
\setlength{\tabcolsep}{3pt}
\resizebox{\textwidth}{!}{
\begin{tabular}{l|ccccc|ccccc}
\hline
\multicolumn{1}{c}{\textbf{Model}}  & \multicolumn{5}{c}{\textbf{RMSE}} & \multicolumn{5}{c}{\textbf{NLL}} \\
\hline
& \textbf{Clean} & \textbf{PGD1} & \textbf{PGD} & \textbf{PGD$^+$} & \textbf{ENT} & \textbf{Clean} & \textbf{PGD1} & \textbf{PGD} & \textbf{PGD$^+$} & \textbf{ENT} \\
\hline
\texttt{BL}      & \textbf{0.20 (0.01)} & \textbf{0.44 (0.02)} & \textbf{0.43 (0.02)} & 0.76 (0.02) & 0.45 (0.01)  & \textbf{-0.15 (0.08)} & 2.03 (0.10) & 1.96 (0.12) & 4.96 (0.29) & 1.84 (0.21)\\
\texttt{MIX}$^\dagger$     & 0.23 (0.01) & 0.53 (0.01) & 0.52 (0.02) & 0.63 (0.01) & 0.52 (0.01) & 0.00 (0.03) & 1.71 (0.19) & 1.73 (0.07) & 2.36 (0.16) & 1.56 (0.14)  \\
\texttt{NN50}$^\dagger$    & 0.33 (0.01) & \textbf{0.42 (0.01)} & \textbf{0.41 (0.01)} & \textbf{0.46 (0.02)} & \textbf{0.42 (0.02)} & 0.33 (0.03) & \textbf{0.58 (0.05)} & \textbf{0.56 (0.04)} & \textbf{0.70 (0.04)} & \textbf{0.62 (0.05)}  \\
\texttt{onPure}$^\dagger$  & 0.44 (0.02) & \textbf{0.42 (0.01)} & 0.44 (0.01) & \textbf{0.45 (0.02)} & \textbf{0.43 (0.02)} & 2.11 (0.31) & 2.23 (0.35) & 2.34 (0.43) & 2.75 (0.59) & 2.41 (0.53) \\
\texttt{offPure}$^\dagger$ & 0.43 (0.02) & \textbf{0.43 (0.01)} & 0.45 (0.02) & \textbf{0.45 (0.03)} & \textbf{0.45 (0.03)} & 9.73 (1.06) & 10.24 (1.11) & 10.20 (0.97) & 11.90 (1.16) & 10.39 (1.31)  \\
\texttt{AT}      & 0.31 (0.01) & 1.58 (0.03) & 1.69 (0.01) & 3.82 (0.05) & 1.97 (0.04)  & 0.44 (0.01) & 3.49 (0.30) & 3.67 (0.27) & 14.01 (0.62) & 2.91 (0.09) \\
\hline
\end{tabular}
}
\end{table*}

%\FloatBarrier

%%%%%%%%%%%%%%%%%%%%%%%%%%%%%%%%%%%%%%%%%%%%%%%%%%%%%%%%%%%%%%%%
\section{A Formal Analysis of Bayesian Adversarial Learning}
\label{sec:bal}

Bayesian Adversarial Learning (BAL) \citep{ye2018bayesian} pursues the same goal as our proactive defense of Section~\ref{sec:protect-train}: building adversarial robustness into the posterior at training time. Given the clean training data, it seeks a robust posterior over the model parameters $\theta$ that anticipates adversarial perturbations of the inputs, and then predicts with the PPD under that posterior, as in \eqref{eq:proactive_ppd}. The two approaches differ in how the adversary enters the model. Our proactive defense defines its target explicitly, averaging over an adversarial channel either the likelihood, as in \eqref{eq:robust_theta_posterior}, or the loss, as in the generalized posterior \eqref{eq:gen_posterior_eta}. BAL instead casts robust learning as a game between a learner and an adversary, and describes each player through its own conditional distribution, leaving the target implicit.

Concretely, BAL specifies a conditional distribution over perturbed datasets given the model parameters and clean training data (the adversary's), together with a conditional distribution over model parameters given the perturbed dataset (the learner's). The former favors larger predictive losses while penalizing perturbation costs; the latter combines the parameter prior with an exponential penalty on the same predictive loss. The authors propose sampling alternately from these distributions and interpret the resulting procedure as Gibbs sampling from a joint posterior (whose marginal over $\theta$ would serve as robust posterior). We examine whether both update distributions can be the full conditionals of a common joint distribution, as this interpretation requires. We show that compatibility holds if and only if the loss is additively separable in the parameters and perturbed data. We then verify that the classification and regression losses proposed in BAL do not satisfy this condition.

Fix the clean dataset $\mathcal{D}$ and denote the perturbed dataset by $\tilde{\mathcal{D}}$. In our notation, the learner and adversary distributions (in equations~(2) and~(1), respectively, of \cite{ye2018bayesian}) are
\begin{equation}\label{eq:learner_cond}
q_{\mathrm L}(\theta\mid \tilde{\mathcal{D}})
=\frac{p(\theta)\exp\{-\ell_\theta(\tilde{\mathcal{D}})\}}
       {Z_{\mathrm L}(\tilde{\mathcal{D}})}
\end{equation}
and
\begin{equation}\label{kkarg}
q_{\mathrm A}(\tilde{\mathcal{D}}\mid\theta,\mathcal{D})
=\frac{\exp\{\ell_\theta(\tilde{\mathcal{D}})
             -\alpha c(\tilde{\mathcal{D}},\mathcal{D})\}}
       {Z_{\mathrm A}(\theta,\mathcal{D})}.
\end{equation}
Here $\ell_\theta(\tilde{\mathcal{D}})$ is the predictive loss of the model with parameters $\theta$ on the perturbed dataset (made explicit below for the losses specified in BAL), $p(\theta)$ is the parameter prior (their $p(\theta\mid\beta)$), $c$ is a perturbation cost independent of $\theta$, weighted by $\alpha>0$, and $Z_{\mathrm L}$ and $Z_{\mathrm A}$ are normalizing constants.

\begin{proposition}[Compatibility of BAL conditionals]
\label{prop:bal_compatibility}
Assume that \eqref{eq:learner_cond} and \eqref{kkarg} are proper and strictly positive on a common product support, with finite loss and cost there. They are the full conditionals of a common joint distribution if and only if
\begin{equation}\label{eq:bal_separability}
\ell_\theta(\tilde{\mathcal{D}})
=\varphi(\tilde{\mathcal{D}})+\chi(\theta)
\end{equation}
almost everywhere on that support, for some measurable functions $\varphi$ and $\chi$. In this case, $\theta$ and $\tilde{\mathcal{D}}$ are conditionally independent given $\mathcal{D}$.
\end{proposition}

\begin{proof}
Suppose that a joint density $j(\theta,\tilde{\mathcal{D}}\mid \mathcal{D})$ has the two prescribed full conditionals \eqref{eq:learner_cond} and \eqref{kkarg} (any joint distribution with these conditionals has a density, which is strictly positive on the product support, because both conditionals are strictly positive densities there). All density identities below are understood almost everywhere. For two parameter values $\theta_1,\theta_2$ and two perturbed datasets $\tilde{\mathcal{D}}_1,\tilde{\mathcal{D}}_2$, define the loss contrast $\Delta =\ell_{\theta_1}(\tilde{\mathcal{D}}_1) +\ell_{\theta_2}(\tilde{\mathcal{D}}_2) -\ell_{\theta_1}(\tilde{\mathcal{D}}_2) -\ell_{\theta_2}(\tilde{\mathcal{D}}_1).$ The joint density determines the ratio
\[
R
=\frac{j(\theta_1,\tilde{\mathcal{D}}_1\mid \mathcal{D})
        j(\theta_2,\tilde{\mathcal{D}}_2\mid \mathcal{D})}
       {j(\theta_1,\tilde{\mathcal{D}}_2\mid \mathcal{D})
        j(\theta_2,\tilde{\mathcal{D}}_1\mid \mathcal{D})}.
\]
Computing this ratio from the learner's conditional cancels the marginal densities of $\tilde{\mathcal{D}}$, the prior factors, and
the learner's normalizers, yielding
\[
R
=\frac{q_{\mathrm L}(\theta_1\mid \tilde{\mathcal{D}}_1)
        q_{\mathrm L}(\theta_2\mid \tilde{\mathcal{D}}_2)}
       {q_{\mathrm L}(\theta_1\mid \tilde{\mathcal{D}}_2)
        q_{\mathrm L}(\theta_2\mid \tilde{\mathcal{D}}_1)}
=e^{-\Delta}.
\]
Computing the same ratio from the adversary's conditional instead cancels the marginal densities of $\theta$, the perturbation costs, and the adversary's normalizers, giving
\[
R
=\frac{q_{\mathrm A}(\tilde{\mathcal{D}}_1\mid\theta_1,\mathcal{D})
        q_{\mathrm A}(\tilde{\mathcal{D}}_2\mid\theta_2,\mathcal{D})}
       {q_{\mathrm A}(\tilde{\mathcal{D}}_2\mid\theta_1,\mathcal{D})
        q_{\mathrm A}(\tilde{\mathcal{D}}_1\mid\theta_2,\mathcal{D})}
=e^{\Delta}.
\]
Consequently, compatibility requires $e^{-\Delta}=e^{\Delta}$, hence $\Delta=0$, almost everywhere. Fixing a reference pair $(\theta_\star,\tilde{\mathcal{D}}_\star)$ for which these identities hold almost
everywhere yields
\[
\ell_\theta(\tilde{\mathcal{D}})
=\ell_\theta(\tilde{\mathcal{D}}_\star)
+\ell_{\theta_\star}(\tilde{\mathcal{D}})
-\ell_{\theta_\star}(\tilde{\mathcal{D}}_\star),
\]
which is precisely \eqref{eq:bal_separability}.

Conversely, under \eqref{eq:bal_separability}, the terms depending only on the conditioning variable cancel upon normalization, so that
\[
q_{\mathrm L}(\theta\mid \tilde{\mathcal{D}})
\propto p(\theta)e^{-\chi(\theta)},
\qquad
q_{\mathrm A}(\tilde{\mathcal{D}}\mid\theta,\mathcal{D})
\propto e^{\varphi(\tilde{\mathcal{D}})
           -\alpha c(\tilde{\mathcal{D}},\mathcal{D})}.
\]
Both distributions are proper by assumption and independent of their respective conditioning variables. Their product is therefore a compatible joint distribution, establishing sufficiency and the stated conditional independence. It is also the only one: any compatible joint is the product of $q_{\mathrm L}(\theta\mid\tilde{\mathcal{D}})$, here free of $\tilde{\mathcal{D}}$, and the marginal density of $\tilde{\mathcal{D}}$, which must then coincide with $q_{\mathrm A}$.
\end{proof}

Proposition~\ref{prop:bal_compatibility} instantiates the classical compatibility criterion for conditionally specified distributions \citep{arnold1989compatible}. Its necessity part, together with the conditional independence, parallels Lemma~1 of \citet{he2019probgan}, who stated it in general and used it to show that the conditionals underlying the Bayesian GAN \citep{saatci2017bayesian} are incompatible. We verify below that the losses specified in BAL indeed violate the condition.

\paragraph*{Application to the losses specified in BAL.} Section~3 of \cite{ye2018bayesian} fixes the labels and perturbs the inputs, so that $\tilde{\mathcal{D}}=\{(\widetilde{\mathbf x}_i,y_i)\}_{i=1}^N$ with $\widetilde{\mathbf x}_i$ the perturbed version of the clean input $\mathbf x_i$. It builds the dataset loss from the loss $\ell_\theta(\mathbf u,y)$ of a single observation with response $y$, viewed as a function of a generic input $\mathbf u\in\mathcal X$ at which the predictive model $f_\theta$ is evaluated, taking
\[
\ell_\theta(\tilde{\mathcal{D}})
=\sum_{i=1}^N\ell_\theta(\widetilde{\mathbf x}_i,y_i),
\qquad
\ell_\theta(\mathbf u,y)=-\log p(y\mid f_\theta(\mathbf u)).
\]
It identifies cross-entropy for classification and squared loss for regression, and takes the quadratic cost \mbox{$c(\tilde{\mathcal{D}},\mathcal{D})=\sum_{i=1}^N\|\widetilde{\mathbf x}_i-\mathbf x_i\|_2^2$}. To disprove the separability between parameters and perturbed data required by Proposition~\ref{prop:bal_compatibility} for the dataset loss, it suffices to find a nonzero contrast while varying one perturbed input and holding all other observations fixed: if the remaining perturbed inputs are common to $\tilde{\mathcal{D}}_1$ and $\tilde{\mathcal{D}}_2$, their losses cancel in $\Delta$, which then only involves $\ell_{\theta_1}$ and $\ell_{\theta_2}$ at the two values of the varied input.

For classification with $K\ge2$ classes, the softmax cross-entropy is
\[
\ell_\theta(\mathbf u,y)
=\log\!\left(\sum_{k=1}^K e^{g_{\theta,k}(\mathbf u)}\right)
-g_{\theta,y}(\mathbf u),
\]
where $g_{\theta,k}$ is the logit for class $k$, that is, the $k$-th component of $f_\theta$. The neural classifiers considered in BAL end with a trainable fully connected output layer, so it suffices to exhibit two parameter values that differ only in this layer and whose loss difference varies with the input. Fix the hidden-layer parameters so that the last hidden layer computes a continuous scalar feature $h(\mathbf u)$ that is nonconstant on the admissible input set (the support of the adversary's conditional). For one observation with label $y$, let $\theta_0$ set all output weights and biases to zero, and let $\theta_a$, with $a>0$, feed $h$ into the logit of class $y$ with weight $a$, all other output weights and biases being zero, so that $g_{\theta_a,y}(\mathbf u)=a h(\mathbf u)$ and $g_{\theta_a,k}(\mathbf u)=0$ for $k\neq y$. Then
\[
\ell_{\theta_0}(\mathbf u,y)=\log K,
\qquad
\ell_{\theta_a}(\mathbf u,y)
=\log\!\left(1+(K-1)e^{-a h(\mathbf u)}\right).
\]
The second expression is strictly decreasing in $h(\mathbf u)$, whereas the first is constant. Choosing admissible inputs $\mathbf u_0,\mathbf u_1$ with $h(\mathbf u_0)\neq h(\mathbf u_1)$ therefore gives
\[
\ell_{\theta_a}(\mathbf u_1,y)-\ell_{\theta_0}(\mathbf u_1,y)
\neq
\ell_{\theta_a}(\mathbf u_0,y)-\ell_{\theta_0}(\mathbf u_0,y).
\]
Thus the dataset loss has a nonzero contrast $\Delta$ which, by continuity of the loss in the parameters and inputs, remains nonzero on a set of positive measure, contradicting the necessary compatibility condition.\footnote{As Proposition~\ref{prop:bal_compatibility} requires, both conditionals are proper under a proper prior: \eqref{eq:learner_cond} because the cross-entropy is nonnegative, and \eqref{kkarg} because, for ReLU networks, the loss grows at most linearly in the input whereas the cost grows quadratically. Under the flat prior of their Algorithm~1, \eqref{eq:learner_cond} cannot even be normalized, since the softmax is invariant to adding a common constant to all logits.}

Squared-loss regression, also proposed in BAL, violates the separability condition even for a scalar linear predictor. Let $u\in\mathbb R$ be the perturbed input of an observation with clean input $x$ and response $y$, and consider the predictor $f_t(u)=tu$, whose parameter $t\in\mathbb R$ plays the role of $\theta$. As a function of $u$, the exponent of \eqref{kkarg} is then $(y-tu)^2-\alpha(u-x)^2$, up to terms not involving $u$, so the adversary's conditional is proper if and only if $t^2<\alpha$; we call such $t$ admissible and restrict the parameter to this range, so that Proposition~\ref{prop:bal_compatibility} applies. With fixed response $y$, compare $t=0$ with an admissible $t\neq0$:
\[
\ell_t(u,y)-\ell_0(u,y)
=(y-tu)^2-y^2=t^2u^2-2ytu.
\]
This is nonconstant on every nontrivial input interval, producing another nonzero contrast which, being a nonzero polynomial in the parameter and the inputs, is nonzero outside a null set. This regression example concerns the formulation proposed in their Section~3.

Hence, for both the classification and the regression losses, the BAL learner and adversary distributions cannot be the full conditionals of a common joint posterior. The alternating algorithm may nevertheless admit a stationary distribution and may be empirically useful; like other pseudo-Gibbs samplers built from incompatible conditionals, its stationary distribution then depends on the scan order \citep{heckerman2000dependency,chen2015behaviour}. What fails is its interpretation as Gibbs sampling from a joint posterior with the two specified full conditionals, and with it the definition of the robust posterior as the marginal over $\theta$ of that joint. Our proactive defense does not face this difficulty, since, for any fixed channel, its targets, \eqref{eq:robust_theta_posterior} and \eqref{eq:gen_posterior_eta}, are defined directly rather than through a pair of conditional distributions. For the learned channel of Section~\ref{subsec:channels}, they are those at the final value of $\omega$.

\end{document}